\documentclass[twoside]{article}

\usepackage[accepted]{aistats2022}
\usepackage[round]{natbib}

\usepackage{hyperref}       
\usepackage{url}            
\usepackage{booktabs}       
\usepackage{nicefrac}       
\usepackage{microtype}      
\usepackage{xcolor}         

\usepackage{amssymb, amsthm, amsmath}
\usepackage[bb=dsserif]{mathalpha} 
\usepackage{graphicx}
\usepackage{subfig}
\usepackage{wrapfig} 
\usepackage{mathtools} 
\usepackage{enumitem} 
\usepackage{mathtools} 

\usepackage[ruled,vlined]{algorithm2e} 
\SetKwInput{KwInput}{Input}                
\SetKwInput{KwOutput}{Output}              

\usepackage{tikz}
\newcommand*\circled[1]{\tikz[baseline=(char.base)]{
    \node[shape=circle, draw, inner sep=1pt, 
        minimum height=12pt] (char) {#1};}}
        
\newcommand{\indep}{\perp \!\!\! \perp}

\newcommand{\notni}{\not\owns}

\newcommand{\env}{e}
\newcommand{\envSet}{\mathcal{E}}
\newcommand{\numenv}{E}

\newcommand{\predBase}{X}
\newcommand{\targetBase}{Y}
\newcommand{\pred}[1]{\predBase^{#1}}
\newcommand{\target}[1]{\targetBase^{#1}}
\newcommand{\noise}{\varepsilon}
\newcommand{\noiseDistr}{D}
\newcommand{\numpred}{p}
\newcommand{\predIdx}{k}

\newcommand{\nObsBase}{n}
\newcommand{\nObs}[1]{\nObsBase_{#1}}
\newcommand{\predObs}{x}
\newcommand{\targetObs}{y}
\newcommand{\predData}{\textbf{\predObs}}
\newcommand{\targetData}{\textbf{\targetObs}}

\newcommand{\fct}{f}
\newcommand{\predSubset}{S}
\newcommand{\identifPred}{\predSubset(\envSet)}
\newcommand{\causalPred}{\predSubset^*}
\newcommand{\fctClass}{\mathcal{F}}

\newcommand{\hyp}{H}
\newcommand{\confLevel}{\alpha}

\newtheorem{assumption}{Assumption}
\newtheorem{definition}{Definition}
\newtheorem{lemma}{Lemma}
\newtheorem{theorem}{Theorem}
\newtheorem{proposition}{Proposition}
\newtheorem{remark}{Remark}

\newcommand{\wass}{W_2}
\newcommand{\coupling}{\pi}
\newcommand{\couplingSpace}[2]{\Pi(#1,#2)}
\newcommand{\wassVar}[2]{\text{WV}_{#2}(#1)}
\newcommand{\dualFunctional}[1]{S_{#1}}
\newcommand{\wassBarycenter}{\probBase^*}

\newcommand{\probSpace}{P_2}
\newcommand{\probBase}{\nu}
\newcommand{\probVectBase}{\boldsymbol{\probBase}}
\newcommand{\dirac}{\delta}
\newcommand{\empProbBase}{\hat{\probBase}}
\newcommand{\empProbBaseVect}{\boldsymbol{\empProbBase}}

\newcommand{\weight}{w}
\newcommand{\weights}{\boldsymbol{\weight}}
\newcommand{\weightSpace}{\Lambda}

\newcommand{\minWV}{\Gamma}
\newcommand{\WVMidentifPred}{\Tilde{\predSubset}(\envSet)}
\newcommand{\WVMest}{\hat{\predSubset}(\envSet)}

\newcommand{\sqrBdFctClass}{\mathcal{C}_{b,2}}
\newcommand{\contFct}{f}
\newcommand{\subFct}{g}
\newcommand{\subFctClass}{\mathcal{G}}

\newcommand{\empMax}{m}
\newcommand{\boundedContFctSet}[1]{\mathcal{C}_{#1}}
\newcommand{\contFctSet}{C(\mathbb{R})}
\newcommand{\ball}[2]{B(#1,#2)}
\newcommand{\lipFct}{\phi}
\newcommand{\lipFctClass}[1]{\mathcal{L}^0_{#1}}
\newcommand{\maxConst}{M}
\newcommand{\bdG}[3]{G_{#1}\left(#2,#3\right)}
\newcommand{\bdH}[3]{H_{#1}\left(#2,#3\right)}
\newcommand{\radVar}{\xi}
\newcommand{\radVarVect}{\boldsymbol{\radVar}}
\newcommand{\radComp}{\mathcal{R}}
\newcommand{\empRadComp}{R}
\newcommand{\coverSet}{B}
\newcommand{\lipDiffFct}{\psi}
\newcommand{\maxVar}{Z}
\newcommand{\maxVarEnv}[1]{\maxVar_{#1}}
\newcommand{\subGaussMean}[1]{\mu_{#1}}
\newcommand{\subGaussStd}[1]{\sigma_{#1}}
\newcommand{\constA}[1]{A_{#1}}
\newcommand{\constB}[1]{B_{#1}}
\newcommand{\constC}[1]{C_{#1}}
\newcommand{\bdProb}{\delta}

\newcommand{\brownBridge}{B}
\newcommand{\quantDensity}{q}
\newcommand{\covFct}[2]{\eta(#1,#2)}
\newcommand{\bandWidth}{h}
\newcommand{\kernel}{K}
\newcommand{\residual}{\epsilon}
\newcommand{\quantEst}{\hat{\quantDensity}}
\newcommand{\sobolev}{W}
\newcommand{\sobolevDegree}{d}
\newcommand{\cdf}{F}
\newcommand{\pdf}{f}
\newcommand{\compactset}{\mathcal{X}}
\newcommand{\contSpace}{C_0}
\newcommand{\sobBall}{B}
\newcommand{\sobRad}{R}

\newcommand{\nObsVect}{\boldsymbol{\nObs{}}}
\newcommand{\BrownianVect}{\boldsymbol{\brownBridge}}

\newcommand{\hausDist}{d_{H}}
\newcommand{\proofConstA}{c_1(\delta)}
\newcommand{\proofConstB}{c_2(\delta, \epsilon)} 
\newcommand{\proofConstC}{c_3(s,\delta)} 
\newcommand{\proofConstD}{c_4(s,\delta)} 
\newcommand{\proofConstE}{c_5(s,\delta,\epsilon')} 
\newcommand{\proofConstF}{c_6(\delta, \epsilon')} 
\newcommand{\proofConstG}{c_7(s, \epsilon')} 
\newcommand{\proofConstH}{c_8(s,\delta, \epsilon')} 
\newcommand{\minimizer}{\hat{\fct}_{\nObs{}}}
\newcommand{\Gset}{\subFctClass^{\delta}}
\newcommand{\Gsetn}{\Gset_{\nObs{}}}
\newcommand{\GsetZero}{\Gset_0}
\newcommand{\Czero}{\contSpace(\compactset, || \cdot ||_\infty )}

\DeclarePairedDelimiter{\ceil}{\lceil}{\rceil}

\usetikzlibrary{shapes,decorations,arrows,calc,arrows.meta,fit,positioning}
\tikzset{
    -Latex,auto,node distance =1 cm and 1 cm,semithick,
    state/.style ={ellipse, draw, minimum width = 0.7 cm},
    point/.style = {circle, draw, inner sep=0.04cm,fill,node contents={}},
    bidirected/.style={Latex-Latex,dashed},
    el/.style = {inner sep=2pt, align=left, sloped}
}

\begin{document}

%

%
\runningauthor{ Guillaume Martinet, Alexander Strzalkowski, Barbara E. Engelhardt}
\runningtitle{Variance Minimization in the Wasserstein Space for Invariant Causal Prediction}

\twocolumn[

\aistatstitle{Variance Minimization in the Wasserstein Space for \\ Invariant Causal Prediction}

\aistatsauthor{ Guillaume Martinet$^*$ \And Alexander Strzalkowski$^*$ \And  Barbara E. Engelhardt}

\aistatsaddress{ Princeton University \And  Princeton University \And Princeton University \\ Gladstone Institutes} ]

\begin{abstract}
Selecting powerful predictors for an outcome is a
cornerstone task for machine learning. However, some types of questions can
only be answered by identifying the predictors that causally affect the outcome. A recent approach to this causal inference problem
leverages the invariance property of a causal mechanism across
differing experimental environments~\citep{icp,nonlinear-icp}.
This method, \emph{invariant causal prediction} (ICP),
has a substantial computational defect -- the runtime
scales exponentially with the number of possible causal variables. In this
work, we show that the approach taken in ICP may be reformulated
as a series of nonparametric tests that scales linearly in
the number of predictors. Each of these tests relies on the
minimization of a novel loss function -- the Wasserstein
variance -- that is derived from tools in
optimal transport theory and is used to quantify distributional
variability across environments.  
We prove under mild assumptions that our method is able to recover the
set of identifiable direct causes, and we demonstrate
in our experiments that it is competitive with other benchmark causal discovery algorithms.
\end{abstract}

\section{INTRODUCTION}



Distinguishing between correlation and causation is a fundamental challenge that has been studied extensively over the years~\citep{pearl}. 
This distinction is necessary, for instance, to understand the behavior of regression under interventions.
Although regression is well understood in statistics and machine learning, when the same regression model is applied in different experimental conditions, the results may differ dramatically. Identifying which predictors are causal for an outcome is central to solving this limitation, since causal mechanisms by definition remain invariant across different experimental settings~\citep{peters-textbook}. 

Causal relationships are often represented by a
directed \emph{causal} graph, where each arrow signifies a direct
cause-effect relationship between two variables.
Usually, the approach to causal discovery has been to learn from
observational or interventional data the entire causal 
graph of the variables, sometimes only up to Markov equivalence. 
Many methods have been developed that use a variety of assumptions.
For example, methods such as Inductive Causation (IC, \cite{pearl}), Fast Causal Inference, and Peter and Clark's algorithm (FCI and PC, \cite{FCI-and-PC}) identify the Markov equivalence class of the causal graph using conditional independence tests under the so-called \emph{faithfulness} assumption, that all observable conditional independences stem only from the graph.
Score-based methods such as Greedy Equivalence Search (GES, \cite{ges}) and Greedy Interventional Equivalence Search (GIES, \cite{gies,gies2})
try to find the graph that maximizes some score function. 
On the other hand, methods like Linear Non-Gaussian Additive Models (LiNGAM, \cite{lingam,direct-lingam}), Regression with Subsequent Independence Test (RESIT, \cite{resit}), or Causal Additive Models (CAM, \cite{cam}) rely on model restrictions, such as additive nonlinear structural equations or non-Gaussian noises.
Another example is the Greedy Sparsest Permutation (GSP) family of methods~\citep{GSP, IGSP, UT-IGSP}, which combine conditional independence tests with score-based ideas. 

In practice, however, learning the whole causal graph is excessive. Often we are only interested in determining which variables are a \emph{direct 
cause} of a specific target variable. 
Here, we define the \emph{direct causes} as the parents of the target 
in the causal graph, which means that their causal effect on the target is not fully mediated by other observed variables.

A useful framework has been developed for inferring the direct causes 
of a target variable that -- instead of using conditional independence tests, score maximization, or model assumptions -- uses 
the stability of causal relationships across 
environments~\citep{icp}.
This approach, known as invariant causal prediction (ICP), leverages a key property of causal mechanisms: the conditional distribution of the target given its direct causes will not change when we intervene on any of the observed variables excluding the target. This method has desirable advantages over previous approaches (e.g., in general, conditional independence is not a testable hypothesis~\citep{hardness-cond-test}) and has been the source of inspiration for many recent algorithms~\citep{backshift, invariance-regression, causalDantzig, IRM}.

Unfortunately, the number of tests that ICP
needs to perform scales
exponentially in the number of predictors. Thus, ICP
often cannot be used even when the number of predictors is moderate.
In general, ICP is applied to only a \emph{small}
 subset of the predictors, pre-selected by a
sparse regression technique such as Lasso \citep{lasso}
or boosting \citep{greedy_boosting, stat_learning}.
This preselection step may severely reduce the power of ICP by rejecting variables that are direct causes, while including others that are not.

In this work, we show that the approach taken in ICP may be reformulated as a multiple-testing problem,
where the number of tests scales linearly in the
number of predictors. Given data from different experimental environments, we propose,
for each predictor, to test for the existence of an
invariant causal mechanism that does not involve the
predictor in question. Each test involves a
statistic based on a new loss function -- the 
Wasserstein variance -- that is used to quantify
distributional variability across environments. More
precisely, each of these statistics is obtained by solving a
Wasserstein variance minimization (WVM) program over a
restricted class of functions; when the resulting value surpasses some threshold, we declare
the corresponding predictor as causal.

This paper is organized as follows: Section \ref{sec:background}
introduces the setting, ICP, and our reformulation; Section \ref{sec:WVM} defines useful concepts from optimal transport and introduces the WVM algorithm; 
Section \ref{sec:theory} derives theoretical guarantees about WVM; Section \ref{sec:implementation} describes implementation details of WVM; Section \ref{sec:experiments} compares WVM against 
other standard methods on experiments; Section \ref{sec:conclusion} concludes.



\section{BACKGROUND} \label{sec:background} 

Suppose we are given data from \(\numenv\) distinct experimental environments 
\(\env \in \envSet \doteq \{1, \dots, E\}\). 
Let \(\pred{\env} \doteq (\pred{\env}_\predIdx )_{\predIdx = 1, \ldots, \numpred } \in \mathbb{R}^{\numpred}\) denote the
\(\numpred\) predictors and \(\target{\env} \in \mathbb{R}\) denote the target variable in environment $\env$.
For each environment, we observe \(\nObs{\env} \) i.i.d.~samples. The main 
assumption of our paper is that the causal mechanism that
relates the target variable to its direct causes is invariant across all environments.
This is the invariance property that ICP exploits. Like ICP, 
we model the causal mechanism as a structural
equation (SE) with additive noise~\citep{peters-textbook}.

\begin{assumption}[Invariant SE]
\label{invar-assump}
Denote \(\causalPred \subseteq \{1, \dots, \numpred\}
\) as the set of direct causes.
Let $\fctClass$ represent a class of functions of the predictors, and $\fctClass_{\causalPred} \subseteq \fctClass$ a subclass of functions that depend only on the direct causes.
For some fixed and unknown distribution \(\noiseDistr\) and function $\fct^* \in \fctClass_{\causalPred}$, \(\forall \env \in \envSet\),
\begin{equation} \label{eq:invariantSEM}
    \target{\env} = \fct^*(\pred{\env}) + \noise^\env, \;\;\; \noise^\env \sim \noiseDistr,\;\;\;
    \varepsilon^\env \indep \pred{\env}_{\causalPred} \doteq (\pred{\env}_{\predIdx})_{\predIdx \in \causalPred}.
\end{equation}
\end{assumption}


\begin{figure}[!t]
\vskip -0.1in
\subfloat[Environment $e = 1$.\label{fig:scmExamplea}]{\includegraphics[width=0.24\textwidth]{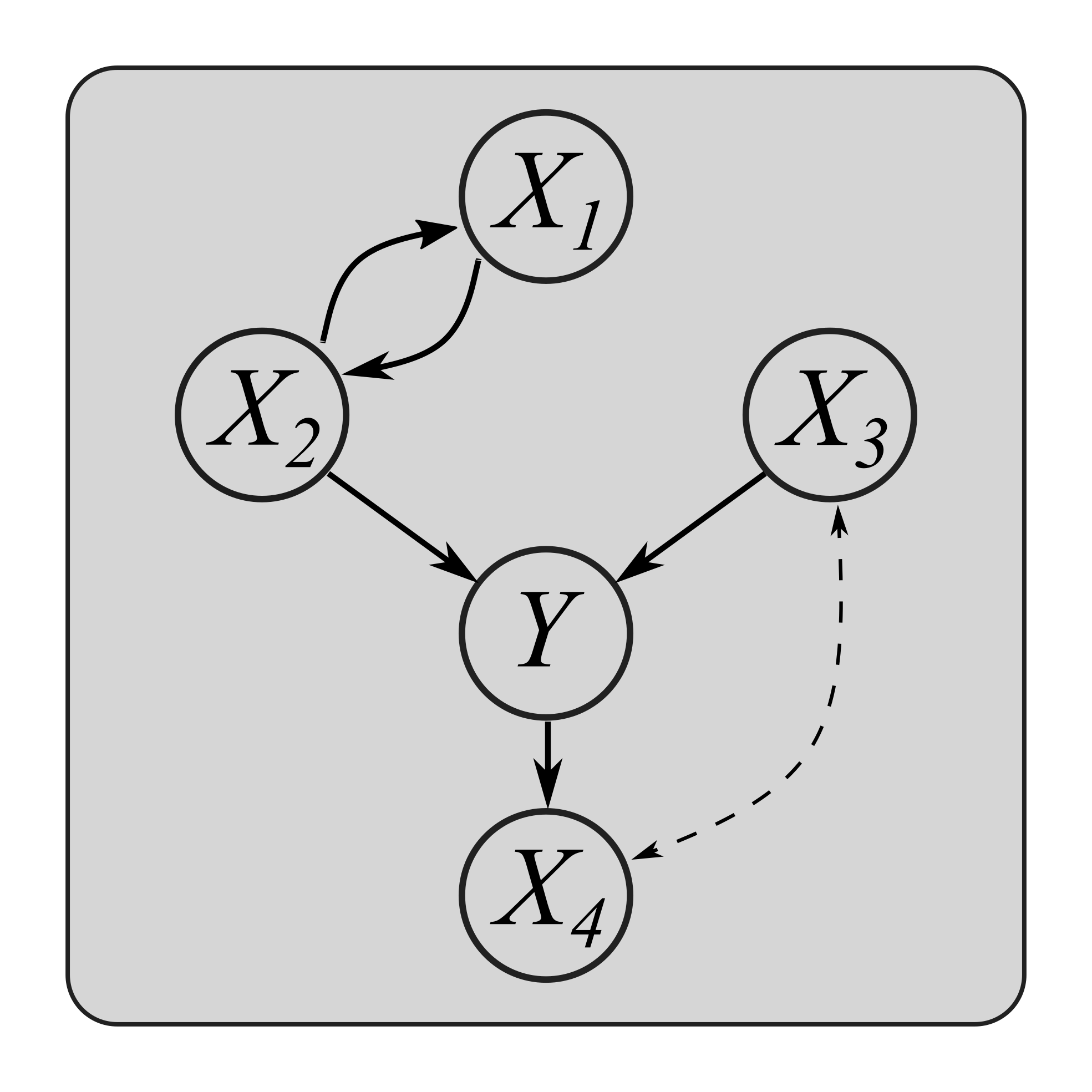}} \hfill
\subfloat[Environment $e = 2$.\label{fig:scmExampleb}] {\includegraphics[width=0.24\textwidth]{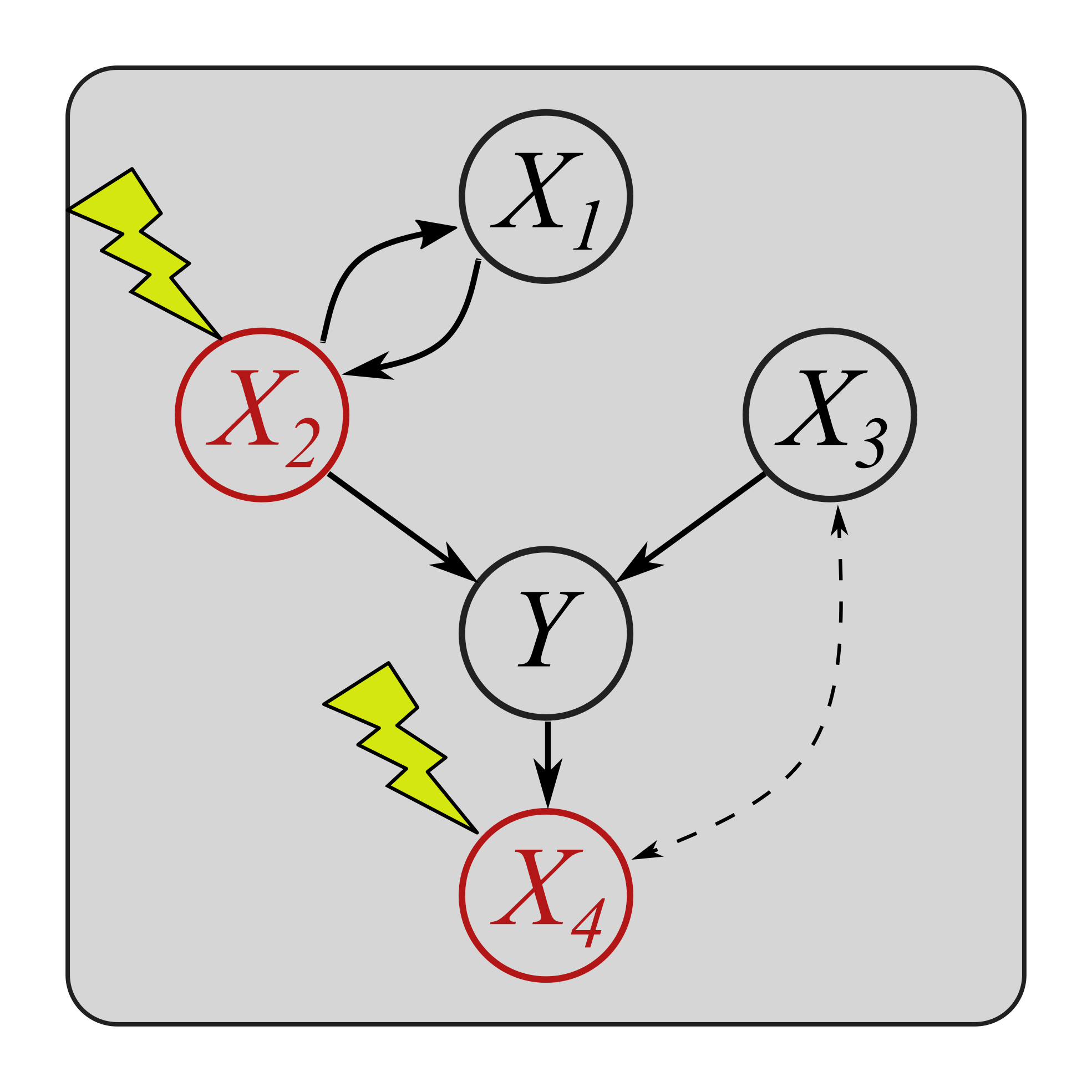}}
\caption{An SCM with $\numpred = 4$, $\causalPred = \{ 2,3 \}$, a feedback loop between $\predBase_1$ and $\predBase_2$, and a hidden confounder between $\predBase_3$ and $\predBase_4$: (a) in an observational setting $\env = 1$; (b) in an interventional setting $\env = 2$ with interventions on $\predBase_2$, $\predBase_4$.
} \label{fig:scmExample}
\end{figure}


Typically, this invariance property arises in
situations where the data are generated by interventions on
variables other than the target. Suppose that in an
observational setting \( \env = 1 \) the variables \(\pred{1}\) and
\(\target{1}\) are generated by a structural causal model (SCM; Figure \ref{fig:scmExamplea}). We allow the SCM
to admit feedback loops and hidden confounders as long as they do not
affect the causal mechanism between the target and its direct
causes. In another setting \( \env = 2 \), if some potentially unknown variables other than the target are intervened on (Figure \ref{fig:scmExampleb}),
then the SE between $\target{2}$ and $\pred{2}_{\causalPred}$ remains unchanged, and 
Assumption \ref{invar-assump} is satisfied between the two
environments. 
Note that the interventions can either remove causal relationships or modify SEs by changing their functions or the distribution of their noise.  
This type of scenario has been
studied for instance by~\citet{applied-icp} in the
context of a gene deletion experiment in yeast, where different environments are generated by knocking
out one or more of the genes, and ICP is used to
predict the causal effect of future interventions.

In addition, when only observational data are available, it is also
possible to generate different environments satisfying Assumption
\ref{invar-assump} by splitting up the data according to the values of
variables that are nondescendant of the target (e.g., an
instrumental variable) in the original SCM (see \citet{icp} for more details).

\paragraph{Invariant Causal Prediction (ICP).} The goal of ICP is to
recover the set $\predSubset^*$ for the target variable $\targetBase$ from the data.
We will denote by $\fctClass_{\predSubset}$ the set of functions from $\fctClass$ that depend only on predictors from $\predSubset$.
The invariance property in Equation \eqref{eq:invariantSEM} offers a
way to infer which predictors are direct causes by looking at
subsets of predictors $\predSubset \subseteq \{1, \dots,
\numpred\}$ that satisfy the following null hypothesis:
\begin{equation*} \label{eq:invarianceHyp}
    \hyp_{0, \predSubset}(\envSet): \left\{\begin{array}{ll}
      \text{for fixed } \fct \in \fctClass_{\predSubset} \text{ and distr. } \noiseDistr,  \forall \env \in \envSet,&\\
      \target{\env} = \fct(\pred{\env}) + \noise^\env, \; \; \; \noise^\env \sim \noiseDistr,\; 
    \varepsilon^\env \indep \pred{\env}_{\predSubset}.&
    \end{array}\right.
\end{equation*}

Assumption~\ref{invar-assump} implies that $\hyp_{0, \causalPred}(\envSet)$ is true. However, this is not sufficient
to guarantee the full identifiability of \( \causalPred \), since other subsets $\predSubset$ of predictors may also satisfy the hypothesis. Instead, ICP seeks to recover the set of \emph{identifiable} causal predictors that are defined to be predictors common to all $\predSubset$ for which
$\hyp_{0, \predSubset}(\envSet)$ is true \citep{icp}.

\begin{definition}[Identifiable causal predictors] Under Assumption \ref{invar-assump}, the set of identifiable causal predictors is:
\begin{equation} \label{eq:identifPred}
    \identifPred \doteq \bigcap_{\predSubset: \,\, \hyp_{0, \predSubset}(\envSet) \text{ is true}} \predSubset \subseteq \predSubset^*.
\end{equation}
\end{definition}

For linear Gaussian SCMs, \citet{icp} provides several sufficient conditions on the types of
interventions applied in experimental environments that imply the identifiability of the direct
causes, that is $\identifPred =
\predSubset^*$. 
Regardless of identifiability, both the linear \citep{icp}
and nonlinear~\citep{nonlinear-icp} versions of ICP derive an estimator of $\identifPred$ from the
data by testing
$\hyp_{0, \predSubset}(\envSet)$ for all subsets $\predSubset$,
and taking the intersection as in Equation \eqref{eq:invarianceHyp}.
One way to test the hypothesis $\hyp_{0, \predSubset}(\envSet)$ is to regress the target
variable on the set $\predSubset$ of predictors and to test whether
the resulting noise has an invariant distribution across
environments, e.g., by using a Kolmogorov-Smirnov test. However, this formulation requires a combinatorial search over all subsets $\predSubset$.
ICP's runtime hence scales exponentially with $\numpred$, as it needs to perform $2^p$ tests in total.

\paragraph{Another Formulation of ICP.}
The exponential scaling of ICP prohibits its application
to settings with even a moderate number of variables.
As we show, it is in fact possible to estimate
the set of identifiable causal predictors with
potentially many fewer tests.
For $\predIdx \in \{1, \dots, \numpred\}$, consider the following null hypothesis:
\begin{equation*}
    \hyp_{0, \predIdx}'(\envSet): \left\{\begin{array}{ll}
      \exists \predSubset \notni \predIdx, \exists \fct \in \fctClass_{\predSubset} \text{ and distr. } \noiseDistr \text{ s.t. } \forall \env \in \envSet, &\\  \;
      \noise^\env \doteq \target{\env} - \fct(\pred{\env}) \sim \noiseDistr \text{ and } 
    \varepsilon^\env \indep \pred{\env}_{\predSubset}.&
    \end{array}\right.
\end{equation*}
In other words, the hypothesis $\hyp_{0, \predIdx}'(\envSet)$
means that it is possible to find a function in
$\fctClass$ that does not depend on the predictor $\predIdx$
and yet satisfies Equation \eqref{eq:invariantSEM}. 
It is easy to prove that $ \identifPred$ can in fact be expressed as the set of predictors $\predIdx$ such that $\hyp_{0, \predIdx}'(\envSet)$ is false (proof in Appendix \ref{app:reformulation}):

\begin{lemma} [Reformulation of $\identifPred$] \label{lem:reformulation}

Under Assumption \ref{invar-assump}, the set of identifiable causal predictors can be expressed as:
    \begin{equation} \label{eq:reformulation}
    \identifPred = \left\{ \, \predIdx :  \hyp_{0, \predIdx}'(\envSet) \text{ is false } \right\}.
    \end{equation}
\end{lemma}

Lemma \ref{lem:reformulation} suggests that the approach
taken in ICP may be treated as a multiple testing problem. We build an estimator of $\identifPred$ by collecting all
predictors $\predIdx$ such that the null hypothesis $\hyp_{0, \predIdx}'(\envSet)$
may be rejected with high enough confidence.
This formulation of the problem requires $\numpred$ tests, instead
of $2^\numpred$ tests as in the original ICP, scaling linearly with the number of predictors.

To test each hypothesis $\hyp_{0, \predIdx}'(\envSet)$, we
rely on the minimization of a new loss function, the
Wasserstein variance, over $\fctClass_{-\predIdx}$ defined as the
set of functions in $\fctClass$ that do not depend on the
predictor $\predIdx$. When the resulting minimum is above
some threshold, we reject the hypothesis $\hyp_{0,
\predIdx}'(\envSet)$. More precisely, the Wasserstein variance is used to
quantify the distributional variability of the residuals
$\target{\env} - \fct(\pred{\env}) $ across
environments, in the sense that a high
Wasserstein variance means that the residuals' distributions differ substantially across environments; 
conversely, a Wasserstein variance equal to zero means that these distributions are identical.
Hence, a high value of the minimal Wasserstein variance over
$\fctClass_{-\predIdx}$ provides evidence of the nonexistence of
a function $\fct$ in this class such that the residuals have the
same distribution across environments.
Thus, a high value of the minimal Wasserstein variance means that $\hyp_{0, \predIdx}'(\envSet)$ should be rejected.
We provide more details on our algorithm in the following
sections.


\paragraph{Additional Notation.}
For the remainder, we introduce the following additional notation. We define \([p] \doteq \{1, \dots, p\}\) for any \(p \in \mathbb{N}\), the total number of observations as \(\nObs{} \doteq \sum_{\env=1}^{\numenv} \nObs{\env}\), and the minimum over the $\nObs{\env}$s as $\nObs{0} \doteq \min_{\env \in [\numenv]} \nObs{\env}$. 
Also, we call \(\probSpace\) the set of
probability measures on $\mathbb{R}$ with finite second moment, and $\dirac_{x}$ is the Dirac measure at $x$.
$\weightSpace$ will refer to the set \(\weightSpace \doteq \{ \weights = (\weight_\env)_{\env=1}^{\numenv} : \sum_{\env=1}^{\numenv} w_{\env} = 1 \text{ and } w_{\env} > 0, \; \forall \env \in [\numenv]\}\).
We call $\predData^\env \doteq (\predObs_i^\env)_{\env = 1}^{\nObs{\env}}$ the $\nObs{\env}$ observations of $\pred{\env}$ and $\predData \doteq (\predData^\env )_{\env = 1}^\numenv$, and define $\targetObs_i^\env, \targetData^\env, \targetData$ similarly for the target variable $\target{\env}$. We also write \(s \wedge t \doteq \min(s,t)\) for \(s, t \in \mathbb{R}\).

\section{WASSERSTEIN VARIANCE MINIMIZATION (WVM)} \label{sec:WVM}


In order to test whether $\hyp_{0, \predIdx}'(\envSet)$ can be rejected,
we need to measure the difference in distribution
between the residuals $\target{\env} - \fct(\pred{\env})$ across
environments; we rely on the Wasserstein distance to quantify that difference.
Compared to other metrics used in machine learning, such as the Kullback-Leibler (KL) divergence
or maximum mean discrepancy (MMD, \cite{MMD}), with Wasserstein distances it is possible to quantify
the variability of multiple distributions in both an efficient and nonparametric way. 
For instance, the KL divergence requires parametric models of the distributions, and while MMD is nonparametric its
complexity is \(\mathcal{O}(n^2)\) compared to \(\mathcal{O}(n\log n)\) for the Wasserstein distance in our setting. 
We rely on the Wasserstein variance, a quantity we derive
from the notion of Wasserstein barycenter first introduced by \cite{barycenter}.
We introduce these concepts below before presenting our method, the Wasserstein variance minimization (WVM) algorithm.

\paragraph{Wasserstein Distance.} The $2$-Wasserstein
distance (squared) $\wass^2$ is the optimal transportation cost between two probability distributions
$\probBase_1, \probBase_2 \in \probSpace$ with a squared Euclidean cost function:
\[
\wass^2(\probBase_1, \probBase_2) \doteq \inf_{\coupling \in \couplingSpace{\probBase_1}{ \probBase_2}} \int | x - y |^2 d\coupling(x,y).
\]
Here, $\couplingSpace{\probBase_1}{\probBase_2}$ is the set
of all joint distributions, also called couplings, with
marginals equal to $\probBase_1$ and $\probBase_2$.
The Wasserstein distance defines a metric on
$\probSpace$ (Theorem 7.3,
\citet{villani2003topics}).
Thus, $\wass(\probBase_1,\probBase_2) = 0$ iff $\probBase_1 = \probBase_2$. The resulting metric space $(\probSpace, \wass)$ is also called the \emph{Wasserstein space}.

\paragraph{Wasserstein Barycenter and Variance.} In a Euclidean
space, the barycenter $x$ of $\numenv$ points
$(x_\env)_{\env=1}^{\numenv}$ with respective weights $(\weight_\env)_{\env =
1}^{\numenv} \in \weightSpace$ minimizes $x \mapsto \sum_{\env} \weight_\env |x_\env - x|^2$, and the resulting minimal value is their variance.
By analogy, \citet{barycenter} defines the Wasserstein barycenter of
$\numenv$ probability distributions $(\probBase_\env)_{\env = 1}^{\numenv}$
as above by simply replacing the Euclidean distance by $\wass$.
Similarly, we define the Wasserstein variance as the resulting minimal
value:

\begin{definition}[Wasserstein variance] \label{def:WV}
Let $\probVectBase \doteq (\probBase_\env)_{\env = 1}^{\numenv}$ be probability
distributions from $\probSpace$. Their Wasserstein variance w.r.t.~the
weights $\weights \doteq (\weight_\env)_{\env=1}^{\numenv} \in \weightSpace$ is defined as follows:
\[
\wassVar{\probVectBase}{\weights} \doteq \inf_{\probBase \in \probSpace} \sum_{\env = 1}^{\numenv} \weight_{\env} \cdot \wass^{2}(\probBase_{\env}, \probBase).
\]
A minimizer $\wassBarycenter$ of the above infimum is called a \emph{Wasserstein
barycenter}, and there always exists at least one Wasserstein barycenter (see Proposition 2.3, \cite{barycenter}).

\end{definition}

The Wasserstein variance is a practical tool to quantify the
variability of different probability distributions; a low
Wasserstein variance means that the distributions are more similar. In particular, the next result follows directly from the
fact that $\wass$ is a metric:

\begin{lemma}[A zero Wasserstein variance means no variability] \label{lem:noVariability} Let $\probVectBase$, $\weights$ be as in Definition \ref{def:WV}. Then,
\[
\wassVar{\probVectBase}{\weights} = 0 \Longleftrightarrow \probBase_1 = \probBase_2 =  \ldots = \probBase_\numenv.
\]

\end{lemma}



\paragraph{WVM Algorithm.} Call $\probVectBase(\fct) \doteq (\probBase_\env(\fct))_{\env=1}^\numenv$ the distribution of the
residuals $\target{\env} - \fct(\pred{\env}) $ for $\env \in [\numenv]$. Fix weights $\weights \doteq
(\weight_\env)_{\env=1}^{\numenv} \in \weightSpace$. We propose
to test $\hyp_{0, \predIdx}'(\envSet)$ for each $\predIdx
\in [\numpred]$ by checking whether a zero optimal value is obtained for the following population-wise Wasserstein variance minimization (WVM):
\begin{equation} \label{eq:minWV}
\minWV_{\weights }(\fctClass_{-\predIdx}) \doteq \inf_{\fct \in \fctClass_{-\predIdx}}  \wassVar{\probVectBase(\fct)}{\weights}.
\end{equation}
As a consequence of Lemma \ref{lem:noVariability} and the
definition of $\hyp_{0, \predIdx}'(\envSet)$, we have that $\hyp_{0, \predIdx}'(\envSet)$ is false
whenever $\minWV_{\weights
}(\fctClass_{-\predIdx}) > 0$. 
The WVM algorithm thus aims 
at testing whether the
following null hypotheses may be rejected: for $\predIdx \in [\numpred]$,
\[
\Tilde{\hyp}_{0, \predIdx}(\envSet) : \; \minWV_{\weights }(\fctClass_{-\predIdx}) = 0, \;\; \text{vs} \;\; \Tilde{\hyp}_{1, \predIdx}(\envSet) : \; \minWV_{\weights }(\fctClass_{-\predIdx}) > 0.
\]
We form an estimator of $\identifPred$ by collecting every predictor
$\predIdx$ such that $\Tilde{\hyp}_{0,
\predIdx}(\envSet)$ may be rejected with high enough
confidence. Note that $\Tilde{\hyp}_{0, \predIdx}(\envSet)$ is a weaker null hypothesis than
$\hyp_{0, \predIdx}'(\envSet)$, as the latter
implies the former but not the converse, since
$\hyp_{0, \predIdx}'(\envSet)$ also implies that the
residuals are independent of the causal predictors.
Thus, in some situations our approach may yield 
a conservative estimate of $\identifPred$, even in
the limit of infinite data.
\begin{definition}[WVM's identifiable causal predictors]\label{def:WVMindentifPred} Under Assumption \ref{invar-assump}, we define the set of identifiable causal predictors for the WVM algorithm to be:
\begin{equation} \label{eq:WVMidentifPred}
    \WVMidentifPred \doteq \left\{ \, \predIdx :  \Tilde{\hyp}_{0, \predIdx}(\envSet) \text{ is false } \right\} \subseteq \identifPred \subseteq \causalPred.
\end{equation}
\end{definition}
Several remarks are in order. First, for
practical reasons, ICP also tests hypotheses that are
effectively weaker than $\hyp_{0, \predSubset}(\envSet)$
(see Section 3.1 from \cite{icp}). Moreover, most of the known
identifiability conditions for ICP (i.e.,~Theorem 2 from
\cite{icp}) apply here since their proofs rely only on the
invariant distribution of the residuals. Thus, 
those conditions are also sufficient to have $\WVMidentifPred = \causalPred$. Finally, a weaker null hypothesis means
that the WVM algorithm would also work under less
restrictive conditions than Assumption \ref{invar-assump}. In particular,
the independence condition in
\eqref{eq:invariantSEM} excludes any possibility of a hidden
confounder between $\targetBase$ and
$\predBase_{\causalPred}$. 

\cite{icp} also considers
a more general setting with instrumental variables that allows the presence of hidden confounders; they show that
ICP may be adapted to this setting at the cost
of having to perform an extensive grid search over all
regressors in order to test each null hypothesis $\hyp_{0, \predSubset}(\envSet)$.
Under this general setting, the WVM algorithm may be used to recover the same set of direct causes as ICP, and its advantage there is
twofold since it avoids the combinatorial search of Equation \eqref{eq:identifPred}
and also the extensive grid search; we include
details in Appendix \ref{app:generalSetting}.

\setlength{\textfloatsep}{5pt}
\begin{algorithm}[t]
\SetAlgoLined
\KwInput{$\predData, \targetData, \confLevel, \text{ \text{get\_threshold}()}$}
\KwOutput{$\WVMest$}
 Initialize $\WVMest \leftarrow \emptyset$ \;
 \For{$\predIdx \in [\numpred]$}{
  Obtain $\hat{\minWV}_{\weights }(\fctClass_{-\predIdx})$ and $\hat{\fct}_\predIdx $ \; Set $t \leftarrow \text{get\_threshold}(\confLevel, \weights, \predData, \targetData,\hat{\fct}_\predIdx)$ \;
  \lIf{$\hat{\minWV}_{\weights }(\fctClass_{-\predIdx}) > t$}{
   $\WVMest \leftarrow \WVMest \cup \{\predIdx\}$
   }
 }
 \caption{WVM Algorithm}\label{alg:WVM}
\end{algorithm}

The statistic that we use for testing $\Tilde{\hyp}_{0, \predIdx}(\envSet)$ is the minimal value \eqref{eq:minWV}, where each
distribution $\probBase_\env(\fct)$ is replaced by its
empirical counterpart $\empProbBase_\env(\fct) \doteq \nObs{\env}^{-1} \sum_i \dirac_{\targetObs^\env_i - \fct(\predObs_i^\env)}$.
More precisely, we compute $\hat{\minWV}_{\weights }(\fctClass_{-\predIdx}) \doteq \inf_{\fct \in \fctClass_{-\predIdx}}  \wassVar{\empProbBaseVect(\fct)}{\weights}$ for every $\predIdx \in [\numpred]$,
and we reject $\Tilde{\hyp}_{0, \predIdx}(\envSet)$
whenever it is above some threshold; we also call
$\hat{\fct}_\predIdx$ the resulting minimizer (see Algorithm \ref{alg:WVM}).
We discuss how these thresholds are chosen and how the optimization is performed below.

\paragraph{Connection with Likelihood Ratio Tests.} 
The test we propose is similar to the classical \emph{likelihood ratio test} (LRT). 
If we were interested in testing the statistical significance of a predictor $X_k$ in a regression model parametrized by $\theta \in \Theta$, we might use a LRT to test whether $\inf_{\theta \in \Theta_k} - l(\theta) - \inf_{\theta \in \Theta} - l(\theta)$ is zero or strictly positive, where $l(\theta)$ is the log likelihood and $\Theta_k \subset \Theta$ is a restricted model that excludes $X_k$ from the regression.
Under Assumption \ref{invar-assump}, we can rewrite $\Tilde{\hyp}_{0, \predIdx}(\envSet)$ as $\minWV_{\weights }(\fctClass_{-\predIdx}) - \minWV_{\weights }(\fctClass) = 0$. Thus, the WVM test essentially replaces the negative log likelihood from the LRT, which measures lack-of-fit, by the Wasserstein variance, which measures distributional variability instead.

Note that we can use a LRT to test a more restricted model, say $\Theta_S$, that excludes a subset $S$ of predictors such that $|S|\geq 1$; this extension to subset exclusion is another advantage of WVM over ICP. 
Then, WVM may be used to detect whether \emph{at least one} of the predictors from $S$ is causal, which can be useful in situations where these predictors are correlated and thus their effects are hard to distinguish statistically. ICP generally cannot be extended to subset exclusion. 
We discuss this extension in Appendix \ref{app:extensionWVM}.


\section{THEORETICAL ANALYSIS} \label{sec:theory} 

In this section, we first establish a new uniform bound for finite samples between the Wasserstein variance and its empirical counterpart in terms of the Rademacher complexity \citep{shalev2014understanding}.
This guaranties that the Wasserstein variance is no more prone to over-fitting than any of the classical loss functions used in machine learning, and in particular that $\hat{\minWV}_{\weights }(\fctClass_{-\predIdx})$ will get close to $\minWV_{\weights }(\fctClass_{-\predIdx})$ in finite samples for a suitable function class $\fctClass_{-\predIdx}$.
The proof is in Appendix \ref{app:proofUnifBound}.

\begin{theorem}[Uniform Bound] \label{thm:unifBound}
Let $\bdProb \in (0,1)$ and $\subFctClass$ be some class of functions of the predictors.
If, for each $\env \in [\numenv]$, the variable $\maxVarEnv{\env} \doteq \sup_{\subFct \in \subFctClass} | \targetBase^\env - \subFct(\predBase^\env) |$ is sub-Gaussian, then with probability at least $1 - \bdProb$ we have:
\begin{align}\label{eq:unifBound}
      \forall \subFct \in \subFctClass, \quad &\left| \wassVar{
    \empProbBaseVect(\subFct)}{\weights} - \wassVar{
    \probVectBase(\subFct)}{\weights} \right|  \\ 
    \leq 
 \sum_{\env=1}^{\numenv} \weight_{\env}  &\left( \frac{\constA{\bdProb,\nObs{}}}{\sqrt{\nObs{\env}}} + \constB{\bdProb,\nObs{}} (1 + \log(\nObs{\env})) \radComp_{\nObs{\env}}(\subFctClass) \right)  + \frac{\constC{\bdProb,\nObs{}}}{\nObs{}}, \nonumber 
\end{align}
where $\radComp_{\nObs{\env}}(\subFctClass)$ is the Rademacher complexity of $\subFctClass$ under environment $\env$ (see Definition \ref{def:radComp} in Appendix \ref{app:proofUnifBound}), and $\constA{\bdProb,\nObs{}},\constB{\bdProb,\nObs{}},\constC{\bdProb,\nObs{}} = O(\log(\nObs{}/ \bdProb))$.
Also, if the variables $\maxVarEnv{\env}$ are bounded with probability one, then  $\constA{\bdProb,\nObs{}},\constB{\bdProb,\nObs{}},\constC{\bdProb,\nObs{}}$ are just $O(\log(1 /\bdProb))$.
As a consequence, the bound from \eqref{eq:unifBound} is also verified for 
$| \inf_{\subFct \in \subFctClass} \wassVar{
    \empProbBaseVect(\subFct)}{\weights} - \inf_{\subFct \in \subFctClass} \wassVar{
    \probVectBase(\subFct)}{\weights} |$ 
with probability at least $1-\bdProb$.
\end{theorem}
We derived the bound in Theorem \ref{thm:unifBound} for more general function classes than $\fctClass$. Often in practice the minimizer of a loss over some (large) class $\fctClass$ belongs to a more restricted class $\subFctClass \subset \fctClass$, which is more useful for the convergence analysis \citep{bartlett2005local}. 
Theorem \ref{thm:unifBound} establishes a guarantee on the ability of the WVM algorithm to recover $\WVMidentifPred$ in finite samples, and is used for the asymptotic results of Theorem \ref{thm:asymptGuar}.
This bound (Equation \eqref{eq:unifBound}) shows that, with high probability and enough data, $\hat{\minWV}_{\weights }(\fctClass_{-\predIdx})$ is close to $\minWV_{\weights }(\fctClass_{-\predIdx})$ for all $\predIdx$, and therefore it is possible to distinguish the identifiable causal predictors from the others.
This means that there exists a choice of threshold $t$ in Algorithm \ref{alg:WVM} such that the output $\WVMest$ is equal to $\WVMidentifPred$ with high probability when $\nObs{}$ is large enough; below we discuss how to choose this threshold.

Theorem \ref{thm:unifBound} also offers insight on how to choose the weights $\weights$. 
Since the Rademacher complexity converges to zero for the usual classes of functions, in general at a $\tilde{O}( \nObs{\env}^{-1/2})$ rate \citep{bartlett2002rademacher}, bound \eqref{eq:unifBound} suggests that we use smaller weights in environments with less data.
In the next section, we set $\weight_\env = \nObs{\env} / \nObs{}$, which leads to a $\tilde{O}(\nObs{}^{-1/2})$ bound in Equation \eqref{eq:unifBound}.


\paragraph{Setting the Thresholds.} We show how the thresholds can be set based on the asymptotic distribution of the Wasserstein variance under $\Tilde{\hyp}_{0, \predIdx}(\envSet)$.
For a probability distribution $\probBase$ on the real line, and $\empProbBase_{\nObs{}}$, its empirical estimate with $\nObs{}$ samples, the asymptotic distribution of $\nObs{} \wass^2(\empProbBase_{\nObs{}}, \probBase)$ has already been established in the literature; see \cite{del2005asymptotics} for a complete treatment. 
More precisely, under some technical conditions we have:
\begin{equation} \label{eq:wassAsymptLimit}
    n \wass^2(\empProbBase_{\nObs{}}, \probBase) \xrightarrow[\nObs{} \rightarrow \infty]{d} \int_0^1 \brownBridge{}^2(t) \, \quantDensity^2(t) dt,
\end{equation}
where $(\brownBridge(t))_{t\in [0,1]}$ is the Brownian bridge between $0$ and $1$, i.e.,~a Gaussian process with covariance function $\covFct{s}{t} \doteq t \wedge s - s t$, and $\quantDensity$ is the quantile density of $\probBase$, i.e.,~the derivative of its quantile function.
The asymptotic result of Equation \eqref{eq:wassAsymptLimit} holds when the CDF of $\probBase$ is twice differentiable and $\int_0^1 \covFct{t}{t} \quantDensity^2(t) dt$ is finite, along with other regularity conditions \citep{del2005asymptotics}; we provide the full list of these conditions in Appendix \ref{app:listSuffCondAsymptLimit}. 
Similarly, and under the same set of conditions, we derive the following asymptotic result for the Wasserstein variance:
\begin{proposition} \label{prop:WVAsymptLimit}
Assume that data from different environments are independent of each other, and set $\weight_\env = \nObs{\env} / \nObs{}$. 
Let $\fct$ be any function in $\fctClass$ such that $\wassVar{\probVectBase(\fct)}{\weights} = 0$, i.e.,~there exists $\probBase(\fct)$ such that $\forall \env \in [\numenv], \, \probBase_\env(\fct) = \probBase(\fct)$.
Assuming $\probBase(\fct)$ respects the assumptions needed for Equation \eqref{eq:wassAsymptLimit} to hold, we have:
\begin{equation} \label{eq:WVExactAsymptLimit}
    \nObs{} \wassVar{\empProbBaseVect(\fct)}{\weights} \xrightarrow[\nObs{0} \rightarrow \infty]{d} \sum_{\env = 1}^{\numenv - 1} \int_{0}^{1} \brownBridge^2_{\env}(t) \quantDensity_{\fct}^2(t) dt,
\end{equation}
where $\quantDensity_\fct$ is the quantile density of $\probBase(\fct)$, and $(\brownBridge_{\env}(t))_{\env = 1}^{\numenv - 1}$ are $\numenv - 1$ independent Brownian bridges.
\end{proposition}

To test $\Tilde{\hyp}_{0, \predIdx}(\envSet)$ at confidence level $\confLevel$, we set the threshold at the $(1-\confLevel)$-quantile of the limit distribution from \eqref{eq:WVExactAsymptLimit}, for $\fct = \hat{\fct}_\predIdx$.
Because the quantile density $\quantDensity_\fct$ is unknown, we propose to estimate it within each environment by a kernel quantile density estimator 
\citep{sheather1990kernel, jones1992estimating}.

\begin{definition}[Kernel quantile density estimator] \label{def:quantEstDef}
Denote by $(\residual^\env_{(i)}(\fct))_{i=1}^{\nObs{\env}}$ the residuals for function $\fct$ in environment $\env$ and sorted in increasing order; we use the following quantile density estimator with bandwidth $\bandWidth_\env \propto \nObs{\env}^{-1/3}$ to estimate the quantile density of $\probBase_\env(\fct)$, for any $t \in [0,1]$:
\begin{equation} \label{eq:kernelDensityEstimator}
 \quantEst_{\fct}^{\env}(t) \doteq \sum_{i=2}^{\nObs{\env}} \left(\residual_{(i)}^\env(\fct) - \residual_{(i-1)}^\env(\fct) \right) \kernel_{\bandWidth_{\env}}\left(t - \frac{i-1}{\nObs{\env}}\right),
\end{equation}
where $\kernel_\bandWidth(u) \doteq \bandWidth^{-1} \kernel(u / \bandWidth)$, and $\kernel$ is a Lipschitz kernel supported on $[-1,1]$ such that $\int \kernel(u) du = 1$.
\end{definition}

\begin{algorithm}[t]
\SetAlgoLined
\KwInput{$\confLevel, \weights = \nObs{\env} / \nObs{}, \predData, \targetData,\fct = \hat{\fct}_{\predIdx}$}
\KwOutput{$\hat{t}_{\confLevel}$}
 \For{$\env \in [\numenv]$}{
    Estimate $\quantEst_\env$ using kernel estimator \eqref{eq:kernelDensityEstimator} \;
 }
 Set $\hat{t}_{\confLevel}$ as the $(1-\confLevel)$-quantile of variable \eqref{eq:limitVariable} \;
 \caption{get\_threshold}\label{alg:get_threshold}
\end{algorithm}

In Theorem \ref{thm:asymptGuar} we show that, by choosing the threshold as discussed above,
with the quantile density replaced by its estimator from Definition \ref{def:quantEstDef}, we get a consistent test of level $\confLevel$ for $\Tilde{\hyp}_{0, \predIdx}(\envSet)$.
We need however to impose some additional assumptions; in particular, we assume that $\hat{\fct}_\predIdx$ is a bounded function in a Sobolev space, a class of functions used in nonparametric statistics \citep{tsybakov2008introduction}. 
A detailed list of these regularity conditions in Assumption \ref{regularityAssumptions} may be found in Appendix \ref{app:regCondTh2}. 

\begin{assumption}[Summary of the reg.~conditions] \label{regularityAssumptions} 
For any $\env \in [\numenv]$, $\pred{\env}$ is bounded with probability one, $\target{\env}$ is sub-Gaussian, and $\nObs{\env}\geq \lambda \nObs{}$ for a constant $\lambda>0$.
Also, data from different environments are independent. 
For $\nObs{}$ large enough, and any $\predIdx \in [\numpred]$ we have with high probability that $\hat{\fct}_\predIdx$ belongs to a fixed bounded set in a Sobolev space $\sobolev^{\sobolevDegree,2}$ with $\sobolevDegree > \numpred / 2$. 
Furthermore, uniformly over all functions $\fct$ in this set and $\env \in [\numenv]$, $\probBase_\env(\fct)$ satisfies the conditions needed for the asymptotic result in Equation \eqref{eq:wassAsymptLimit} to hold.
\end{assumption}

Now we present Theorem \ref{thm:asymptGuar}; see Appendix \ref{app:proofAsymptGuar} for proof.

\begin{theorem} [Asymptotic guaranties]\label{thm:asymptGuar} Assume Assumption \ref{regularityAssumptions} is true and let $\predIdx \in [\numpred]$.
For every $\env \in [\numenv]$, set $\weight_{\env} = \nObs{\env} / \nObs{}$ and, for simplicity, call $\quantEst_\env$ the quantile density estimator from Equation \eqref{eq:kernelDensityEstimator} for $\fct = \hat{\fct}_{\predIdx}$, the minimizer of $\hat{\minWV}_{\weights }(\fctClass_{-\predIdx})$. 
Set $\quantEst \doteq \sum_{\env = 1}^{\numenv} \weight_{\env} \quantEst_\env$ and let $\hat{t}_{\confLevel}$ be the $(1-\confLevel)$-quantile of the variable:
\begin{equation} \label{eq:limitVariable}
    \frac{1}{\nObs{}}\sum_{\env = 1}^{\numenv - 1} \int_0^1 \brownBridge_{\env}^2(t) \, \quantEst^2(t) dt,
\end{equation}
where $(\brownBridge_{\env}(t))_{\env = 1}^{\numenv-1}$ are $\numenv - 1$ independent Brownian bridges between $0$ and $1$. Rejecting $\Tilde{\hyp}_{0, \predIdx}(\envSet)$ whenever $\hat{\minWV}_{\weights }(\fctClass_{-\predIdx}) > \hat{t}_{\confLevel}$ forms a consistent test of asymptotic level $\confLevel$. That is:
\begin{align*}
    \text{Under }& \Tilde{\hyp}_{0, \predIdx}(\envSet): \, \limsup_{\nObs{} \rightarrow \infty} \, \mathbb{P}(\hat{\minWV}_{\weights }(\fctClass_{-\predIdx}) > \hat{t}_{\confLevel}) \leq \confLevel, \\ \text{and under }& \Tilde{\hyp}_{1, \predIdx}(\envSet): \lim_{\nObs{} \rightarrow \infty}\mathbb{P}(\hat{\minWV}_{\weights }(\fctClass_{-\predIdx}) > \hat{t}_{\confLevel}) = 1.
\end{align*}
\end{theorem}
In the appendix we prove a slightly more general result than Theorem \ref{thm:asymptGuar} to allow for the use of function classes that depend on the sample size (Theorem \ref{thm:generalAymptResult} in \ref{app:regCondTh2}).

\paragraph{Multiple Testing Correction.} Theorem \ref{thm:asymptGuar} says that, for each $\predIdx \in [\numpred]$, testing for $\Tilde{\hyp}_{0, \predIdx}(\envSet)$ has an asymptotic probability less than $\confLevel$ to return a false positive.
Since we need to perform $\numpred$ tests, if we want to control for the total number of false positives, one option is to correct for these multiple tests by choosing a lower $\confLevel$. One possibility, to control the family-wise error rate (FWER), is to use Bonferroni correction. 
We argue that, for WVM, such corrections may lead to conservative results. In general, Bonferroni correction is appropriate in situations where the tests are mostly independent of one another. 
In the case of WVM, however, the statistics $\hat{\minWV}_{\weights }(\fctClass_{-\predIdx})$ for $\predIdx \notin \causalPred$ may often be well correlated since they are all bounded by the same quantity $\wassVar{\empProbBaseVect(\fct^*)}{\weights}$ that converges to $0$.
As we show in Theorem \ref{thm:multCorrection}, correction is not needed when identifiability holds, since under identifiability the thresholds derived in Theorem \ref{thm:asymptGuar} for $\predIdx \notin \causalPred$ all converge toward the $(1-\confLevel)$-quantile of the limit distribution in \eqref{eq:WVExactAsymptLimit} for $\fct = \fct^{*}$. In that case, the probability of \emph{any} false positive across all of the tests is already bounded by $\confLevel$ asymptotically. 
By identifiability, we mean that $\fct^*$ is the only function in the closure of $\fctClass$ such that $\wassVar{\probVectBase(\fct^*)}{\weights} = 0$. In the case of linear functions, the sufficient conditions of Theorem 2 from
\citet{icp} also imply identifiability in this sense.

\begin{theorem} [When no correction is needed]  \label{thm:multCorrection}
Assume Assumptions \ref{invar-assump} and \ref{regularityAssumptions} are true, and set $\weight_\env = \nObs{\env} / \nObs{}$. Call $\WVMest$ the output of Algorithm \ref{alg:WVM} where the confidence level for the thresholds returned by Algorithm \ref{alg:get_threshold} is set at a fixed $\confLevel > 0$.
Under identifiability of Equation \eqref{eq:invariantSEM}, we have:
\[
\liminf_{\nObs{} \rightarrow \infty} \mathbb{P}(\WVMest = \WVMidentifPred) = \liminf_{\nObs{} \rightarrow \infty} \mathbb{P}(\WVMest = \causalPred) \geq 1 - \confLevel.
\]
\end{theorem}


\section{IMPLEMENTATION DETAILS} \label{sec:implementation}

We now explain how the optimization of \(\hat{\minWV}_{\weights }(\fctClass_{-\predIdx})\) is performed, and how we approximate the distribution of Equation \eqref{eq:limitVariable} to obtain the thresholds. Additional details can be found in Appendix \ref{app:addImpDetails}.

\paragraph{Optimization.} Since the residuals are one dimensional, the Wasserstein variance here admits a closed form. 
More precisely, for any collection of distributions \(\probVectBase = (\probBase_{\env})_{\env = 1}^\numenv\) defined on $\mathbb{R}$ let 
$\cdf_\env^{-1}$ denote the quantile function of $\probBase_\env$. By remarks 2.30 and 9.6 from \citet{peyre2019computational}, we have:
\begin{equation} \label{eq:explicitWV}
    \wassVar{\probVectBase}{\weights} = \int_0^1 \sum_{\env = 1}^{\numenv} \weight_{\env} \cdot \left( F_\env^{-1}(t) - \sum_{\env' = 1}^{\numenv} \weight_{\env'} \cdf_{\env'}^{-1}(t) \right)^2 dt.
\end{equation}
The closed form~\eqref{eq:explicitWV} may be efficiently computed when each \(\probBase_\env\) is an empirical distribution such as \(\empProbBase_\env(\fct)\); this mainly requires sorting the residuals in each environment. 
Another useful property of Equation~\eqref{eq:explicitWV} is that when the functions are parametrized, that is, \(\fct(\cdot) = \fct(\, \cdot \, ; \theta)\) for some $\theta$, the Wasserstein variance $\wassVar{\empProbBaseVect(\fct(\, \cdot \, ; \theta))}{\weights}$ is almost
everywhere differentiable w.r.t.~$\theta$. 
Thus, to minimize \(\wassVar{\hat{\probVectBase}(\fct(\, \cdot \, ; \theta))}{\weights}\), one can use any gradient-based optimization method to obtain 
\(\hat{\minWV}_{\weights }(\fctClass_{-\predIdx})\);
we use L-BFGS in our experiments. 

\paragraph{Approximation of the Asymptotic Distribution.} 
From~\citet{del2005asymptotics}, the RHS 
of Equation~\eqref{eq:wassAsymptLimit} is a generalized $\chi$-square distributed variable, and so can be expressed in distribution as the sum $\sum_{i} \lambda_{i} Z_i^2$, where the $Z_i$s are i.i.d.~standard normal variables and the $\lambda_i$s are the eigenvalues of the integral operator with kernel $\eta'(s,t) = \covFct{s}{t}q(s)q(t)$.
The generalized $\chi$-square distribution may be accurately approximated by 
a Gamma distribution with the same mean and variance~\citep{gretton2007kernel, johnson1995continuous, kankainen1995consistent}; in our case, the mean and variance may be expressed in terms of the trace and Hilbert-Schmidt norms of the above-mentioned integral operator:
\begin{proposition}\label{prop:gammaApprox}
Let $\hat{\eta}(s,t) \doteq \covFct{s}{t} \hat{\quantDensity}(s)\hat{\quantDensity}(t)$ and $\numenv' \doteq \numenv - 1$. The mean $\hat{m}$ and variance $\hat{\sigma}^2$ of~\eqref{eq:limitVariable} are as follows:
\[
\hat{m} = \frac{\numenv' }{\nObs{}}   \hspace{-0.02in} \int_0^1 \hspace{-0.02in} \hat{\eta}(t,t) dt, \;\; \hat{\sigma}^2 = \frac{ 2 \numenv' }{\nObs{}^2}  \hspace{-0.02in} \iint_0^1 \hspace{-0.02in} \hat{\eta}^2(s,t) ds dt.
\]
\end{proposition}
We can therefore approximate the distribution of Equation~\eqref{eq:limitVariable} as 
a Gamma distribution with shape parameter \(\hat{\alpha} = \hat{m}^2/\hat{\sigma}^2\) and scale parameter 
\(\hat{\theta} = \hat{\sigma}^2/\hat{m}\). In practice, one can estimate the above integrals using Monte Carlo integration.

\begin{figure*}[t]
\vskip-0.1in
\centering
\subfloat[\label{fig:icp-wvm-time}]{%
      \includegraphics[width=0.25\textwidth]{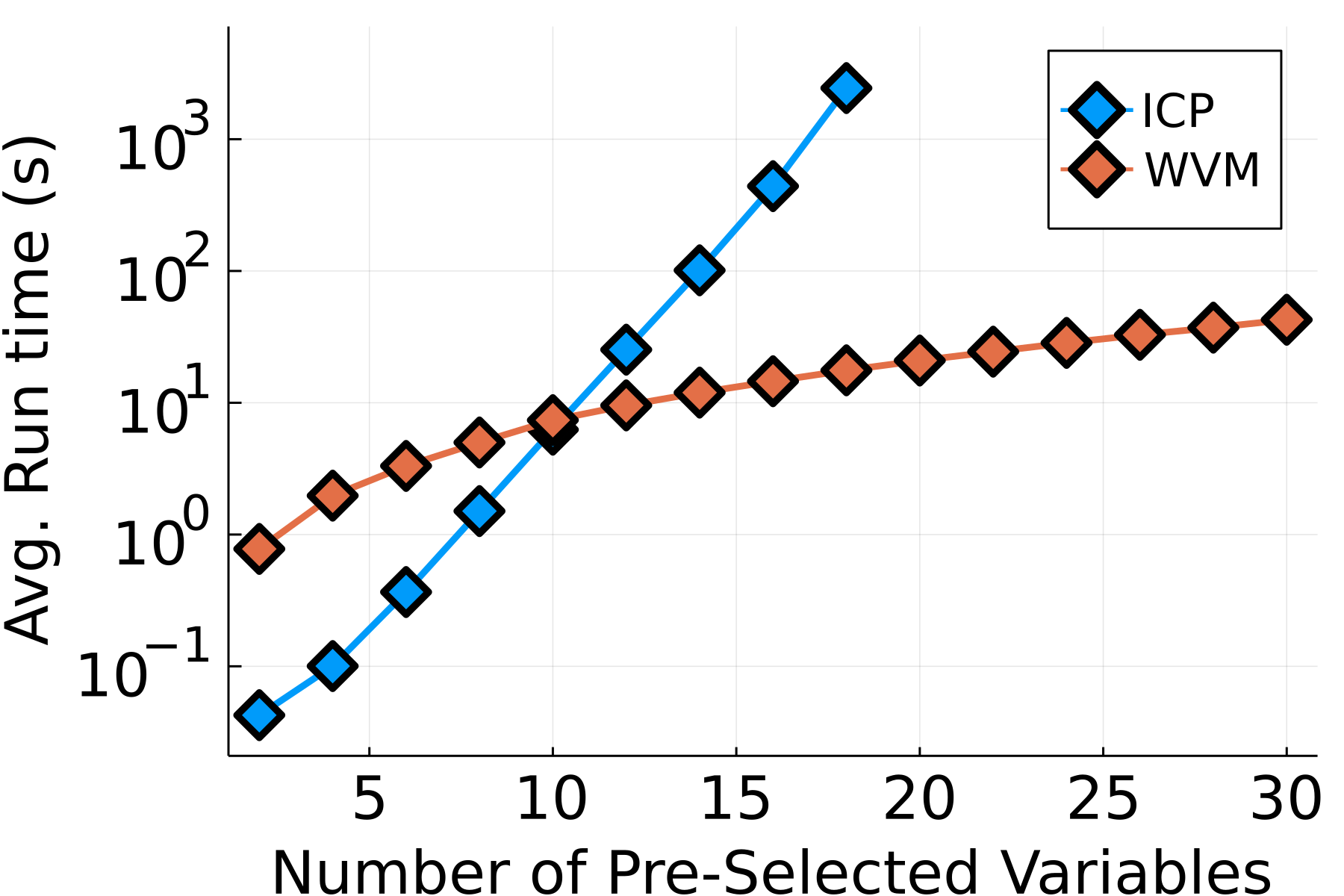}}
   \subfloat[\label{fig:fps-fns-wvm-icp-diff-vars} ]{%
      \includegraphics[width=0.25\textwidth]{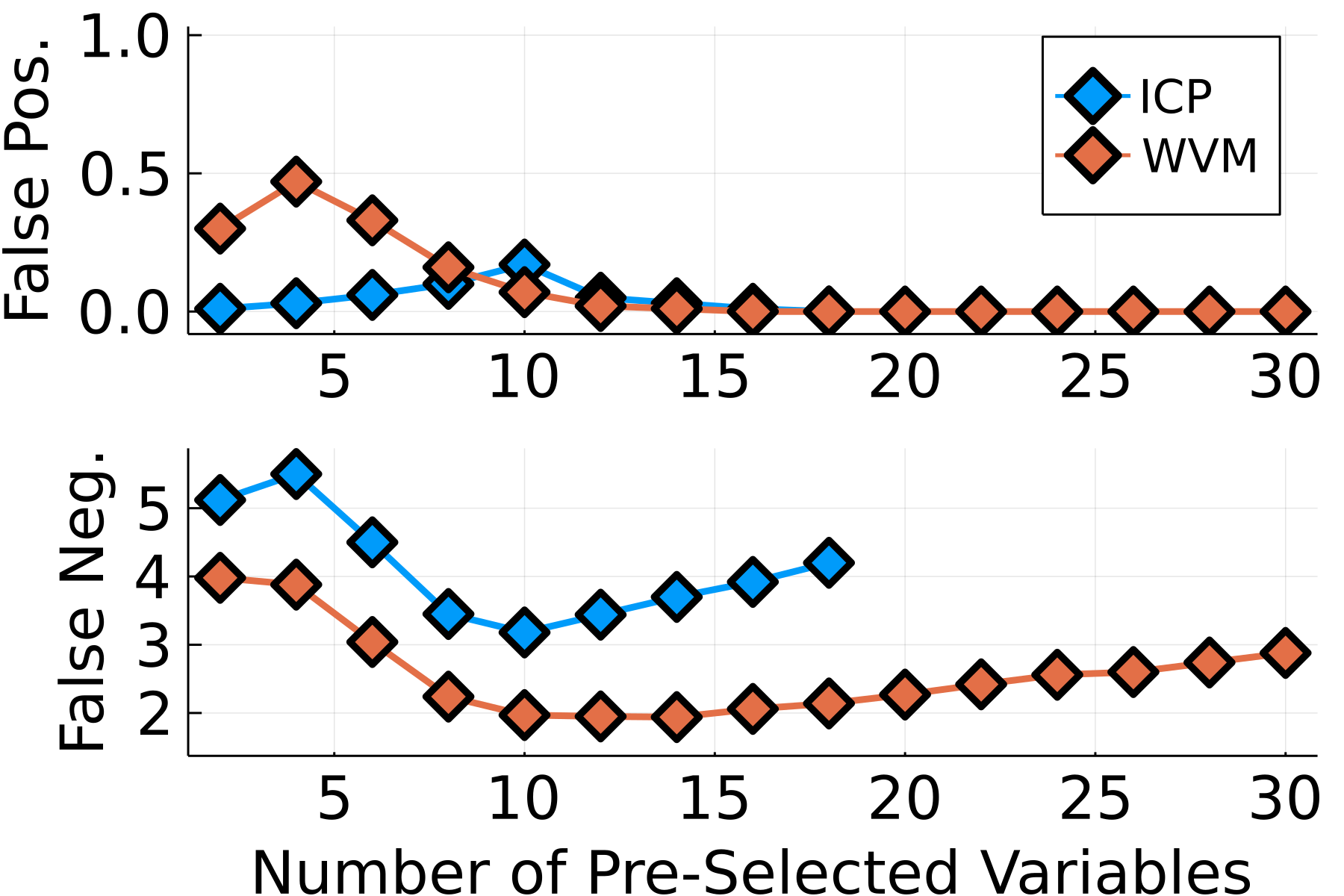}}
   \subfloat[\label{fig:wvs-pvals-example}]{%
      \includegraphics[width=0.25\textwidth]{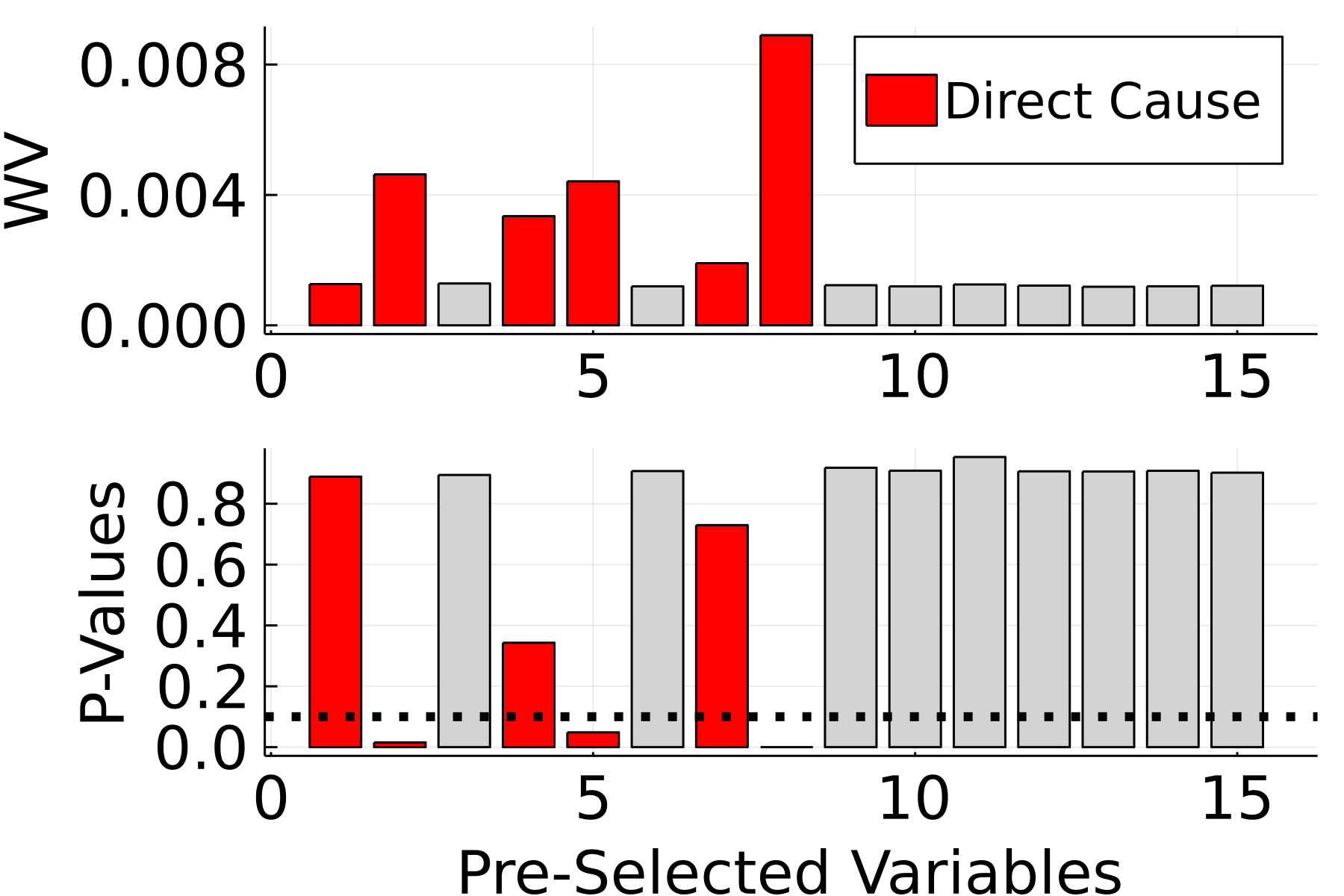}}
    \subfloat[\label{fig:pr-curve-default-sim}]{%
      \includegraphics[width=0.25\textwidth]{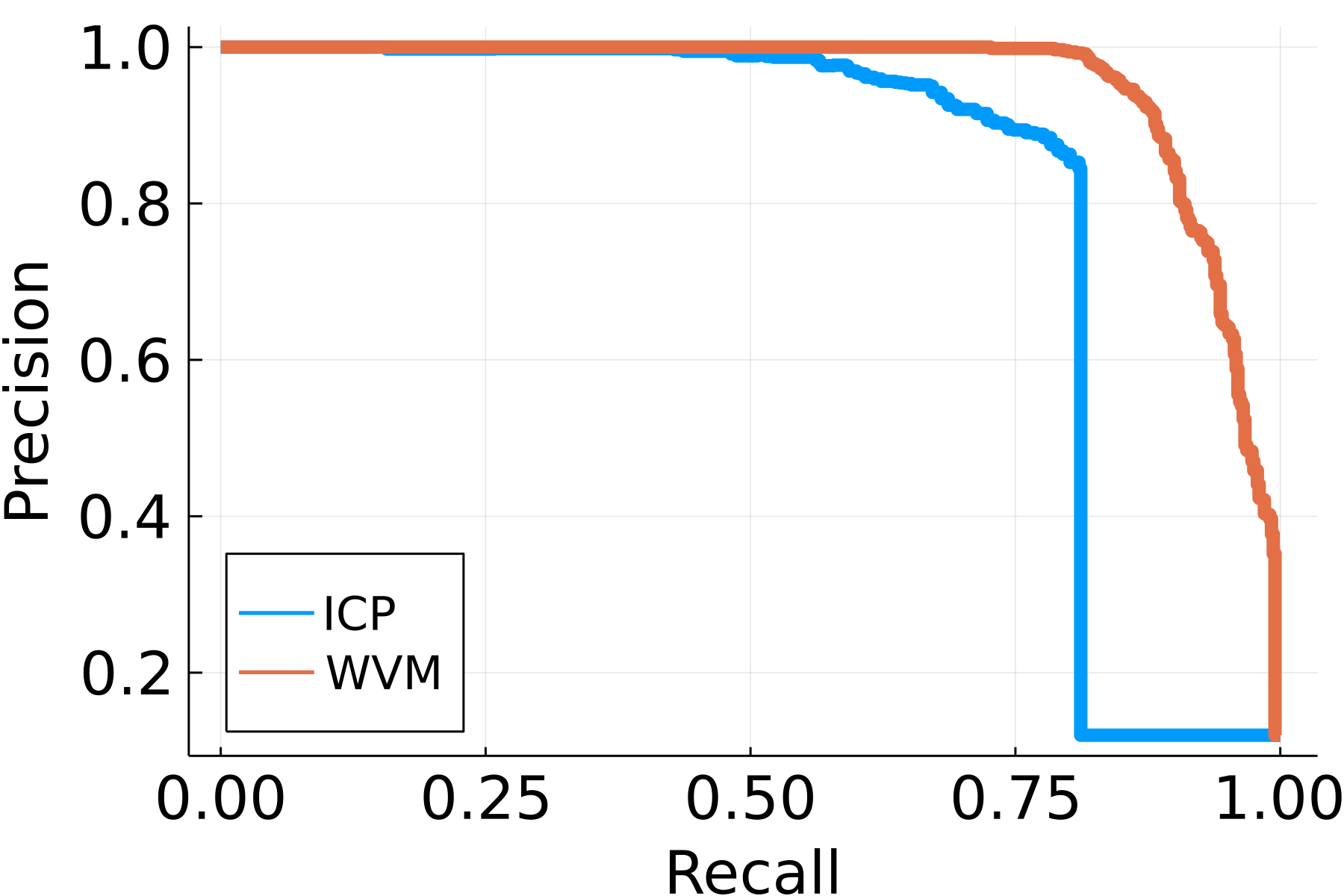}}
\caption{(a): Run time in seconds for different numbers of pre-selected variables for ICP and WVM, averaged over the 100 simulations; we stopped at 18 variables for ICP as it took $>50$ hours for this data-point. (b): Average number of false positives (top) and false negatives (bottom) for different numbers of pre-selected variables. (c): An example of the outputs of WVM with $15$ pre-selected variables. (d): Precision-recall curves for ICP and WVM, averaged over the $100$ simulations.}
\label{fig:wvm-vs-icp}
\end{figure*}

\begin{figure}[!b]

   \subfloat[\label{fig:error-ratio}]{%
      \includegraphics[width=0.24\textwidth]{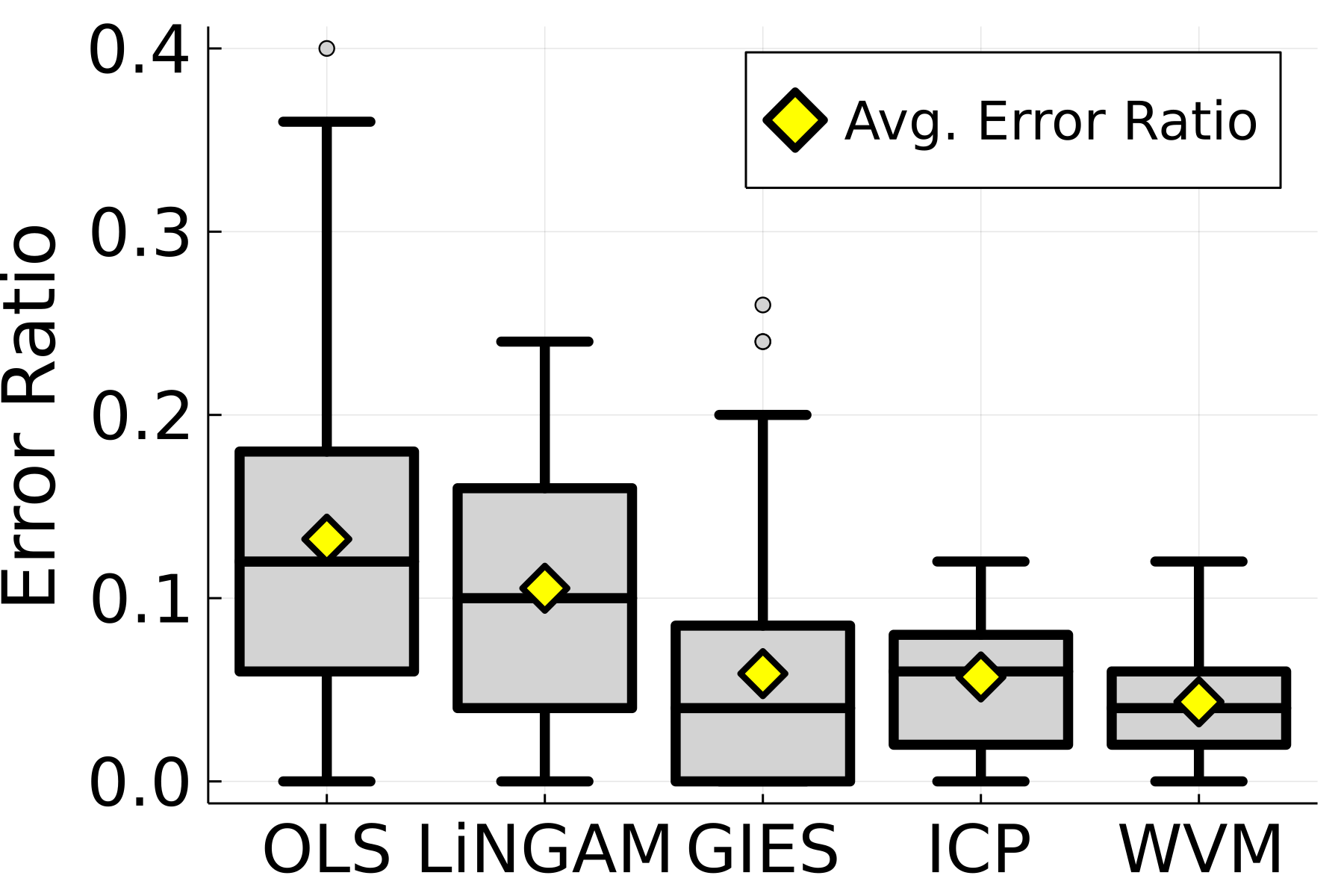}}
   \subfloat[\label{fig:fprs} ]{%
      \includegraphics[width=0.24\textwidth]{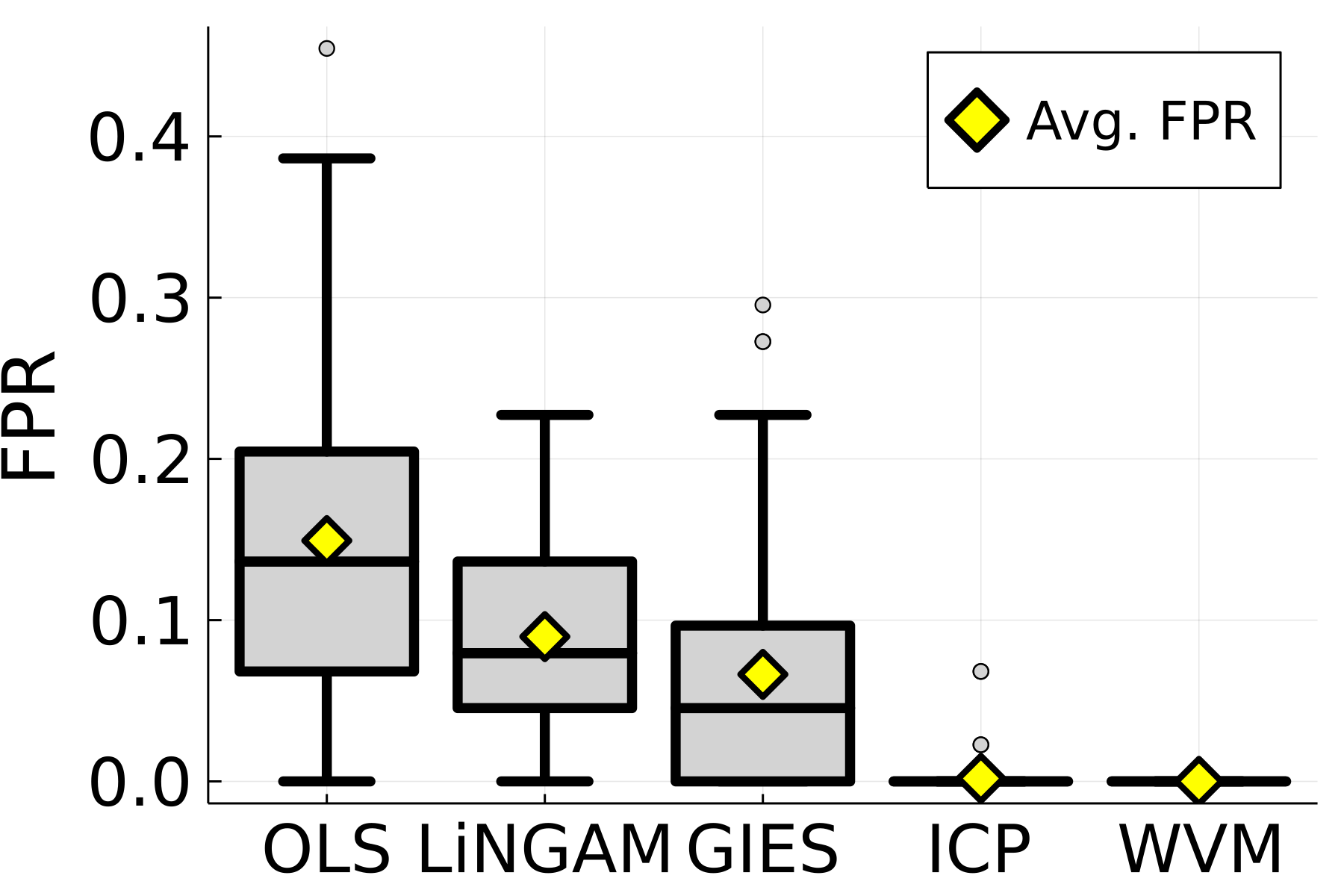}}
\caption{(a): Error ratios of the algorithms. (b): False positive rates (FPR) of the algorithms.}

\end{figure}

\paragraph{Bootstrap Approximation.} Since the statistics \(\hat{\minWV}_{\weights }(\fctClass_{-\predIdx}) \) 
result from a minimization, the thresholds based on the 
asymptotic distribution in Equation~\eqref{eq:limitVariable} may be too conservative in finite samples.
Instead, we find empirically that using a bootstrap estimate of the expectation and variance of 
\(\hat{\minWV}_{\weights}(\fctClass)\) to set the Gamma distribution leads to a better and less conservative threshold. 
We use this heuristic in our experiments.

\section{EXPERIMENTS} \label{sec:experiments}


We now analyze the performance of the WVM algorithm
compared to related algorithms in simulations.
Additional experiments and details are in Appendix \ref{app:fullDetailSim}. The code
to reproduce all experiments is available at \url{https://github.com/astrzalk/WVM_reproducibility}.

\paragraph{Data Generating Process.}
We focus on the case where the causal model in Equation~\eqref{eq:invariantSEM} is linear, 
i.e., \(\target{\env} = \beta^{*T}\pred{\env} + \noise^\env\), where 
\(\beta^*_k = 0\) for \(k \notin S^*\).
In other words, we consider \(\fctClass = \{\fct : \fct(x; \beta) =\beta^T x, \, \beta \in \mathbb{R}^p\}\)
and \(\fctClass_{-\predIdx}\) consists of all functions \(\fct(\,\cdot\,; \beta)\) from $\fctClass$
such that \(\beta_\predIdx = 0\).
For our simulations, we use 
linear SCMs with independent Gaussian 
noise (i.e., no hidden confounders) for the observations. 
We sample 100 random graphs with 51 variables (\(p = 50\))
with average degree 12, and we fix for all graphs the number of direct
causes to be \(|S^*| = 6\). 
For each of these settings, the graph coefficients and noise variances are randomly sampled from uniform distributions.
We generate four interventional environments by applying simple mechanism change interventions on 
random subsets of the variables (excluding the target variable).
This leads to a total of \(E = 5\) environments. 
For each environment, we generate \(n_\env = 500\) 
i.i.d.~samples resulting in \(n = 2500\) samples in total across all environments. 


\paragraph{Benchmarking.} We compare WVM with naive OLS regression and three baseline causal discovery algorithms: LiNGAM \citep{direct-lingam}, GIES~\citep{gies}, and ICP \citep{icp}. 
ICP and WVM both use confidence level $\confLevel = 0.1$;
the inferred direct causes for OLS are the significant predictors at level \(\confLevel / \numpred\). 
As in ICP, we preselect variables using Lasso before applying WVM to improve its power. However we are not constrained by computation and can preselect as many predictors as desirable. 
We preselect $18$ variables for this experiment while ICP fixes the total number of preselected variables at $8$.
Furthermore, LiNGAM is applied on the aggregated
dataset across environments, and we specify that all non-target variables 
are intervened on for GIES.
Across methods, let \(\hat{S} \subset [p]\) denote the inferred direct causes for the target. Define the false positives as
\(FP \doteq \{\predIdx \in [\numpred] : \predIdx \in \hat{S} \;\text{and} \;\predIdx \notin \causalPred\}\), and similarly the false negatives as 
\(FN \doteq \{\predIdx \in [\numpred] : \predIdx \notin \hat{S} \;\text{and} \;\predIdx \in \causalPred\}\). 
To evaluate the performance of causal discovery algorithms, we use the \emph{Error Ratio}  \(\doteq (|FP| + |FN|)/p\) and 
the false positive rate \(FPR \doteq |FP| / (\numpred - |S^*|)\).
We also consider the \emph{Precision}  \(\doteq 1 - |FP| / |\hat{S}|\) and the \emph{Recall}  \( \doteq 1 - |FP| / |S^*|\).
WVM outperforms ICP and the other algorithms in terms of the error ratio (Figure~\ref{fig:error-ratio}), and behaves similarly to ICP in terms of the false positive rate (Figure~\ref{fig:fprs}).


\paragraph{Further Comparison between WVM and ICP.}
We now investigate the run time and power of WVM and ICP for different numbers of preselected variables.
For moderate to large
numbers of preselected variables, ICP's exponential scaling is much slower than WVM's runtime (Figure~\ref{fig:icp-wvm-time}).
For instance, with 
18 preselected variables, ICP takes $2443$s on average while WVM takes only $17$s, a 
$100$ times speed-up.
WVM's power is also less sensitive to the number of preselected variables (Figure \ref{fig:fps-fns-wvm-icp-diff-vars}),
and
WVM identifies 1 to 2 (out of 6) more causes on average than ICP when applied on the same set of preselected variables.
Even though Equation \eqref{eq:WVMidentifPred} suggests for infinite data WVM may be less powerful than ICP, for finite samples the converse is often true, since ICP’s output is an intersection of exponentially many accepted sets of potential causes.



The statistics \(\hat{\minWV}_{\weights}(\fctClass_{-\predIdx})\) returned by WVM are good indicators of the ``strength'' of a potential cause, and in many situations it is possible to identify causal predictors by looking at these values (Figure \ref{fig:wvs-pvals-example}). We further analyze the potential of these statistics to recover direct causes by looking
 at the precision-recall curve constructed for different choices of thresholds.
In particular, WVM often recovers more causes than ICP with higher precision (Figure \ref{fig:pr-curve-default-sim}).

\begin{figure}[t]
\vskip -0.1in
   \subfloat[\label{fig:pr-curve-small-sample}]{%
      \includegraphics[width=0.24\textwidth]{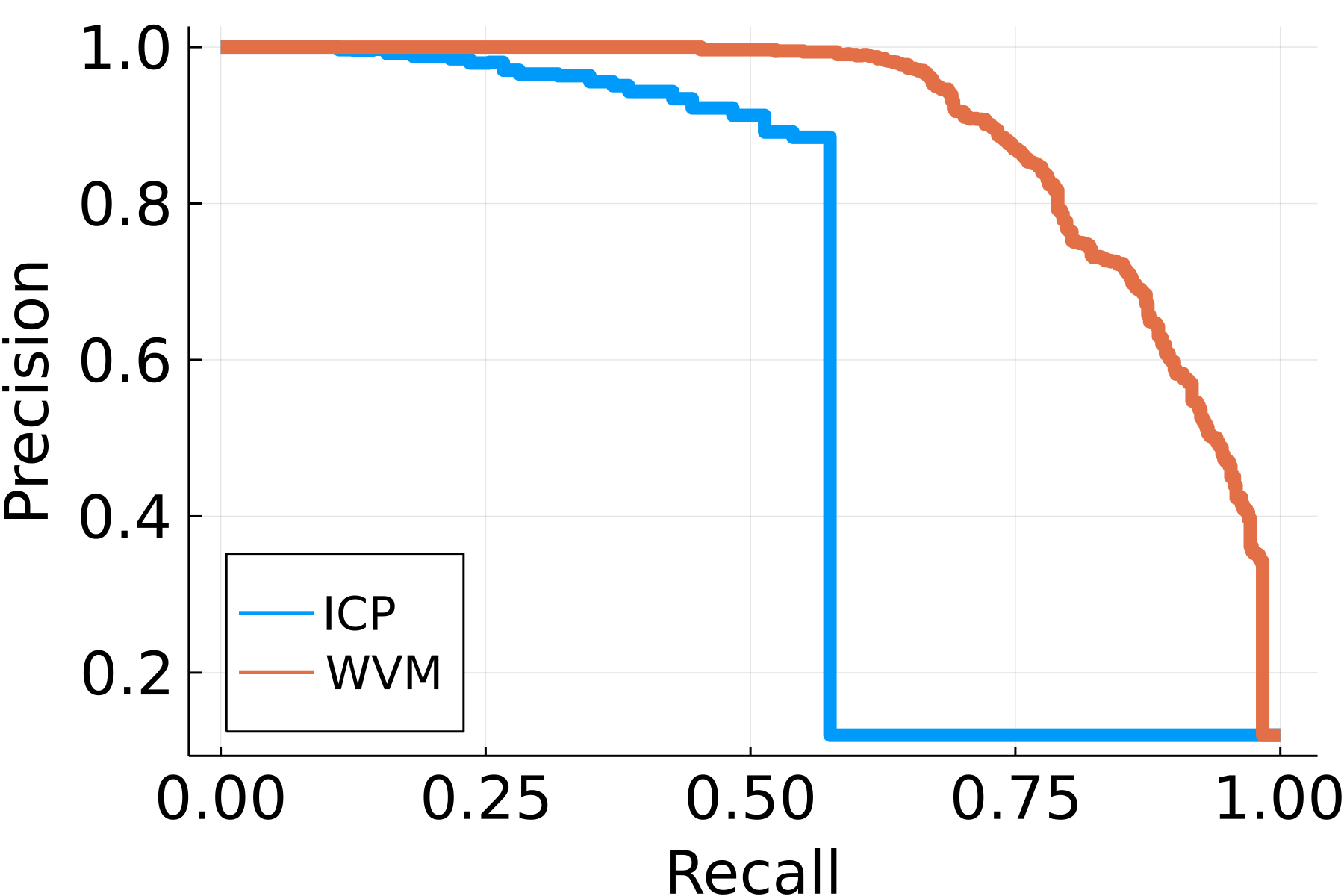}}
   \subfloat[\label{fig:pr-curve-more-causes} ]{%
      \includegraphics[width=0.24\textwidth]{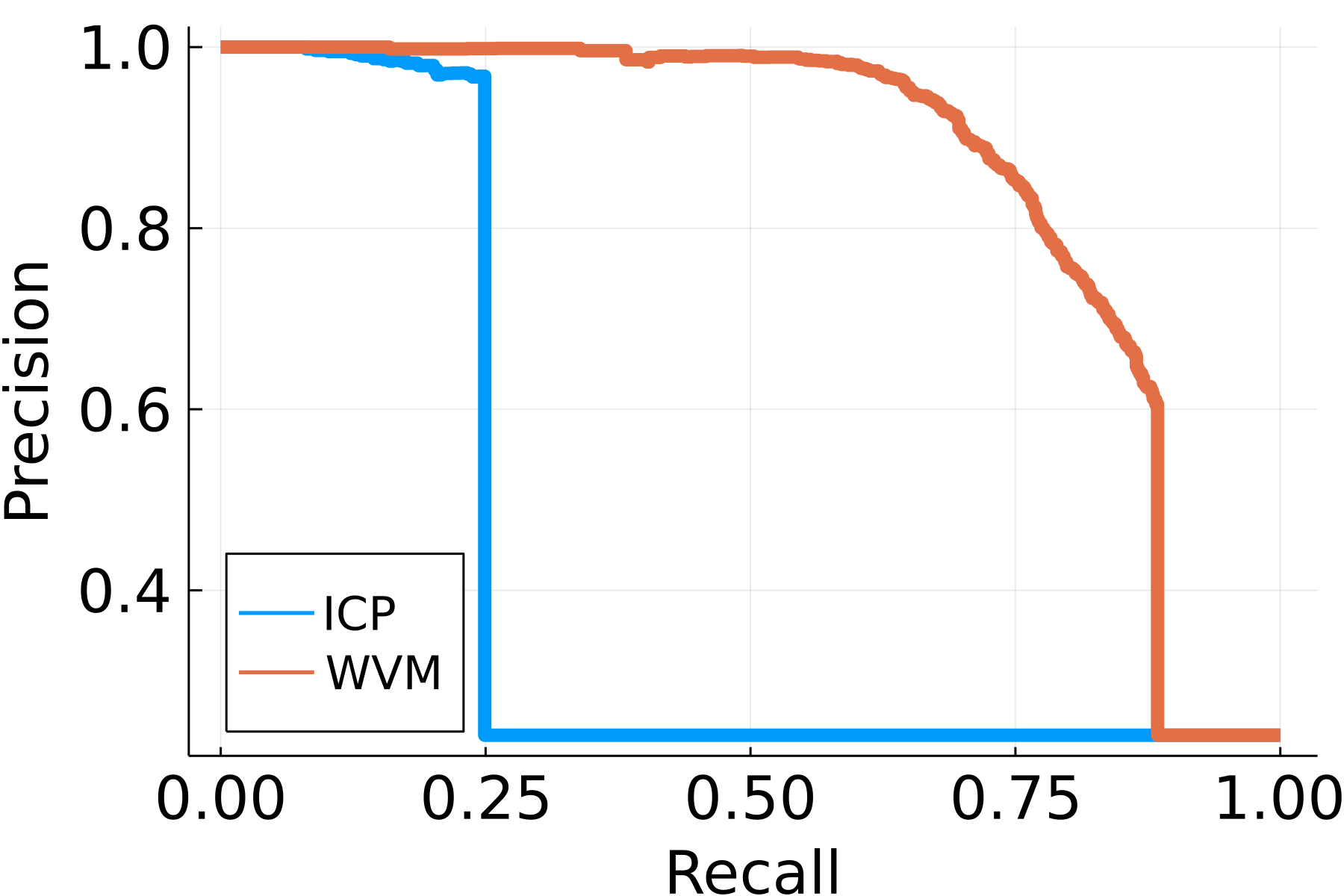}}
\caption{Average precision-recall curves of ICP and WVM when (a): \(n_e = 100\); and 
(b): $|\causalPred|=12$.}
\end{figure}

\paragraph{Additional Settings.} We consider two additional settings, one where the sample size is reduced to \(n_e = 100\), and another where the number of direct causes is set to \(|\causalPred| = 12\).
The average precision-recall curves for ICP and WVM again show that WVM recovers more causes than ICP with higher precision (Figures \ref{fig:pr-curve-small-sample} and \ref{fig:pr-curve-more-causes}).

The advantages of WVM over ICP are more prominent in 
these situations.
When 
\(|S^*| = 12\), ICP's performance quickly deteriorates compared to WVM (Figure~\ref{fig:pr-curve-more-causes}).
This scenario is of interest since ICP may often require fewer than $|S^*|$ preselected variables.
This shows that ICP's computational complexity constrains its statistical power, and that WVM's practical improvement over ICP is more than run time.





\vspace{-0.05in}
\section{DISCUSSION} \label{sec:conclusion}
\vspace{-0.1in}

In this paper we show that causal inference using ICP may be reformulated as a multiple hypothesis testing problem with only $\numpred$ tests to perform, compared to the $2^\numpred$ tests that the original ICP requires.
Each of those tests is similar to a likelihood ratio test, where the negative log likelihood is replaced by a new loss function that we call Wasserstein variance, which quantifies the distributional variability of the residuals across environments.
WVM is nonparametric and can easily adapt to more general settings than ICP (see remarks after Definition \ref{def:WVMindentifPred} and Appendix~\ref{app:generalSetting}).
We derived asymptotic guarantees on the ability of WVM to recover the direct causes with a limited number of false positives, and our simulations confirm our theoretical results.

There are possible improvements and extensions that we leave for future work.
In practice, the thresholds based on our asymptotic results and bootstrap approximation may sometimes be conservative.
Therefore, deriving a more accurate limit distribution for the statistics $\hat{\minWV}_{\weights }(\fctClass_{-\predIdx})$ under $\Tilde{\hyp}_{0, \predIdx}(\envSet)$ for some specific classes of functions is of interest.
We would like to stress however that under our rather weak assumptions on the class of functions $\fctClass$, the asymptotic distribution in \eqref{eq:limitVariable} is the best achievable limit distribution -- when $\fctClass_{-k}$ is finite and $\fct^*$ is identifiable, the distributions of $\hat{\minWV}_{\weights }(\fctClass_{-\predIdx})$ and \eqref{eq:limitVariable} coincide asymptotically under $\Tilde{\hyp}_{0, \predIdx}(\envSet)$.
Finally, the primary assumption that WVM relies on is the additive noise specification from Equation \eqref{eq:invariantSEM}.
Even though additive noise models are used by many causal discovery algorithms, ICP included, such an assumption may be too restrictive in some situations.
Adapting WVM to more general functional relationships with nonadditive noise is another question left for future work.

\bibliographystyle{plainnat}
\bibliography{refs.bib}


\clearpage
\appendix

\thispagestyle{empty}

\onecolumn \makesupplementtitle

\section{PROOF OF LEMMA \ref{lem:reformulation}} \label{app:reformulation}

Recall that Assumption \ref{invar-assump} implies that $\hyp_{0, \causalPred}(\envSet)$ is true. Equation \eqref{eq:reformulation} then results directly from the following chain of logical equivalences:
\begin{align*}
    \predIdx \in \identifPred \;\; \Longleftrightarrow \;\; & \forall \predSubset \subseteq [\numpred] \; \text{ s.t. } \; \hyp_{0, \predSubset}(\envSet) \text{ is true, } \; \predIdx \in \predSubset \\
    \Longleftrightarrow \;\; & \nexists \predSubset \subseteq [\numpred] / \{\predIdx\} \; \text{ s.t. }  \; \hyp_{0, \predSubset}(\envSet) \text{ is true} \\
    \Longleftrightarrow \;\; & \hyp_{0, \predIdx}'(\envSet) \text{ is false.} 
\end{align*}

\section{A MORE GENERAL SETTING} \label{app:generalSetting}
\cite{icp} propose some extensions of ICP to settings where Assumption \ref{invar-assump} is violated (e.g.,~we refer to Section 5 in their paper).
In particular, they consider the case of a general form of SCM that allows for the presence of hidden confounders and feedback loops between the target and the causal predictors, and only imposes that the environment variable $I \in [\numenv]$ acts as an instrumental variable on the predictors $X$ (see Figure \ref{fig:general-scm} below).
We show in this section that WVM is directly applicable in this setting, whereas ICP's extension is computationally intractable.

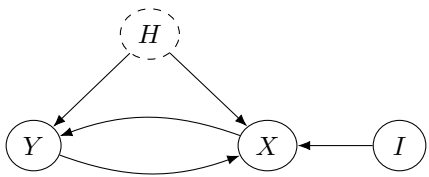
\begin{figure}[!h] 
\centering
\begin{tikzpicture}
    \node[state, dashed] (H) {$H$};
    \node[state] (X) [below right = of H] {$X$};
    \node[state] (Y) [below left = of H] {$Y$};
    \node[state] (I) [right =of X] {$I$};

    \path (X) edge[bend right=20] (Y);
    \path (Y) edge[bend left=-20] (X);
    \path (I) edge (X);
    \path (H) edge (X);
    \path (H) edge (Y); 
\end{tikzpicture}
\caption{An SCM with an unobserved confounder, node \(H\) in the graph, with a feedback cycle between \(X\) and \(Y\), and where the instrumental variable $I \in [\numenv]$ indicates from which environment an observation is drawn.}
\label{fig:general-scm}
\end{figure}

The SCM illustrated in Figure~\ref{fig:general-scm} contains a 
hidden confounder $H$ and a feedback loop between variables \(X\) and \(Y\). 
Concretely, the underlying system of structural equations for the above SCM is: 
\begin{equation}\label{eqn:hidden-sem}
    \left\{\begin{array}{ll}
        X = h(I, H, Y, \eta),  &\\
        Y = f^*(X) + g(H, \noise), &
    \end{array}\right.
\end{equation}
where \(f^* \in \mathcal{F}_{S^*}\), and $H, I, \epsilon, \eta$ are mutually independent.

Note that Assumption~\ref{eq:invariantSEM} no longer holds for the structural equations \eqref{eqn:hidden-sem} as 
the residual $g(H, \noise)$ is no longer independent of the covariates.
To deal with this more general setting~\cite{icp}
introduced a relaxed null hypothesis that removes the assumption of 
the independence of the noises from the covariates:
\begin{equation*} 
    \hyp_{0, \predSubset, hidden}(\envSet):
    \left\{\begin{array}{ll}
    \exists f \in \mathcal{F}_{S} 
    \text{ such that the distribution of } 
    Y^e - f(X^e) \text{ is identical for all } e \in \mathcal{E}.&
    \end{array}\right.
\end{equation*}
$\hyp_{0, \causalPred, hidden}(\envSet)$ is true under model \eqref{eqn:hidden-sem}, since the environment variable $I$ is independent of $H$ and $\epsilon$. Given this new weaker null hypothesis, they propose to recover the set \(S_H(\mathcal{E})\) of identifiable causal predictors under model \eqref{eqn:hidden-sem} defined as:
\begin{equation}
    S_H(\mathcal{E}) \doteq 
    \bigcap_{S : \; \hyp_{0, \predSubset, hidden}(\envSet) \;is \;true} S \subseteq S^*.
\end{equation}
It turns out however that it is computationally challenging to test for each hypothesis $\hyp_{0, \predSubset, hidden}(\envSet)$.
The main reason is that we can no longer use regression techniques to recover the residuals since they are dependent on the covariates.
The only solution \cite{icp} propose for testing $\hyp_{0, \predSubset, hidden}(\envSet)$ is to go through all functions in $\fctClass_{S}$ (or over some approximating grid of it) and to check if for at least one of them the resulting residuals have an invariant distribution across environments.
Such an approach is of course quite intractable in practice.

On the other hand, WVM can recover without any modification, and thus in a tractable way, the set $S_H(\mathcal{E})$ from the data.
To see this, first define \(H'_{0,k, hidden}(\mathcal{E})\) as follows:
\begin{equation}
    H'_{0,k,hidden}(\mathcal{E}) :
    \left\{\begin{array}{ll}
        \exists S \notni k, \exists f \in \mathcal{F}_S \text{ and a fixed distribution } D \text{ s.t. } &\\
        \text{ for all } e \in \mathcal{E}, Y^e - f(X^e) \sim D.& 
    \end{array}\right.
\end{equation}
Note by the 
same reasoning as the proof for Lemma~\ref{lem:reformulation}, we have that
\(S_H(\mathcal{E}) = \{k : H'_{0,k,hidden}(\mathcal{E}) \;is \;false\}\).
The important observation is that, by Lemma \ref{lem:noVariability}, \(H'_{0,k,hidden}(\mathcal{E})\) is 
false if and only if \(\tilde{H}_{0,k}(\mathcal{E})\) is false.\footnote{As a point of rigor, this equivalence might not be true for some classes of functions $\fctClass$. Indeed, it is technically possible to have $\minWV_{\weights }(\fctClass_{-\predIdx}) = 0$ while $\wassVar{\probVectBase(\fct)}{\weights} > 0$ for any $\fct \in \fctClass_{-\predIdx}$ for some $\predIdx$. Note that when this happens, the set of identifiable causal predictors for WVM will be smaller than $S_H(\mathcal{E})$, similarly to what we saw in Definition \ref{def:WVMindentifPred}. However, we believe that for the usual classes of functions one encounters in practice this equivalence holds. This is true for instance when the predictors are bounded and $\fctClass$ is a class of linear functions with bounded coefficients; then in that case $\inf_{\fct \in \fctClass_{-\predIdx}}  \wassVar{\probVectBase(\fct)}{\weights}$ is in fact a minimum -- this is a direct consequence of the ``continuity" of the Wasserstein variance as expressed in Lemma \ref{lem:th2Lem1}.} 
Therefore, the sets of identifiable causes for the WVM algorithm and for the above extension of ICP are exactly the same, that is
\(\tilde{S}(\mathcal{E}) = S_H(\mathcal{E})\); compared with this extension of ICP however, WVM is much more computationally efficient.

\section{AN EXTENSION OF THE WVM TEST TO BLOCKS OF VARIABLES}\label{app:extensionWVM}

It is possible to extend WVM to detect whether there is a direct cause among a set $S$ of several predictors, instead of testing for each of the predictors separately using $\minWV_{\weights}(\fctClass_{-k})$.
Let $\fctClass_{-S}$ denote the set of functions in $\fctClass$ that don't depend on any of the predictors from $S$ and consider the following hypotheses:
\[
\Tilde{\hyp}_{0, S}(\envSet) : \;\; \minWV_{\weights }(\fctClass_{-S}) = 0, \quad \text{against} \quad \Tilde{\hyp}_{1, S}(\envSet) : \;\; \minWV_{\weights }(\fctClass_{-S}) > 0.
\]
From Assumption \ref{invar-assump} (or even under the more general setting considered in Section \ref{app:generalSetting}), we have that $\minWV_{\weights }(\fctClass_{-S}) = 0$ whenever the set $S$ does not include any of the direct causes. 
Therefore, if we observe with enough confidence that $\minWV_{\weights }(\fctClass_{-S}) > 0$ then we can conclude that $S$ contains at least one direct cause.
Such a test is apparently not possible within the framework of ICP since the independence property of the noise from Assumption \ref{invar-assump} can easily be violated when we group variables together; for instance, when $S$ includes a variable dependent on the residual from \eqref{eq:invarianceHyp} (e.g.~a descendant of $Y$).

We can push this extension of WVM further by considering a partition $\mathcal{P} \doteq \{S_1, \ldots, S_m \}$ of the $p$ predictors and in the same spirit as Equation \eqref{eq:WVMidentifPred}, we can seek to recover the collection of identifiable blocks of variables containing at least one cause:
\begin{equation} \label{eq:WVMidentifBlocks}
    \Tilde{S}_{\mathcal{P}}(\envSet) \doteq \left\{ \, S_i :  \Tilde{\hyp}_{0, S_i}(\envSet) \text{ is false}, i \in [m] \right\}.
\end{equation}
Grouping variables and testing with WVM in such a way can be beneficial in situations where some of the variables are highly correlated.
To see this, consider the situation where a predictor $k_1$ is a direct cause and another predictor $k_2$ is highly correlated with $k_1$.
In this case, $\minWV_{\weights}(\fctClass_{-k_1})$ can be equal to $0$, or close to it, since $k_1$ in the regression can easily be substituted by $k_2$.
On the other hand, by grouping them in a set $S$ we might have $\minWV_{\weights }(\fctClass_{-S})$ large enough to detect that at least one of them is causal.
Therefore, using such an extension may potentially recover more information about the causal structure of the data when some of the predictors are highly correlated.

This issue also was discussed for ICP in \cite{nonlinear-icp}.
Indeed, ICP can return an empty set in the presence of highly correlated variables for the same reason discussed above.
The authors propose to change the output of ICP so that it includes defining sets (see Section 2.2 in \cite{nonlinear-icp}), in which at least one variable is a direct cause with high probability.
However, the concept of a defining set is hard to translate to the WVM algorithm.
Instead, in the situation where some variables are highly correlated, we propose to first group the predictors into a collection $\mathcal{P}$ of clusters of highly correlated variables (where such clusters can potentially contain only one predictor) and then use WVM as mentioned above to recover the set $\Tilde{S}_{\mathcal{P}}(\envSet)$ from Equation~\eqref{eq:WVMidentifBlocks}.

\section{PROOF OF THEOREM \ref{thm:unifBound}} \label{app:proofUnifBound}

In order to prove Theorem 1, we use the dual formulation of the Wasserstein barycenter optimization problem from  \cite{barycenter} as an alternative expression for the Wasserstein variance.
Recall from Definition \ref{def:WV} that the optimal value of this optimization problem is simply what we call the Wasserstein variance.
As we shall see below, this formulation will be useful in our derivation of Theorem \ref{thm:unifBound}.
The following result is an adaptation of Proposition 2.2 form \cite{barycenter}:
\begin{proposition}[Proposition 2.2 from \cite{barycenter}] \label{prop:dualWV}
 Define $\sqrBdFctClass(\mathbb{R}) \doteq \left\{ \contFct \in \contFctSet : \frac{\contFct}{1 + |\, . \, |^2}\text{ is bounded }\right\} $, where $\contFctSet $ is the set of continuous functions defined on $\mathbb{R}$. Let $\probVectBase = (\probBase_i)_{i = 1}^{\numenv}$ be probability distributions from $\probSpace$ and $\weights = (\weight_i)_{i=1}^{\numenv} \in \weightSpace$ some weights. Then the Wasserstein variance (Definition \ref{def:WV}) admits the following dual formulation:
 \begin{equation} \label{eq:dualWV}
     \wassVar{\probVectBase}{\weights} = \sup \left \{ \sum_{\env = 1}^{\numenv} \int \dualFunctional{\weight_\env} \contFct_\env d\probBase_\env : \sum_{\env = 1}^{\numenv} \contFct_\env = 0, \, \contFct_\env \in \sqrBdFctClass(\mathbb{R})\right \},
 \end{equation}
 where 
 \begin{equation} \label{dualFunctionalDef}
     \dualFunctional{\weight_\env} \contFct(x) \doteq \inf_{y \in \mathbb{R}} \left\{ \weight_\env | x - y |^2 - \contFct(y) \right\}, \, \forall x \in \mathbb{R}, \contFct \in \sqrBdFctClass(\mathbb{R}), \weight_\env > 0.
 \end{equation}
\end{proposition}

One of the main advantages of using the dual formulation of Equation \eqref{eq:dualWV} is that it expresses the Wasserstein variance almost as sum of expectations; the only difference is of course the supremum over a subset of $\sqrBdFctClass$ in front of it.
In general, available tools to derive uniform bounds in the spirit of Theorem \ref{thm:unifBound} are essentially meant for loss functions that can be expressed as an expectation of a penalty term, therefore the main difficulty here is the presence of the supremum.
We show however that, by a chaining argument and an adaptation of Massart Lemma, this supremum will add only a $O(\log(\nObs{\env}))$ factor in front of the Rademacher complexity compared to classical uniform bounds \citep{shalev2014understanding}.

It is also possible to derive a uniform bound by using the explicit formulation of the Wasserstein variance given in Equation \eqref{eq:explicitWV} instead of the dual formulation in Equation \eqref{eq:dualWV}. 
However, by using this approach in a first attempt we obtained a bound that was slightly worse with a higher power for the $\log$-factor; and the proof was essentially using similar steps and wasn't necessarily shorter.
More importantly, the proof we provide based on \eqref{eq:dualWV} can be easily adapted to situations where the target is multi-dimensional, while \eqref{eq:explicitWV} can be used only when $\targetBase$ is one-dimension.

Finally, we prove Theorem \ref{thm:unifBound} by assuming only that the data are independent (more precisely, i.i.d.) \emph{within} each environment but not necessarily \emph{across} environments; that is, for every $\env \in [\numenv]$ we assume that the data $(\predData_\env, \targetData_\env)$ are i.i.d.~but not necessarily that $(\predData_\env, \targetData_\env)$ is independent of $(\predData_{\env'}, \targetData_{\env'})$ for another environment $\env'$.
This means that our bound will also hold in situations where each environment is created by splitting an original observational data set, and where the same observations can appear in different environments; doing so might be useful for instance to increase the number of observations by environment, and thus obtain better bounds.

Before starting our proof, recall the definition of the Rademacher complexity (e.g.,~see \cite{shalev2014understanding}):
\begin{definition}[Rademacher complexity] \label{def:radComp}
Let $\env \in [\numenv]$ and $\radVarVect = (\radVar_{i})_{i=1}^{\nObs{\env}}$ be independent Rademacher variables. For a fixed data $\predData_\env$ we define the empirical Rademacher complexity of a class of functions $\subFctClass$ to be:
\[
\empRadComp_{\predData_\env}(\subFctClass) \doteq \frac{1}{\nObs{\env}} \mathbb{E}_{ \radVarVect} \left[ \sup_{\subFct \in \subFctClass} \sum_{i = 1}^{\nObs{\env}} \radVar_{i} \subFct(\predObs_{i}^\env)\right].
\]
The Rademacher complexity for $\subFctClass$ and environment $\env$ is then defined as:
\[
\radComp_{\nObs{\env}}(\subFctClass) \doteq  \mathbb{E}_{\predData_{\env}} \left[ \empRadComp_{\predData_\env}(\subFctClass)\right].
\]

\end{definition}

\subsection{Short Discussion of the Assumptions} \label{app:shortDiscussTh1}

We briefly discuss the assumptions that the variables $\maxVarEnv{\env} \doteq \sup_{\subFct \in \subFctClass} | \subFct(\predBase^{\env}) - \targetBase^{\env} |$ are either sub-Gaussian or bounded.
We argue that such assumptions are not particularly restrictive.
For instance, in practice it is reasonable to assume that there is a large enough constant $\maxConst>0$ (potentially very large) such that all variables $\predBase^\env$ and $\targetBase^\env$ are bounded (in absolute value) by $\maxConst$; also that for a reasonable choice for $\subFctClass$, $\subFct(\predBase^\env)$ is uniformly bounded with probability one.
Under that scenario the $\maxVarEnv{\env}$s are therefore bounded with probability one.
The main reason we consider the weaker sub-Gaussian assumption is to include the possibility of data generated by a linear Gaussian SCM, a model that is often used in causal inference; in that case if $\subFctClass$ is a class composed of linear functions with bounded norm, then the sub-Gaussian assumption holds for the $\maxVarEnv{\env}$'s -- this fact, in addition to the sub-Gaussianity of the next example, can be shown by using point (II) of Theorem 2.1 from \cite{wainwright2019high}.
Note that nonlinear models are also possible in that case: For instance, if one takes $\subFctClass$ to be a bounded subset of an RKHS, and that the related kernel is also bounded with probability one w.r.t.~$\predBase^{\env}$, then if $\targetBase^\env$ is sub-Gaussian we have $\maxVarEnv{\env}$ sub-Gaussian too.

\subsection{First Steps} \label{app:firstSteps}

To prove the bound from equation \eqref{eq:unifBound} we derive upper-bounds for both $ \wassVar{\empProbBaseVect(\subFct)}{\weights} - \wassVar{\probVectBase(\subFct)}{\weights}$ and $ \wassVar{\probVectBase(\subFct)}{\weights} - \wassVar{\empProbBaseVect(\subFct)}{\weights}$ separately as they need (slightly) different steps. 
First, we start by bounding the former term uniformly, we then focus on the latter one.
Furthermore, in order to improve the exposition of Theorem \ref{thm:unifBound}'s proof, we often use directly some technical results as lemmas and postpone their proofs to Section \ref{app:suppLem1}.

Let $\empMax = \empMax(\predData, \targetData) \doteq \max_{\env \in [\numenv]} \max_{i \in [\nObs{\env}]} \sup_{\subFct \in \subFctClass} | \subFct(\predObs_i^\env) - \targetObs_i^\env|$, and denote $\boundedContFctSet{\empMax} \doteq \{ \contFct \in \contFctSet : \forall x \in \ball{0}{\empMax}^{c}, \contFct(x) = \contFct(\empMax x / |x|) \}$, where $\ball{0}{\empMax}$ refers to the ball (or interval, as we are in $\mathbb{R}$) of center $0$ and radius $\empMax$. 
In other words, $\boundedContFctSet{\empMax}$ is the set of continuous functions defined on $\mathbb{R}$ that are constant on $(-\infty, - \empMax]$ and on $[ \empMax, +\infty)$.
We prove in Lemma \ref{lem:dualForBndSuppDist} that the dual formulation of $\wassVar{\empProbBaseVect(\subFct)}{\weights}$ can be written as a supremum over $\boundedContFctSet{\empMax}$ instead of $\sqrBdFctClass(\mathbb{R})$; hence by Lemma \ref{lem:dualForBndSuppDist} we have:
\begin{align*}
    \circled{1} \doteq \sup_{\subFct \in \subFctClass} \left( \wassVar{\empProbBaseVect(\subFct)}{\weights} - \wassVar{\probVectBase(\subFct)}{\weights} \right) = & \sup_{\subFct \in \subFctClass} \left(\sup_{
          \substack{\contFct_\env \in \boundedContFctSet{\empMax}, \\
          \sum_{\env}\contFct_\env = 0}
    } \sum_{\env=1}^{\numenv} \int \dualFunctional{\weight_\env} \contFct_\env d \empProbBase_{\env}(\subFct) - \sup_{
          \substack{\contFct'_\env \in \sqrBdFctClass, \\
          \sum_{\env}\contFct'_\env = 0}
    } \sum_{\env=1}^{\numenv} \int \dualFunctional{\weight_\env} \contFct'_\env d \probBase_{\env}(\subFct) \right) \\
    \leq & \sup_{\subFct \in \subFctClass} \sup_{
          \substack{\contFct_\env \in \boundedContFctSet{\empMax}, \\
          \sum_{\env}\contFct_\env = 0}
    } \sum_{\env=1}^{\numenv} \int \dualFunctional{\weight_\env} \contFct_\env d (\empProbBase_{\env}(\subFct) - \probBase_\env(\subFct) ).
\end{align*}
We can also write:
\[
\sum_{\env=1}^{\numenv} \int \dualFunctional{\weight_\env} \contFct_\env d \probBase_\env(\subFct) = \mathbb{E}_{(\predData', \targetData')} \left[ \sum_{\env=1}^{\numenv} \int \dualFunctional{\weight_\env} \contFct_\env d \empProbBase'_\env(\subFct) \right],
\]
where $(\predData',\targetData')$ is virtual data drawn from the exact same distribution as $(\predData, \targetData)$, and $\empProbBase'_\env(\subFct)$ is defined as $\empProbBase_\env(\subFct)$ but with the new data instead.
Furthermore, we will also define $\empMax' = \empMax'(\predData',\targetData')$ as we defined $\empMax$, but again using data $(\predData', \targetData')$ instead of $(\predData, \targetData)$.
We get:
\begin{align} \label{eq:firstUpperBound}
\circled{1} \leq & \sup_{\subFct \in \subFctClass} \sup_{
          \substack{\contFct_\env \in \boundedContFctSet{\empMax}, \\
          \sum_{\env}\contFct_\env = 0}
    } \mathbb{E}_{(\predData', \targetData')} \left[ \sum_{\env=1}^{\numenv} \int \dualFunctional{\weight_\env} \contFct_\env d (\empProbBase_{\env}(\subFct) - \empProbBase'_\env(\subFct) )  \right] \nonumber \\
    \leq & \, \mathbb{E}_{(\predData', \targetData')} \left[ 
    \sup_{\subFct \in \subFctClass} \sup_{
          \substack{\contFct_\env \in \boundedContFctSet{\empMax}, \\
          \sum_{\env}\contFct_\env = 0}
    }  \sum_{\env=1}^{\numenv} \int \dualFunctional{\weight_\env} \contFct_\env d (\empProbBase_{\env}(\subFct) - \empProbBase'_\env(\subFct) ) 
    \right] \nonumber \\
    \leq & \, \mathbb{E}_{(\predData', \targetData')} \left[ 
    \sup_{\subFct \in \subFctClass} \sup_{
          \substack{\contFct_\env \in \boundedContFctSet{\empMax \vee \empMax'}}
    }  \sum_{\env=1}^{\numenv} \int \dualFunctional{\weight_\env} \contFct_\env d (\empProbBase_{\env}(\subFct) - \empProbBase'_\env(\subFct) ) 
    \right].
\end{align}
Since for any $\subFct \in \subFctClass$ both $\empProbBase_\env(\subFct)$ and $\empProbBase'_\env(\subFct)$ are supported on $\ball{0}{\empMax \vee \empMax'}$, by Lemma \ref{lem:lipschitzDual}, $\dualFunctional{\weight_\env} \contFct_\env$ is $4 \weight_\env (\empMax \vee \empMax')$-Lipschitz. 
Furthermore, notice that for any constant $c$, $\int c d(\empProbBase_\env(\subFct)) - \empProbBase'_\env(\subFct)) = 0$; so we can always modify the $\contFct_\env$ functions up by an additive constant to get $\dualFunctional{\weight_\env}\contFct_\env(0) = 0$ without changing the value of the expression inside the supremum.
Finally, we can obviously switch the two '$\sup$'. 
Based on all these remarks, we obtain the following new bound:
\begin{align*}
  \circled{1} \leq & \, 4  \mathbb{E}_{(\predData', \targetData')} \left[ 
    \sup_{\lipFct_\env \in \lipFctClass{\empMax \vee \empMax'}} \sup_{ \subFct \in \subFctClass}
     \sum_{\env=1}^{\numenv} \weight_\env \int \lipFct_\env d (\empProbBase_{\env}(\subFct) - \empProbBase'_\env(\subFct) ) \right] \\
     \leq & \, 4  \mathbb{E}_{(\predData', \targetData')} \left[  \sum_{\env=1}^{\numenv} \weight_\env
    \sup_{\lipFct_\env \in \lipFctClass{\empMax \vee \empMax'}} \sup_{ \subFct \in \subFctClass}
      \int \lipFct_\env d (\empProbBase_{\env}(\subFct) - \empProbBase'_\env(\subFct) ) \right],
\end{align*} 
where we define $\lipFctClass{\empMax \vee \empMax'}$ as the set of all $\empMax \vee \empMax'$-Lipschitz functions defined on $\ball{0}{\empMax \vee \empMax'}$ that are equal to zero at the origin, and that we extend outside of $\ball{0}{\empMax \vee \empMax'}$ the same way we did for functions in $\boundedContFctSet{\empMax \vee \empMax'}$:
\[
\lipFctClass{\empMax \vee \empMax'} \doteq \{ \lipFct: \lipFct(0) = 0, \lipFct \text{ is } (\empMax \vee \empMax')\text{-Lipschitz and } \forall x \notin \ball{0}{\empMax \vee \empMax'}, \lipFct(x) = \lipFct((\empMax \vee \empMax') x / |x| ) \}.
\]
Now we simplify the upper bound by expending it into a sum of expectations that depend only on the observations coming from one of the environments.
Beside improving clarity, since we have to treat each of these expectations separately, another important reason for this step is to allow us derive the bound \eqref{eq:unifBound} under the (weaker) assumption that the observations are only independent within each environment, and not necessarily across environments (see discussion at the beginning of Section \ref{app:proofUnifBound}).

In what follows, let $\maxConst > 0$ be a constant to be chosen later on. We have:
\begin{align} \label{eq:firstUnifBound1}
    \circled{1} \leq & \,4 \mathbb{E}_{(\predData', \targetData')} \left[ \mathbb{1}_{\empMax \vee \empMax' \leq \maxConst} \sum_{\env=1}^{\numenv} \weight_\env
    \sup_{\lipFct_\env \in \lipFctClass{\maxConst}} \sup_{ \subFct \in \subFctClass}
      \int \lipFct_\env d (\empProbBase_{\env}(\subFct) - \empProbBase'_\env(\subFct) ) \right] + 8 \mathbb{E}_{(\predData', \targetData')}\left[ (\empMax \vee \empMax')^2 \mathbb{1}_{\empMax \vee \empMax' > \maxConst} \right] \nonumber \\
      \leq & \, 4  \sum_{\env=1}^{\numenv} \weight_\env \mathbb{1}_{\empMax_\env  \leq \maxConst} \mathbb{E}_{(\predData'_\env, \targetData'_\env)} \underbrace{\left[
    \sup_{\lipFct_\env \in \lipFctClass{\maxConst}} \sup_{ \subFct \in \subFctClass}
      \int \lipFct_\env d (\empProbBase_{\env}(\subFct) - \empProbBase'_\env(\subFct) ) \right]}_{\doteq \bdG{\maxConst}{(\predData_\env,\targetData_\env)}{(\predData'_\env,\targetData'_\env)}} + 8 \mathbb{E}_{\empMax'} \underbrace{\left[(\empMax \vee \empMax')^2 \mathbb{1}_{\empMax \vee \empMax' > \maxConst}\right]}_{\doteq \bdH{\maxConst}{\empMax}{\empMax'}} ,
\end{align}
where we used the fact that for any $\lipFct \in \lipFctClass{\empMax \vee \empMax'}$ we have that $|\lipFct| \leq \maxConst^2$ and for all $\env \in [\numenv]$ let $\empMax_\env \doteq \max_{i \in [\nObs{\env}]} \sup_{\subFct \in \subFctClass} | \subFct(\predObs_i^\env) - \targetObs_i^\env|$. Now we turn to finding an upper bound for:
$$
\circled{2} \doteq \sup_{\subFct \in \subFctClass} \left( \wassVar{\probVectBase(\subFct)}{\weights} - \wassVar{\empProbBaseVect(\subFct)}{\weights} \right).
$$
We define $(\predData',\targetData')$, $\empProbBase'_\env(\subFct)$ and $\empMax'$ the same way as we did above. We then get, for any fixed $\subFct \in \subFctClass$:
$$
\wassVar{\probVectBase(\subFct)}{\weights} =  \sup_{
          \substack{\contFct_\env \in \sqrBdFctClass, \\
          \sum_{\env}\contFct_\env = 0}
    } \mathbb{E}_{(\predData', \targetData')} \left[ \sum_{\env=1}^{\numenv} \int \dualFunctional{\weight_\env} \contFct_\env d \empProbBase'_\env(\subFct) )  \right] \leq  \mathbb{E}_{(\predData', \targetData')} \left[\sup_{
          \substack{\contFct_\env \in \sqrBdFctClass, \\
          \sum_{\env}\contFct_\env = 0}
    }  \sum_{\env=1}^{\numenv} \int \dualFunctional{\weight_\env} \contFct_\env d \empProbBase'_\env(\subFct) )  \right],
$$
which in turn means that:
\begin{align*}
    \circled{2} \leq & \, \mathbb{E}_{(\predData', \targetData')} \left[\sup_{\subFct \in \subFctClass} \left( \sup_{
          \substack{\contFct_\env \in \sqrBdFctClass, \\
          \sum_{\env}\contFct_\env = 0}
    }  \sum_{\env=1}^{\numenv} \int \dualFunctional{\weight_\env} \contFct_\env d \empProbBase'_\env(\subFct) ) - \sup_{
          \substack{\contFct'_\env \in \sqrBdFctClass, \\
          \sum_{\env}\contFct'_\env = 0}
    }  \sum_{\env=1}^{\numenv} \int \dualFunctional{\weight_\env} \contFct'_\env d \empProbBase_\env(\subFct) ) \right)\right] \\
    \leq & \, \mathbb{E}_{(\predData', \targetData')} \left[ 
    \sup_{\subFct \in \subFctClass} \sup_{
          \substack{\contFct_\env \in \boundedContFctSet{\empMax \vee \empMax'}}
    }  \sum_{\env=1}^{\numenv} \int \dualFunctional{\weight_\env} \contFct_\env d (\empProbBase'_{\env}(\subFct) - \empProbBase_\env(\subFct) ) 
    \right],
\end{align*}
where we used Lemma \ref{lem:dualForBndSuppDist} in the last inequality. Therefore we fall back to the bound in Equation \eqref{eq:firstUpperBound}, where $\empProbBase_\env(\subFct)$ and $\empProbBase'_\env(\subFct)$ are switched. Following the same steps as before, we arrive at the following bound:
\begin{equation}\label{eq:firstUnifBound2}
    \circled{2} \leq 4  \sum_{\env=1}^{\numenv} \weight_\env \mathbb{1}_{\empMax_\env  \leq \maxConst} \mathbb{E}_{(\predData'_\env, \targetData'_\env)} \left[
    \bdG{\maxConst}{(\predData'_\env,\targetData'_\env)}{(\predData_\env,\targetData_\env)} \right] + 8 \mathbb{E}_{\empMax'} \left[\bdH{\maxConst}{\empMax'}{\empMax}\right].
\end{equation}
As it is usually the case in the derivation of high probability bounds, we first bound the expectations of $\circled{1}$ and $\circled{2}$ (in Section \ref{app:bdExpectation}) and then prove some concentration bounds of $\circled{1}$ and $\circled{2}$ around their respective averages (in Section \ref{app:concentrationBds}).

\subsection{Bounding the Expectations} \label{app:bdExpectation}

Notice that since $(\predData, \targetData)$ and $(\predData', \targetData')$ have the same distribution (and are independent of each other), the expectations of the bounds \eqref{eq:firstUnifBound1} and \eqref{eq:firstUnifBound2} are identical, that is, we have:
$$
\mathbb{E}_{(\predData, \targetData)}\left[\circled{1}\right] \vee \mathbb{E}_{(\predData,\targetData)}\left[\circled{2}\right] \leq 4  \sum_{\env=1}^{\numenv} \weight_\env  \mathbb{E}_{(\predData_\env, \targetData_\env),(\predData'_\env, \targetData'_\env)} \left[
    \bdG{\maxConst}{(\predData_\env,\targetData_\env)}{(\predData'_\env,\targetData'_\env)} \right] + 8 \mathbb{E}_{\empMax,\empMax'} \left[\bdH{\maxConst}{\empMax}{\empMax'}\right] \doteq \circled{3}.
$$
So there are two quantities to bound for any $\env \in [\numenv]$: $\mathbb{E}_{(\predData_\env, \targetData_\env),(\predData'_\env, \targetData'_\env)} \left[
\bdG{\maxConst}{(\predData_\env,\targetData_\env)}{(\predData'_\env,\targetData'_\env)} \right]$ and $\mathbb{E}_{\empMax,\empMax'} \left[\bdH{\maxConst}{\empMax}{\empMax'}\right]$. Let's start with the first one. Set $\env \in [\numenv]$, we have:
$$
\mathbb{E}_{(\predData_\env, \targetData_\env),(\predData'_\env, \targetData'_\env)} \left[
\bdG{\maxConst}{(\predData_\env,\targetData_\env)}{(\predData'_\env,\targetData'_\env)} \right] = \mathbb{E}_{(\predData_\env, \targetData_\env), (\predData'_\env, \targetData'_\env)}\left[ \sup_{\lipFct \in \lipFctClass{\maxConst}, \subFct \in \subFctClass} \frac{1}{\nObs{\env}} \sum_{i = 1}^{\nObs{\env}} \left( \lipFct(\targetObs_{i}^{\env} - \subFct(\predObs_{i}^\env)) - \lipFct({\targetObs_{i}'}^{\env} - \subFct({\predObs_{i}'}^\env)) \right) \right] \doteq \circled{4}
$$
As all the observations are i.i.d., the above expectation would remain unchanged if we switched any observation $(\predObs_i^{\env}, \targetObs_i^\env)$ with its counterpart $({\predObs'_i}^{\env}, {\targetObs'_i}^\env)$.
Let $\radVarVect = (\radVar_i)_{i = 1}^{\nObs{\env}}$ be i.i.d.~Rademacher variables. By symmetry we thus have:
\begin{align} \label{eq:closerToRadComp}
    \circled{4} = & \, \mathbb{E}_{(\predData_\env, \targetData_\env), (\predData'_\env, \targetData'_\env)}\left[ \mathbb{E}_{\radVarVect} \left[ \sup_{\lipFct \in \lipFctClass{\maxConst}, \subFct \in \subFctClass} \frac{1}{\nObs{\env}} \sum_{i = 1}^{\nObs{\env}} \radVar_i \left( \lipFct(\targetObs_{i}^{\env} - \subFct(\predObs_{i}^\env)) - \lipFct({\targetObs_{i}'}^{\env} - \subFct({\predObs_{i}'}^\env)) \right) \right] \right] \nonumber \\
    \leq & \, 2 \mathbb{E}_{(\predData_\env, \targetData_\env)}\left[ \mathbb{E}_{\radVarVect} \left[
    \sup_{\lipFct \in \lipFctClass{\maxConst}} \sup_{ \subFct \in \subFctClass} \frac{1}{\nObs{\env}} \sum_{i = 1}^{\nObs{\env}} \radVar_i  \lipFct(\targetObs_{i}^{\env} - \subFct(\predObs_{i}^\env))  \right] \right].
\end{align}
The above quantity is close to the Rademacher complexity of the function class $\subFctClass$; if $\lipFct$ were a fixed $\maxConst$-Lipschitz function, then by the contraction lemma (e.g.,~Lemma 26.9 from \cite{shalev2014understanding}) we would be able to bound it directly by $2\maxConst \radComp_{\nObs{\env}}(\subFctClass)$. 
However, the supremum over $\lipFctClass{\maxConst}$ inside the expectation makes it more difficult to derive such a bound involving the Rademacher complexity. We show below that this additional supremum will only add a $\log(\nObs{\env})$ factor and some $O(\nObs{\env}^{-1/2})$ terms to the Rademacher complexity. We prove this using a chaining argument \citep{dudley1987universal}; our next steps are inspired by the proof of Lemma 27.4 from \cite{shalev2014understanding}. Any $\lipFct \in \lipFctClass{\maxConst}$ can be decomposed as follows:
\begin{equation} \label{eq:chaining}
    \lipFct = (\lipFct - \lipFct_K) + (\lipFct_K - \lipFct_{K-1}) + \ldots + (\lipFct_1 - \lipFct_0) + \lipFct_0,
\end{equation}
where for any $k \in \{ 0 \} \cup [K]$, $\lipFct_k$ belongs to a $2^{-k} \maxConst^2$-cover (for the norm $\| \cdot \|_{\infty}$) of $\lipFctClass{\maxConst}$, and chosen such that: $\| \lipFct_k - \lipFct_{k+1} \|_{\infty} \leq 2^{-k} \maxConst^2$, where $\lipFct_{\maxConst +1} = \lipFct$.
As any function in $\lipFctClass{\maxConst}$ is bounded by $\maxConst^2$, we can take $\lipFct_0 = 0$.
Call $\coverSet_k$ the above cover sets of $\lipFctClass{\maxConst}$ such that $\forall k, \lipFct_k \in \coverSet_k$, and for any $k\geq 1$ let $\hat{\coverSet}_k \doteq \{ \lipFct_k - \lipFct_{k-1} : \lipFct_k \in \coverSet_k, \lipFct_{k-1} \in \coverSet_{k-1} \text{ and } \| \lipFct_k - \lipFct_{k-1} \|_{\infty} \leq 2^{-k+1} \maxConst^2 \}$. 
Then from Equation \eqref{eq:chaining} we have that for any fixed $\lipFct \in \lipFctClass{\maxConst}$,
$$
\exists \lipDiffFct_k \in \hat{\coverSet}_k \text{ for } k \in [K], \text{ and } \exists \lipDiffFct \in \lipFctClass{2 \maxConst}, \text{ s.t. } \| \lipDiffFct \|_\infty \leq \maxConst^2 2^{-K} \text{ and } \lipFct = \lipDiffFct + \sum_{k=1}^K \lipDiffFct_k.
$$
Therefore, for any fixed sample $(\predData_\env, \targetData_\env)$ we get:
\begin{equation}\label{eq:firstBoundFromChaining}
    \mathbb{E}_{\radVarVect} \left[
    \sup_{\lipFct \in \lipFctClass{\maxConst}} \sup_{ \subFct \in \subFctClass} \frac{1}{\nObs{\env}} \sum_{i = 1}^{\nObs{\env}} \radVar_i  \lipFct(\targetObs_{i}^{\env} - \subFct(\predObs_{i}^\env))  \right] \leq \maxConst^2 2^{-K} + \sum_{k = 1}^K  \mathbb{E}_{\radVarVect} \left[
    \sup_{\lipDiffFct \in \hat{\coverSet}_k} \sup_{ \subFct \in \subFctClass} \frac{1}{\nObs{\env}} \sum_{i = 1}^{\nObs{\env}} \radVar_i  \lipDiffFct(\targetObs_{i}^{\env} - \subFct(\predObs_{i}^\env))  \right].
\end{equation}
The advantage of this decomposition is that each set $\hat{\coverSet}_k$ is finite. Indeed, it is known that the metric entropy (the logarithm of the covering number) for an $\epsilon$-cover of the class of $L$-Lipschitz functions defined on a ball of diameter $D$ in $\mathbb{R}^d$ is of the order of $(LD/\epsilon)^d$. As we focus on Lipschitz functions on $\mathbb{R}$ here, we have that $|\hat{\coverSet}_k|$ is therefore of the order $2^k$.
More precisely, in Lemma \ref{lem:lipCover} we (briefly) expose a possible construction for $\coverSet_k$, for which we have $\log|\coverSet_k| = 2 \log(3) \cdot 2^k$. 
Furthermore, from the construction given in Lemma \ref{lem:lipCover}, the $\lipFct_k$'s from \eqref{eq:chaining} can actually be chosen so that $\lipDiffFct_k \doteq \lipFct_k - \lipFct_{k-1}$ is $\maxConst$-Lipschitz (instead of $2\maxConst$-Lipschitz) and $\|\lipDiffFct_k\|_\infty \leq \maxConst^2 2^{-k}$ (instead of $\maxConst^2 2^{-k + 1}$).
In the following we will therefore consider that the $\lipFct_k$s and the $\lipDiffFct_k$s satisfy these conditions; note however that this change will just improve our bound up to some constant factor, so it can be ignored.

We can derive more explicit upper bounds for the terms in R.H.S.~of Equation \eqref{eq:firstBoundFromChaining}, using an adaptation of Massart Lemma's proof (see Lemma \ref{lem:adaptMassart}).
More precisely, we show in Lemma \ref{lem:adaptMassart} how to bound $\mathbb{E}_{\radVarVect} \left[
    \sup_{\lipDiffFct \in \coverSet} \sup_{ \subFct \in \subFctClass} \frac{1}{\nObs{\env}} \sum_{i = 1}^{\nObs{\env}} \radVar_i  \lipDiffFct(\targetObs_{i}^{\env} - \subFct(\predObs_{i}^\env)))  \right]$ when $\coverSet$ is a finite set of bounded and Lipschitz functions. 
Therefore, using Lemma \ref{lem:adaptMassart} with $\epsilon = \maxConst^2 2^{-k}$ and $\log|\hat{\coverSet}_k| \leq \log|\coverSet_k|^2 = 4 \log(3)\cdot2^k$, we get:
\begin{align*}
    \mathbb{E}_{\radVarVect} \left[
    \sup_{\lipFct \in \lipFctClass{\maxConst}} \sup_{ \subFct \in \subFctClass} \frac{1}{\nObs{\env}} \sum_{i = 1}^{\nObs{\env}} \radVar_i  \lipFct(\targetObs_{i}^{\env} - \subFct(\predObs_{i}^\env))  \right] \leq & \, \maxConst^2 2^{-K} + \maxConst^2 \sum_{k = 1}^K  2\sqrt{2} 2^{-k} \sqrt{\frac{4\log(3)2^k}{\nObs{\env}}} + K \maxConst \empRadComp_{\predData_\env}(\subFctClass) \\
    \leq & \, \maxConst^2 2^{-K}+ K \maxConst \empRadComp_{\predData_\env}(\subFctClass) + \maxConst^2 \frac{4\sqrt{2\log(3)}}{\sqrt{\nObs{\env}}\left(\sqrt{2} -1 \right)},
\end{align*}
where we used in the last inequality that $\sum_{k = 1}^{\infty} 2^{-k/2} = \frac{1}{\sqrt{2} - 1}$.

Finally, by setting $K = \lceil \log(\nObs{\env}) \rceil$, and since $\log 2 \geq 0.5$, we get that:
\begin{equation} \label{eq:boundForG}
  \mathbb{E}_{(\predData_\env, \targetData_\env),(\predData'_\env, \targetData'_\env)} \left[
\bdG{\maxConst}{(\predData_\env,\targetData_\env)}{(\predData'_\env,\targetData'_\env)} \right] \leq \frac{\maxConst^2}{\sqrt{\nObs{\env}}} \cdot \left( 2 + \frac{8\sqrt{2 \log(3)}}{\sqrt{2}-1} \right) + 2 \left( 1 + \log(\nObs{\env}) \right) \maxConst \radComp_{\nObs{\env}}(\subFctClass).  
\end{equation}

Now let's turn to bounding $\mathbb{E}_{\empMax,\empMax'}[\bdH{\maxConst}{\empMax}{\empMax'}] = \mathbb{E}_{\empMax, \empMax'}[(\empMax \vee \empMax')^2 \mathbb{1}_{\empMax \vee \empMax' > \maxConst}]$. Using the notation from Theorem \ref{thm:unifBound}, we let $\maxVarEnv{\env} \doteq \sup_{\subFct \in \subFctClass} | \subFct(\predBase^{\env}) - \targetBase^{\env} |$; recall that from Theorem \ref{thm:unifBound} we assume these variables are either bounded or sub-Gaussian (e.g.,~see \cite{wainwright2019high}, Chapter 2 for the definition). In particular, if the $\maxVarEnv{\env}$s are all bounded by $\maxConst$, then $\mathbb{E}_{\empMax,\empMax'}[\bdH{\maxConst}{\empMax}{\empMax'}]$ is equal to zero. 

Assume now that the $\maxVarEnv{\env}$s are sub-Gaussian  $\mathcal{G}(\subGaussMean{\env}, \subGaussStd{\env})$, that is with mean $\subGaussMean{\env}$ and sub-Gaussian parameter $\subGaussStd{\env}$. For simplicity call $\maxVar \doteq \empMax \vee \empMax$; if we set $\maxConst = \sqrt{2}(\maxConst' + \subGaussMean{})$ for some $\maxConst' > 0$ and where $\subGaussMean{} \doteq \max_{\env} \subGaussMean{\env}$, we prove in Lemma \ref{lem:simpleIntegrationBound} that:
$$
\mathbb{E}_{\maxVar}[\maxVar^2 \mathbb{1}_{\maxVar > \maxConst}] \leq 2 \sum_{\env = 1}^{\numenv} \nObs{\env} \left( 2(\maxConst' + \subGaussMean{})^2 e^{- \maxConst'^2 / \subGaussStd{\env}^2} + 4 \subGaussStd{\env}^2 e^{-\maxConst'^2 / 2 \subGaussStd{\env}^2} \right).
$$
Call $\subGaussStd{}^2 = \max_{\env \in [\numenv]} \subGaussStd{\env}^2$, and let $\bdProb \in (0,1)$ be the probability from Theorem \ref{thm:unifBound} -- we introduce it now but it will be useful in Section \ref{app:concentrationBds}. We set $\maxConst' = 2 \sqrt{\log(\nObs{} / \bdProb)} \subGaussStd{}$ (recall that $\nObs{} = \sum_\env \nObs{\env}$). Hence, for $\maxConst = \sqrt{2}\left(2 \sqrt{\log(\nObs{}/ \bdProb)}\subGaussStd{} + \subGaussMean{}\right)$:
\begin{align}\label{eq:boundForH}
    \mathbb{E}_{\empMax,\empMax'}[\bdH{\maxConst}{\empMax}{\empMax'}] \leq & \, 2 \sum_{\env=1}^\numenv \frac{\bdProb^2\nObs{\env}}{\nObs{}^2} \cdot \left( 2\left(2\sqrt{\log(\nObs{}/ \bdProb)} + \subGaussMean{} \right)^2 \bdProb^2 / \nObs{}^2 + 4 \subGaussStd{}^2 \right) \nonumber \\
    \leq & \, \frac{4\bdProb^2}{\nObs{}} \left( \left( 2\sqrt{\log(\nObs{}/ \bdProb)} + \subGaussMean{} \right)^2 \bdProb^2 / \nObs{}^2 + 2 \subGaussStd{}^2 \right).
\end{align} 
To conclude this section, let's summarize the bounds we derived based on Equations \eqref{eq:boundForG} and \eqref{eq:boundForH}. When the $\maxVarEnv{\env}$s are sub-Gaussian $\mathcal{G}(\subGaussMean{\env}, \subGaussStd{\env})$, we have:
\begin{equation} \label{eq:expBoundSubG}
    \circled{3} \leq \sum_{\env=1}^{\numenv} \weight_{\env} \left( \frac{\constA{\bdProb,\nObs{}}^0}{\sqrt{\nObs{\env}}} + \constB{\bdProb,\nObs{}}^0 (1 + \log(\nObs{\env})) \radComp_{\nObs{\env}}(\subFctClass) \right) + \frac{\constC{\bdProb,\nObs{}}^0}{\nObs{}},
\end{equation}
where $\constA{\bdProb,\nObs{}}^0 = 8 \left( 2 \sqrt{\log(\nObs{}/\bdProb)} \subGaussStd{} + \subGaussMean{} \right)^2 \cdot \left( 2 + \frac{8 \sqrt{2 \log 3}}{\sqrt{2} - 1} \right)$, $\constB{\bdProb,\nObs{}}^0 = 8 \sqrt{2} \left(2 \sqrt{\log(\nObs{}/\bdProb)} \subGaussStd{} + \subGaussMean{} \right)$ and $\constC{\bdProb,\nObs{}}^0 = 32 \bdProb^2 \left( (2 \sqrt{\log(\nObs{}/\bdProb)} \subGaussStd{} + \subGaussMean{})^2 \bdProb^2 / \nObs{}^2 + 2 \subGaussStd{}^2 \right)$.

When the $\maxVarEnv{\env}$'s are bounded with probability one by some constant $\maxConst>0$, we have:
\begin{equation} \label{eq:expBoundBded}
    \circled{3} \leq \sum_{\env=1}^{\numenv} \weight_{\env} \left( \frac{\constA{}^0}{\sqrt{\nObs{\env}}} + \constB{}^0 (1 + \log(\nObs{\env})) \radComp_{\nObs{\env}}(\subFctClass) \right), 
\end{equation}
where $\constA{}^0 = 4 \left( 2 + \frac{8 \sqrt{2 \log(3)}}{\sqrt{2} - 1} \right) \maxConst^2 $ and $\constB{}^0 = 8 (1 + \log(\nObs{\env})) \maxConst$.

\subsection{Concentration Bounds} \label{app:concentrationBds}

Take any $\env \in [\numenv]$ and $\maxConst > 0$. Notice that $\mathbb{1}_{\empMax_\env \leq \maxConst} \mathbb{E}_{(\predData'_{\env}, \targetData'_{\env})} \left[\bdG{\maxConst}{(\predData_{\env},\targetData_{\env})}{(\predData'_{\env}, \targetData'_{\env})} \right]$, as a function of the data $(\predData_{\env},\targetData_{\env})$, satisfies the bounded difference condition for McDiarmid's inequality (e.g.,~see Lemma 26.4 from \cite{shalev2014understanding}) with constant $2 \maxConst^2 / \nObs{\env}$. Hence with probability at least $1-\bdProb/\numenv$, we have:
$$
\mathbb{1}_{\empMax_\env \leq \maxConst} \mathbb{E}_{(\predData'_{\env}, \targetData'_{\env})} \left[\bdG{\maxConst}{(\predData_{\env},\targetData_{\env})}{(\predData'_{\env}, \targetData'_{\env})} \right] \leq \maxConst^2 \sqrt{\frac{2 \log(\numenv / \bdProb)}{\nObs{\env}}} + \mathbb{E}_{(\predData_{\env},\targetData_{\env}), (\predData'_{\env}, \targetData'_{\env})} \left[ \bdG{\maxConst}{(\predData_{\env},\targetData_{\env})}{(\predData'_{\env}, \targetData'_{\env})} \right].
$$
Similarly, we have also with probability at least $1 - \bdProb / \numenv$:
$$
\mathbb{1}_{\empMax_\env \leq \maxConst} \mathbb{E}_{(\predData'_{\env}, \targetData'_{\env})} \left[\bdG{\maxConst}{(\predData'_{\env},\targetData'_{\env})}{(\predData_{\env}, \targetData_{\env})} \right] \leq \maxConst^2 \sqrt{\frac{2 \log(\numenv / \bdProb)}{\nObs{\env}}} + \mathbb{E}_{(\predData_{\env},\targetData_{\env}), (\predData'_{\env}, \targetData'_{\env})} \left[ \bdG{\maxConst}{(\predData'_{\env},\targetData'_{\env})}{(\predData_{\env}, \targetData_{\env})} \right].
$$
Finally, note that we have also 
\begin{align*}
    \mathbb{E}_{\empMax'}\left[ \bdH{\maxConst}{\empMax}{\empMax'} \right] = \mathbb{E}_{\empMax'}\left[ \bdH{\maxConst}{\empMax'}{\empMax} \right] \leq & \, \mathbb{E}_{\empMax'}\left[ \empMax'^2 \mathbb{1}_{ \empMax' > M} \right] + \empMax^2  \mathbb{1}_{\empMax > \maxConst} \\
    \leq & \, \mathbb{E}_{\empMax, \empMax'}\left[ \bdH{\maxConst}{\empMax}{\empMax'} \right] + \empMax^2  \mathbb{1}_{\empMax > \maxConst}.
\end{align*}
When the $\maxVarEnv{\env}$'s are all bounded by $\maxConst$ with probability one, then $\empMax^2  \mathbb{1}_{\empMax > \maxConst} = 0$ (with probability one). 
However, when the $\maxVarEnv{\env}$'s are only sub-Gaussian $\mathcal{G}(\subGaussMean{\env}, \subGaussStd{\env})$, the probability that $\empMax^2  \mathbb{1}_{\empMax > \maxConst}$ is non-zero is
$$
\mathbb{P}\left( \empMax^2  \mathbb{1}_{\empMax > \maxConst} = 0 > 0 \right) = \mathbb{P}\left( \empMax > \maxConst \right) \leq \sum_{\env=1}^{\numenv} \nObs{\env} \mathbb{P}\left( \maxVarEnv{\env} \geq \maxConst \right) \leq \nObs{}e^{- \maxConst'^2 / \subGaussStd{}^2} \leq \bdProb,
$$
when we choose $\maxConst' = 2 \sqrt{\log(\nObs{} / \bdProb)} \subGaussStd{}$ and $\maxConst = \sqrt{2}(\maxConst' + \subGaussMean{})$.
Hence, combining the above inequalities and equations \eqref{eq:firstUnifBound1} and \eqref{eq:firstUnifBound2} we obtain that with probability at least $1 - 3 \bdProb$ simultaneously:
\begin{equation} \label{eq:concentrationBd1}
    \circled{1} \leq \circled{3} + 2\left(2 \sqrt{\log(\nObs{}/ \bdProb)}\subGaussStd{} + \subGaussMean{}\right)^2 \sqrt{\frac{2 \log(\numenv / \bdProb)}{\nObs{\env}}} \, \text{ and } \, \circled{2} \leq \circled{3} + 2\left(2 \sqrt{\log(\nObs{}/ \bdProb)}\subGaussStd{} + \subGaussMean{}\right)^2 \sqrt{\frac{2 \log(\numenv / \bdProb)}{\nObs{\env}}},
\end{equation} 
when the $\maxVarEnv{\env}$'s are sub-Gaussian. When they are just bounded by some $\maxConst > 0$, we have with probability at least $1 - 2 \bdProb$ that simultaneously:
\begin{equation} \label{eq:concentrationBd2}
    \circled{1} \leq \circled{3} + \maxConst^2 \sqrt{\frac{2 \log(\numenv / \bdProb)}{\nObs{\env}}} \quad \text{ and } \quad \circled{2} \leq \circled{3} + \maxConst^2 \sqrt{\frac{2 \log(\numenv / \bdProb)}{\nObs{\env}}}.
\end{equation}

\subsection{Conclusion}\label{app:conclusion}
We now combine the bounds that were obtained in the previous sections, in particular equations \eqref{eq:firstUnifBound1}, \eqref{eq:firstUnifBound2}, \eqref{eq:expBoundSubG}
Whenever the $\maxVar_\env$'s are sub-Gaussian $\mathcal{G}(\subGaussMean{\env}, \subGaussStd{\env})$, we have that with probability at least $1 - 3 \bdProb$:
$$
\sup_{\subFct \in \subFctClass}\left| \wassVar{\empProbBaseVect(\subFct)}{\weights} - \wassVar{\probVectBase(\subFct)}{\weights} \right| = \circled{1} \vee \circled{2} \leq \sum_{\env=1}^{\numenv} \weight_{\env} \left( \frac{\constA{\bdProb,\nObs{}}}{\sqrt{\nObs{\env}}} + \constB{\bdProb,\nObs{}} (1 + \log(\nObs{\env})) \radComp_{\nObs{\env}}(\subFctClass) \right) + \frac{\constC{\bdProb,\nObs{}}}{\nObs{}},
$$
where $\constA{\bdProb,\nObs{}} = \constA{\bdProb,\nObs{}}^0 + 2\left(2 \sqrt{\log(\nObs{}/ \bdProb)}\subGaussStd{} + \subGaussMean{}\right)^2 \cdot \sqrt{2 \log(\numenv / \bdProb)} $, $\constB{\bdProb,\nObs{}} = \constB{\bdProb,\nObs{}}^0$ and $\constC{\bdProb,\nObs{}} = \constC{\bdProb,\nObs{}}^0$.

Whenever the $\maxVar_\env$'s are bounded with probability one by some constant $\maxConst > 0$, we have with probability at least $1 - 2 \bdProb$:
$$
\sup_{\subFct \in \subFctClass}\left| \wassVar{\empProbBaseVect(\subFct)}{\weights} - \wassVar{\probVectBase(\subFct)}{\weights} \right| = \circled{1} \vee \circled{2} \leq \sum_{\env=1}^{\numenv} \weight_{\env} \left( \frac{\constA{\bdProb}}{\sqrt{\nObs{\env}}} + \constB{\bdProb} (1 + \log(\nObs{\env})) \radComp_{\nObs{\env}}(\subFctClass) \right) ,
$$
where $\constA{\bdProb} = \constA{\bdProb}^0 + \maxConst^2 \cdot \sqrt{2 \log(\numenv / \bdProb)}$ and $\constB{\bdProb} = \constB{\bdProb}^0$.

Therefore the high probability uniform bound from Theorem \ref{thm:unifBound} is obtained by replacing $\bdProb$ respectively by $\bdProb / 3$ and $\bdProb / 2$. 
What is left to show is the bound for the difference between the minimal Wasserstein variance and its empirical counterpart; this can be shown since:
$$
\inf_{\subFct \in \subFctClass} \wassVar{\empProbBaseVect(\subFct)}{\weights} - \inf_{\subFct \in \subFctClass} \wassVar{\probVectBase(\subFct)}{\weights} \leq \circled{1} \quad \text{ and } \quad \inf_{\subFct \in \subFctClass} \wassVar{\probVectBase(\subFct)}{\weights} - \inf_{\subFct \in \subFctClass} \wassVar{\empProbBaseVect(\subFct)}{\weights} \leq \circled{2}.
$$

\subsection{Supporting Lemmas} \label{app:suppLem1}

Recall that in Section \ref{app:firstSteps} we defined $\boundedContFctSet{\empMax}$ as the set of continuous functions defined on $\mathbb{R}$ that are constant on $(-\infty, - \empMax]$ and on $[ \empMax, +\infty)$, that is $\boundedContFctSet{\empMax} \doteq \{ \contFct \in \contFctSet : \forall x \in \ball{0}{\empMax}^{c}, \contFct(x) = \contFct(\empMax x / |x|) \}$.

\begin{lemma} \label{lem:lipschitzDual}
    Let $\weight > 0$, $\empMax > 0$ and $\contFct \in \boundedContFctSet{\empMax}$. Then for any $x \in \ball{0}{\empMax}$ we have:
    $$
    \dualFunctional{\weight} \contFct(x) = \inf_{y \in \ball{0}{\empMax}} \left \{ \weight | x - y |^2 - f(y) \right \},
    $$
    and therefore $\dualFunctional{\weight} \contFct(x)$ is $4 \weight \empMax$-Lipschitz on $\ball{0}{\empMax}$.
\end{lemma}
\begin{proof}
Let $x \in \ball{0}{\empMax}$. Recall that we have:
$$
\dualFunctional{\weight} \contFct(x) \doteq \inf_{y \in \mathbb{R}} \left\{ \weight | x - y |^2 - \contFct(y) \right\}.
$$
Take any $y \in \ball{0}{\empMax}^c$. As $\contFct \in \boundedContFctSet{\empMax}$ then $\contFct(y) = \contFct(\empMax y / |y|)$. Note that $\empMax y / |y|$ is the orthogonal projection of $y$ on $\ball{0}{\empMax}$, hence by the contraction property of orthogonal projections we also have that $|x - \empMax y / |y| |^2 \leq |x - y|^2$.
This concludes the first point.

The second point is implied by the fact that, for all $y \in \ball{0}{\empMax}$, the functions $x \mapsto \weight |x - y|^2 - \contFct(y)$ are $4 \weight \empMax$-Lipschitz on $\ball{0}{\empMax}$ (see for instance Box 1.8 from \cite{santambrogio2015optimal}).
\end{proof}

\begin{lemma} \label{lem:dualForBndSuppDist}
For any probability distributions $\probBase_1, \probBase_2, \ldots, \probBase_\numenv$ that are all supported on the ball $\ball{0}{\empMax}$ for some $\empMax>0$, we have :
\begin{equation} \label{eq:dualForBndSuppDist}
    \wassVar{\probVectBase}{\weights} = \sup \left \{ \sum_{\env = 1}^\numenv \int \dualFunctional{\weight_\env} \contFct_\env d \probBase_\env : \contFct_\env \in \boundedContFctSet{\empMax}, \sum_{\env=1}^{\numenv} \contFct_\env = 0 \right \}.
\end{equation}
\end{lemma}
\begin{proof}
As all function functions $\contFct_\env \in \boundedContFctSet{\empMax}$ are continuous and bounded, they belong to $\sqrBdFctClass$, and it is clear that:
$$
\wassVar{\probVectBase}{\weights} \geq \sup \left \{ \sum_{\env = 1}^\numenv \int \dualFunctional{\weight_\env} \contFct_\env d \probBase_\env : \contFct_\env \in \boundedContFctSet{\empMax}, \sum_{\env=1}^{\numenv} \contFct_\env = 0 \right \}.
$$
The inequality in the opposite direction is proved by modifying a dual solution for Equation \eqref{eq:dualWV} into a solution for Equation \eqref{eq:dualForBndSuppDist}. Note that Proposition 2.3 from \cite{barycenter} proves that indeed the dual problem in Equation \eqref{eq:dualWV} admits a solution $(\contFct_{\env}^*)_{\env=1}^{\numenv}$. If we define a new operator $\dualFunctional{\weight_\env}^\empMax$ on $\sqrBdFctClass$ as follows: 
$$
\dualFunctional{\weight_\env}^\empMax \contFct(x) \doteq \inf_{y \in \ball{0}{\empMax}} \left\{\weight_\env |x - y|^2 - \contFct(y) \right \} \geq 
\dualFunctional{\weight_\env}\contFct(x).$$
We then have that:
$$
\wassVar{\probVectBase}{\weights}  = \sum_{\env=1}^{\numenv} \int \dualFunctional{\weight_\env}\contFct_\env^* d \probBase_{\env}
 \leq \sum_{\env=1}^{\numenv} \int \dualFunctional{\weight_\env}^\empMax \contFct_\env^* d \probBase_{\env}.
$$
For any $\env \in [\numenv]$, consider now another function $\tilde{\contFct}_{\env} \in \boundedContFctSet{\empMax}$ defined as follows:
$$
\tilde{\contFct}_{\env}(x) = \left \{\begin{array}{ll}
     \contFct_{\env}^*(x) \text{ if  } x \in \ball{0}{\empMax},&\\
     \contFct_{\env}^*(\empMax x / |x|) \text{ otherwise.}&
\end{array} \right.
$$
We may check that $\sum_{\env = 1}^{\numenv} \tilde{\contFct}_{\env} = 0$, and that for any $x \in \ball{0}{\empMax}$ we have $\dualFunctional{\weight_\env}^\empMax \contFct_\env^*(x) = \dualFunctional{\weight_\env}^\empMax \tilde{\contFct}_{\env}(x) = \dualFunctional{\weight_\env} \tilde{\contFct}_{\env}(x)$, where the last equality comes from Lemma \ref{lem:lipschitzDual}. As the distributions $\probBase_\env$ are supported on $\ball{0}{\empMax}$:
$$
\wassVar{\probVectBase}{\weights} \leq \sum_{\env=1}^{\numenv} \int \dualFunctional{\weight_\env} \tilde{\contFct}_{\env} d \probBase_{\env},
$$
which proves equality \eqref{eq:dualForBndSuppDist}.
\end{proof}

Let $\maxConst>0$. Recall the definition of $\lipFctClass{\maxConst}$ from Section \ref{app:firstSteps}:
$$
\lipFctClass{\maxConst} \doteq \{ \lipFct: \lipFct(0) = 0, \lipFct \text{ is } \maxConst\text{-Lipschitz and } \forall x \notin \ball{0}{\maxConst}, \lipFct(x) = \lipFct((\maxConst) x / |x| ) \}.
$$

\begin{lemma} \label{lem:lipCover}
For $\maxConst > 0$ and $k \geq 1$, it is possible to construct an $\maxConst^2 2^{-k}$-cover, called $\coverSet_k$, of $\lipFctClass{\maxConst}$ such that $\log|\coverSet_k| = 2 \log(3) \cdot 2^k$.
\end{lemma}
\begin{proof}
A straightforward construction of $\coverSet_k$ can be done as follows: subdivide the interval $[-\maxConst, \maxConst]$ into a grid, each segment of length $\maxConst 2^{-k}$ (there are $2^k$ of them on each side of the origin); set $\coverSet_k$ to be composed of all (continuous) piece-wise linear functions equal to $0$ at the origin and either increase or decrease by $\maxConst^2 2^{-k}$ or stay constant to the next point in the grid. These functions are also set to stay constant outside of $[-\maxConst, \maxConst]$. It is easy from there to check that this construction is indeed a $\maxConst^2 2^{-k}$-cover of $\lipFctClass{\maxConst}$, and that $\log|\coverSet_k| = 2 \log(3) \cdot 2^k$.
\end{proof}

\begin{lemma} \label{lem:adaptMassart}
Let $\maxConst>0$ and $\coverSet$ a finite set of $\maxConst$-Lipschitz functions, that are bounded by some $\epsilon>0$ for the infinite norm. For a fixed sample $(\predData_{\env}, \targetData_{\env})$, we have:
$$
\mathbb{E}_{\radVarVect} \left[
    \sup_{\lipDiffFct \in \coverSet} \sup_{ \subFct \in \subFctClass} \frac{1}{\nObs{\env}} \sum_{i = 1}^{\nObs{\env}} \radVar_i  \lipDiffFct(\targetObs_{i}^{\env} - \subFct(\predObs_{i}^\env)))  \right] \leq \maxConst \empRadComp_{\predData_\env}(\subFctClass) + 2\sqrt{2} \epsilon \sqrt{\frac{\log|\coverSet|}{\nObs{\env}}}.
$$
\end{lemma}
\begin{proof}
Our proof follows similar steps than in the proof of the Massart Lemma; see Lemma 26.8 from \cite{shalev2014understanding}. Let $\lambda > 0$, we have:
\begin{align*}
    \lambda \mathbb{E}_{\radVarVect} \left[
    \sup_{\lipDiffFct \in \coverSet} \sup_{ \subFct \in \subFctClass} \frac{1}{\nObs{\env}} \sum_{i = 1}^{\nObs{\env}} \radVar_i  \lipDiffFct(\targetObs_{i}^{\env} - \subFct(\predObs_{i}^\env))  \right] = & \, \mathbb{E}_{\radVarVect} \left[
    \log \left( \sup_{\lipDiffFct \in \coverSet} \exp \left( \sup_{ \subFct \in \subFctClass} \frac{\lambda}{\nObs{\env}} \sum_{i = 1}^{\nObs{\env}} \radVar_i  \lipDiffFct(\targetObs_{i}^{\env} - \subFct(\predObs_{i}^\env)) \right) \right)  \right] \\
    \leq & \, \mathbb{E}_{\radVarVect} \left[
    \log \left( \sum_{\lipDiffFct \in \coverSet} \exp \left( \sup_{ \subFct \in \subFctClass} \frac{\lambda}{\nObs{\env}} \sum_{i = 1}^{\nObs{\env}} \radVar_i  \lipDiffFct(\targetObs_{i}^{\env} - \subFct(\predObs_{i}^\env)) \right) \right)  \right] \\
    \leq & \,
    \log  \sum_{\lipDiffFct \in \coverSet} \mathbb{E}_{\radVarVect} \left[ \exp \left( \lambda \sup_{ \subFct \in \subFctClass} \frac{1}{\nObs{\env}} \sum_{i = 1}^{\nObs{\env}} \radVar_i  \lipDiffFct(\targetObs_{i}^{\env} - \subFct(\predObs_{i}^\env)) \right)   \right] \doteq \circled{4}.
\end{align*}
Where the last inequality comes from Jensen's inequality. By Azuma-Hoeffding's Theorem (see for instance Chapter 2 from \cite{wainwright2019high}) the variable $\sup_{ \subFct \in \subFctClass} \nObs{\env}^{-1} \sum_{i = 1}^{\nObs{\env}} \radVar_i  \lipDiffFct(\targetObs_{i}^{\env} - \subFct(\predObs_{i}^\env))$ is sub-Gaussian with mean $\empRadComp'(\lipDiffFct) \doteq \mathbb{E}_{\radVarVect} [ \sup_{ \subFct \in \subFctClass} \nObs{\env}^{-1} \sum_{i = 1}^{\nObs{\env}} \radVar_i  \lipDiffFct(\targetObs_{i}^{\env} - \subFct(\predObs_{i}^\env))]$ and parameter $\sigma^2 \doteq \nObs{\env}^{-1} \sum_{i = 1}^{\nObs{\env}} 4 \epsilon^2 = 4 \epsilon^2 / \nObs{\env}$. Notice that by the Contraction Lemma (i.e.~Lemma 26.9 from \cite{shalev2014understanding}) we have $\empRadComp'(\lipDiffFct) \leq \empRadComp_{\predData_\env} (\subFctClass) \maxConst$. Hence we get this bound:
\begin{align*} 
    \circled{4} \leq & \, \log  \sum_{\lipDiffFct \in \coverSet} \mathbb{E}_{\radVarVect} \left[ \exp \left( \lambda \left(\sup_{ \subFct \in \subFctClass} \frac{1}{\nObs{\env}} \sum_{i = 1}^{\nObs{\env}} \radVar_i  \lipDiffFct(\targetObs_{i}^{\env} - \subFct(\predObs_{i}^\env)) - \empRadComp'(\lipDiffFct)\right) + \lambda \empRadComp'(\lipDiffFct) \right)   \right] \\
    \leq & \, \log\left(|\coverSet| e^{\lambda \maxConst \empRadComp_{\predData_\env}(\subFctClass)} \cdot e^{2 \lambda^2 \epsilon^2 / \nObs{\env}}\right) = \log|\coverSet| + \lambda \maxConst \empRadComp_{\predData_\env}(\subFctClass) + 2 \lambda^2 \epsilon^2 / \nObs{\env} .
\end{align*}
Hence, 
$$
\mathbb{E}_{\radVarVect} \left[
    \sup_{\lipDiffFct \in \coverSet} \sup_{ \subFct \in \subFctClass} \frac{1}{\nObs{\env}} \sum_{i = 1}^{\nObs{\env}} \radVar_i  \lipDiffFct(\targetObs_{i}^{\env} - \subFct(\predObs_{i}^\env))  \right] \leq \log|\coverSet| / \lambda + \maxConst \empRadComp_{\predData_\env}(\subFctClass) + 2 \lambda \epsilon^2 / \nObs{\env}.
$$
Taking $\lambda = \sqrt{\frac{\nObs{\env} \log|\coverSet|}{2 \epsilon^2}}$ we get the result.
\end{proof}

As in Section \ref{app:firstSteps} we let $\empMax = \empMax(\predData, \targetData) \doteq \max_{\env \in [\numenv]} \max_{i \in [\nObs{\env}]} \sup_{\subFct \in \subFctClass} | \subFct(\predObs_i^\env) - \targetObs_i^\env|$. Also as in Section \ref{app:firstSteps}, we let $\empMax' = \empMax'(\predData', \targetData')$ to be similarly defined but for another (independent) data set $(\predData', \targetData')$ drawn from the same distribution as $(\predData, \targetData)$.

\begin{lemma} \label{lem:simpleIntegrationBound}
Let $\maxConst > 0$ be a constant and call $\maxVar \doteq \empMax \vee \empMax'$. Assume $\maxVarEnv{\env} \doteq \sup_{\subFct \in \subFctClass} | \subFct(\predBase^{\env}) - \targetBase^{\env} |$ is sub-Gaussian  $\mathcal{G}(\subGaussMean{\env}, \subGaussStd{\env})$, for any $\env \in [\numenv]$. Finally, call $\subGaussMean{} \doteq \max_{\env} \subGaussMean{\env}$. We have:
$$
\mathbb{E}_{\maxVar}[\maxVar^2 \mathbb{1}_{\maxVar > \maxConst}] \leq 2 \sum_{\env = 1}^{\numenv} \nObs{\env} \left( 2(\maxConst' + \subGaussMean{})^2 e^{-{\maxConst'}^2 / \subGaussStd{\env}^2} + 4 \subGaussStd{\env}^2 e^{-{\maxConst'}^2 / 2\subGaussStd{\env}^2} \right),
$$
whenever $\maxConst' \doteq \maxConst / \sqrt{2} - \subGaussMean{} > 0$.
\end{lemma}
\begin{proof}
We have:
\begin{align*}
    \mathbb{E}_{\maxVar}[\maxVar^2 \mathbb{1}_{\maxVar > \maxConst}] = & \int_{0}^{\infty} \mathbb{P}\left(\maxVar^2 \mathbb{1}_{\maxVar > \maxConst} > t\right) dt \\
    = & \int_{0}^{\maxConst^2} \mathbb{P}\left(\maxVar > \maxConst \right) dt + \int_{\maxConst^2}^{\infty} \mathbb{P}\left(\maxVar^2 > t\right) dt \\
    = & \,\maxConst^2 \mathbb{P}\left(\maxVar > \maxConst \right) + \int_{\maxConst^2}^{\infty} \mathbb{P}\left(\maxVar > \sqrt{t}\right) dt.
\end{align*}
Notice that by a simple union bound we have $\mathbb{P}(\maxVar > t) \leq 2 \sum_{\env=1}^{\numenv} \nObs{\env} \mathbb{P}(\maxVarEnv{\env} > t)$. Now recall that the $\maxVarEnv{\env}$'s are sub-Gaussian with means $\subGaussMean{\env}$ and sub-Gaussian parameters $\subGaussStd{\env}$, and that we defined $\maxConst'>0$ such that $\maxConst = \sqrt{2}(\maxConst' + \subGaussMean{})$. Finally, note that $\sqrt{2}\sqrt{t + \subGaussMean{}^2} \geq \sqrt{t} + \subGaussMean{}$, hence for any $\env \in [\numenv]$:
$$
\mathbb{P}\left( \maxVarEnv{\env} \geq \sqrt{2} \sqrt{t + \subGaussMean{}^2}\right) \leq \mathbb{P}\left( \maxVarEnv{\env} \geq \sqrt{t} + \subGaussMean{} \right) \leq e^{-t / 2 \subGaussStd{\env}^2}.
$$
Using the change of variable $t \rightarrow 2(t+\subGaussMean{}^2)$ we get, for all $\env \in [\numenv]$:
\begin{align*}
    \int_{\maxConst^2}^{\infty} \mathbb{P}\left(\maxVarEnv{\env} \geq \sqrt{t}\right) dt \leq & \, 2 \int_{\maxConst^2 / 2 - \subGaussMean{}^2}^{\infty} e^{- t / 2 \subGaussStd{\env}^2} dt \\
    \leq & \, 2 \int_{\maxConst'^2}^{\infty} e^{-t / 2 \subGaussStd{\env}^2} = 4 \subGaussStd{\env}^2 e^{- \maxConst'^2 / 2 \subGaussStd{\env}^2}.
\end{align*}
Also: $\mathbb{P}\left(\maxVarEnv{\env} \geq \maxConst\right) \leq \mathbb{P}\left( \maxVarEnv{\env} - \subGaussMean{} \geq \sqrt{2} \maxConst' \right) \leq e^{-\maxConst'^2 / \subGaussStd{\env}^2}$. Combining all of the above, we get the result.
\end{proof}

\section{PROOFS OF PROPOSITION \ref{prop:WVAsymptLimit}, Theorem \ref{thm:asymptGuar} and Theorem \ref{thm:multCorrection}} \label{app:proofAsymptGuar}

We provide in this section the proofs of Proposition \ref{prop:WVAsymptLimit}, Theorem \ref{thm:asymptGuar} and Theorem \ref{thm:multCorrection}, as well as the details of the regularity conditions that are needed for them.

\subsection{Sufficient Conditions for the Asymptotic Limit in Equation \eqref{eq:wassAsymptLimit}} \label{app:listSuffCondAsymptLimit}

Let $(z_i)_{i=1}^{\nObs{}}\stackrel{\text{i.i.d.}}{\sim}\probBase$, where $\probBase$ is a probability distribution on $\mathbb{R}$.
Define $\empProbBase = \nObs{}^{-1} \sum_{i = 1}^{\nObs{}} \dirac_{z_i}$ the corresponding empirical distribution.
In the following we call $\cdf$ the CDF of $\probBase$, $\pdf \doteq \cdf'$ its PDF, $\cdf^{-1}$ its quantile function and $\quantDensity \doteq (\cdf^{-1})'$ the quantile density.
Then, from \cite{del2005asymptotics} Theorem 4.6 we have:
$$
\nObs{} \wass^{2}(\empProbBase, \probBase) \xrightarrow[\nObs{} \rightarrow \infty]{d} \int_{0}^{1} \brownBridge^2(t) \quantDensity^2(t) dt,
$$
where $\brownBridge(t)$ is the Brownian bridge between $0$ and $1$, if the following conditions are satisfied:

\begin{assumption}[Case (i) from Theorem 4.6 of \cite{del2005asymptotics}] \label{DelBarrioAssumptions}
Using the above notations, the distribution $\probBase$ satisfies the following properties:
\begin{enumerate}[label=(\roman*)]
\item $\probBase$ is supported on an interval $(a_{\cdf}, b_\cdf)$, $\cdf$ is twice differentiable and $\pdf > 0$ on $(a_\cdf, b_\cdf)$; note this means $\cdf^{-1}$ is also twice differentiable on $(0,1)$,
\item $\sup_{0<t<1} t(1-t) | \quantDensity'(t)| / \quantDensity(t) < \infty$,
\item $\int_0^1 t (1 - t) \quantDensity^2(t) dt < \infty$,
\item either $a_\cdf > - \infty$ or $\liminf_{t \rightarrow 0^+} t |\quantDensity'(t)| / \quantDensity(t) > 0$ and either $b_\cdf > - \infty$ or $\liminf_{t \rightarrow 0^+} t |\quantDensity'(1-t)| / \quantDensity(1-t) > 0$,
\item $\lim_{t \rightarrow 0^+} t \quantDensity(t) = 0$ and $\lim_{t \rightarrow 0^+} t \quantDensity(1-t) = 0$.
\end{enumerate}
\end{assumption}

\cite{del2005asymptotics} also provides other distributional limit results for the square Wasserstein variance under assumptions different from Assumption \ref{DelBarrioAssumptions}.
For simplicity, we limit ourselves to the above case as under this setting the asymptotic distribution is relatively simple, but it might happen that some realistic distributions does not satisfy Assumption \ref{DelBarrioAssumptions}.
It turns out however that a small modification of the Wasserstein distance can alleviate this potential issue (see Remark \ref{rem:onDelBarrioAssumption} below).

\begin{remark}[When Assumption \ref{DelBarrioAssumptions} holds] \label{rem:onDelBarrioAssumption}
It is easy to see that Assumption \ref{DelBarrioAssumptions} would hold, for instance, for any distribution $\probBase$ that is compactly supported with a continuously differentiable density that does not converge to zero too fast at the borders of its support, or is simply bounded away from zero.  
Furthermore, the assumption can hold also for distributions that are not compactly supported, as long as their tails of distribution are not too heavy.
For instance, in Examples 4.1 and 4.2 \cite{del2005asymptotics} shows that for the Weibull distribution, or distributions with tails of the form $\exp(-|x|^\alpha)$, Assumption \ref{DelBarrioAssumptions} holds only when $\alpha > 2$.

This means in particular that the normal distribution unfortunately does not respect Assumption \ref{DelBarrioAssumptions} and in fact falls into another category of distributions (see case (ii) from Theorem 4.6 of \cite{del2005asymptotics}) for which one needs to correct $\nObs{} \wass^{2}(\empProbBase, \probBase)$ by a drift that goes to infinity in order to obtain convergence in distribution.
This drift for the normal distribution will actually diverge relatively slowly at a logarithmic rate though, so actually even when the distribution $\probBase$ has tails relatively similar to a Gaussian, and therefore does not respect Assumption \ref{DelBarrioAssumptions}, setting the thresholds based on the asymptotic distribution in the R.H.S. of \eqref{eq:wassAsymptLimit} (or the R.H.S. of \eqref{eq:WVExactAsymptLimit}) might not be such a detriment in practice. 
For instance, we used Gaussian noises in our simulations' observational environments and didn't encounter any major issue.

However, if it is believed that the residuals don't respect Assumption \ref{DelBarrioAssumptions} in a way that might affect the validity of the thresholds, there is in fact a very simple way to solve the issue.
\cite{del2005asymptotics} actually derived their asymptotic results for a more general version of the Wasserstein distance called weighted Wasserstein distance; it is simply defined as the weighted L2 norm of the difference of the quantile functions.
We could easily replace in our paper the Wasserstein distance by its weighted counterpart and most of our results would still hold for 'well-behaved' weight functions (Theorem \ref{thm:unifBound}'s proof would probably be the most difficult to modify -- for simplicity and clarity we focus only on the classical Wasserstein distance in our paper and leave this potential extension for future work); the only minor difference would be that the integrals in the R.H.S.~of \eqref{eq:wassAsymptLimit}, \eqref{eq:WVExactAsymptLimit} will include the weight function.
In that case, since in general Assumption \ref{DelBarrioAssumptions} fails because the quantile density diverge too fast at $0$ and $1$, choosing a weight function that goes to zero fast enough at $0$ and $1$ will allow our asymptotic analysis to hold for a potentially much wider range of distributions.

A more radical choice a weight function can be one that is equal to zero outside of an interval $[\alpha, 1 - \alpha]$ for $\alpha \in (0,1/2)$ -- note that in that case we would loose the metric property needed for Lemma \ref{lem:noVariability} to hold.
This kind of weighted Wasserstein distance is sometimes called trimmed Mallows distance in the literature \citep{munk1998nonparametric}, and its asymptotic properties hold under quite weaker assumptions than Assumption \ref{DelBarrioAssumptions}.

\end{remark}

\subsection{Proof of Proposition \ref{prop:WVAsymptLimit}} \label{app:proofPropWVAsymptLimit}

Using the notations from Section \ref{sec:implementation}, let $\hat{\cdf}^{-1}_\env$ be the empirical quantile function for $\empProbBase_{\env}(\fct)$, and similarly call $\cdf^{-1}$ the quantile function of the distribution $\probBase(\fct)$ defined in Proposition \ref{prop:WVAsymptLimit}.
For any $\env \in [\numenv]$ and $t \in [0,1]$, we also define the empirical quantile process $v_{\nObs{\env}}(t) \doteq \sqrt{\nObs{\env}} \left( \hat{\cdf}_{\env}^{-1}(t) - \cdf^{-1}(t) \right)$.
Finally, let $\nObsVect \doteq (\nObs{\env})_{\env = 1}^{\numenv}$ and $\mathbf{1}$ is a vector of size $\numenv$ composed only of ones.
Based on equation \eqref{eq:explicitWV}, we have:
\begin{align*}
    \nObs{} \wassVar{\empProbBaseVect(\fct)}{\weights} & = \int_0^1 \sum_{\env = 1}^{\numenv} \nObs{} \weight_{\env} \left( \hat{\cdf}_\env^{-1}(t) - \sum_{\env' = 1}^{\numenv} \weight_{\env'} \hat{\cdf}^{-1}_{\env'}(t) \right)^2 dt \\
    & = \int_0^1 \sum_{\env = 1}^{\numenv} \weight_{\env} \left( \weight_\env^{-1/2} v_{\nObs{\env}}(t) - \sum_{\env' = 1}^{\numenv} \sqrt{\weight_{\env'}} v_{\nObs{\env'}}(t) \right)^2 dt \\
    & = \int_{0}^{1} V_{\nObsVect}^{T}(t) \left( D_{\weights}^{-1/2} - \mathbf{1} \mathbf{1}^{T} D_{\weights}^{1/2} \right)^{T} D_{\weights} \left( D_{\weights}^{-1/2} - \mathbf{1} {\mathbf{1}}^{T} D_{\weights}^{1/2} \right) V_{\nObsVect}(t) dt \\
    & = \int_{0}^{1} V_{\nObsVect}^{T}(t) \left( I_\numenv -  D_{\weights}^{1/2} \mathbf{1} \mathbf{1}^{T} D_{\weights}^{1/2} \right)^2 V_{\nObsVect}(t) dt,
\end{align*}
where $I_\numenv$ refers to the identity matrix of size $\numenv$, $D_{\weights} \doteq \text{diag}(\weights)$,$V_{\nObsVect}(t) \doteq \left((v_{\nObs{\env}}(t))_{\env = 1}^\numenv \right)^{T}$ and $V_{\nObsVect}^{T}(t)$ is its transpose.
In the following, we also call $A_{\weights} \doteq I_\numenv -  D_{\weights}^{1/2} \mathbf{1} \mathbf{1}^{T} D_{\weights}^{1/2}$.

Under Assumption \ref{DelBarrioAssumptions}, note that we have for any $\env \in [\numenv]$, 
$$v_{\nObs{\env}}(t) \xrightarrow[\nObs{\env} \rightarrow \infty]{d} \quantDensity_{\fct}(t) \brownBridge_{\env}(t),$$
in $L_2(0,1)$ for a Brownian bridge $\brownBridge_\env (t)$ on $[0,1]$ -- this is actually a consequence of Theorem 4.1 and Lemma 2.3 from \cite{del2005asymptotics} which show that a truncated version of $v_{\nObs{\env}}(t)$ converges in distribution toward $\quantDensity_{\fct}(t) \brownBridge_{\env}(t)$ in $L_2(0,1)$, and of their Lemma 2.4 which shows that the difference between this truncated version and the full process $v_{\nObs{\env}}(t)$ goes to $0$ in probability.
Furthermore, recall that data from different environments are independent of each other.
This means in particular that:
$$
V_{\nObsVect}(t) \xrightarrow[\nObs{0} \rightarrow \infty]{d} \quantDensity_{\fct}(t) \cdot \BrownianVect(t),
$$
in $L_2(0,1)^\numenv$ with $\BrownianVect(t) \doteq \left( (\brownBridge_\env(t))_{\env = 1}^\numenv \right)^{T}$ being a vector of $\numenv$ independent Brownian bridges on $[0,1]$.

To show Proposition \ref{prop:WVAsymptLimit}, we need to establish that for any sequence $(\nObsVect(k))_{k \geq 0}$ of the numbers of observations per environment such that $\nObs{0}(k) \doteq \min_{\env \in [\numenv]} \nObs{\env}(k) \rightarrow \infty$ as $k$ goes to infinity, we have that:
\begin{equation} \label{eq:prop1ProofObjective}
    \int_{0}^1 V_{\nObsVect(k)}^{T}(t) \cdot A_{\weights(k)}^2 \cdot V_{\nObsVect(k)}(t) dt \xrightarrow[k \rightarrow \infty]{d} \sum_{\env = 1}^{\numenv-1} \int_{0}^{1} \quantDensity_{\fct}^2 (t) \brownBridge_{\env}^2(t) dt.
\end{equation}
We are going to prove this by the selection principle (see for instance Proposition 1.6 in Chapter 3 of \cite{cinlar2011probability}), that is, we are going to show that every sub-sequence of the series in the LHS of equation \eqref{eq:prop1ProofObjective} admits a further sub-sequence that converges in distribution to the RHS of the equation.

Consider a sub-sequence $\left(\nObsVect(\phi(k))\right)_{k \geq 0}$ (with $\phi$ strictly increasing) of $\left(\nObsVect(k)\right)_{k \geq 0}$; since the weights $\weight_\env(\phi(k))$ are in $[0,1]$ and therefore bounded, there is a further sub-sequence $(\nObsVect(\psi \circ \phi(k)))_{k \geq 0}$ such that:
$$
\weights(\psi \circ \phi(k)) \xrightarrow[k \rightarrow \infty]{} \Bar{\weights} \in \Bar{\weightSpace}.
$$
As a consequence:
$$
 V_{\nObsVect(\psi \circ \phi(k))}^{T}(t) \cdot A_{\weights(\psi \circ \phi(k))}^2 \cdot V_{\nObsVect(\psi \circ \phi(k))}(t) \xrightarrow[k \rightarrow \infty]{d} \quantDensity_\fct^2(t) \cdot \BrownianVect(t)^{T} A_{\Bar{\weights} }^2 \BrownianVect(t).
$$
in $L_2(0,1)$ by the continuous mapping theorem.

Since $A_{\Bar{\weights}}$ is symmetric, it can be diagonalized $A_{\Bar{\weights}} = R^{T} \cdot D \cdot R$ where $R$ is an orthonormal matrix of size $\numenv \times \numenv$ and $D$ is a diagonal matrix composed of the $\numenv$ eigen-values of $A_{\Bar{\weights}}$.
Notice that $(\sqrt{\Bar{\weight}_1}, \ldots, \sqrt{\Bar{\weight}_\numenv})$ is an eigen vector of $A_{\Bar{\weights}}$ with eigen value $0$.
Also, notice that the matrix $D^{1/2}_{\Bar{\weights}} \mathbf{1} \mathbf{1}^{T} D^{1/2}_{\Bar{\weights}}$ in the RHS of the definition of $A_{\Bar{\weights}}$ is of rank one, and hence its null space is of dimension $\numenv - 1$; it is easy to see that each of these null vectors is an eigen-vector of $A_{\Bar{\weights}}$ with eigen-value $1$.
Hence, $D = \text{diag}(\underbrace{1, 1, \ldots, 1}_{\numenv-1 \text{ times}}, 0 )$.

Call $\Bar{\BrownianVect}(t) \doteq R \BrownianVect(t)$; since $R$ is orthonormal, it is easy to check that $\Bar{\BrownianVect}(t)$ is a vector of $\numenv$ independent Brownian bridges. We then have
$$
\BrownianVect(t)^{T} A_{\Bar{\weights} }^2 \BrownianVect(t) = \Bar{\BrownianVect}^{T}(t) D \Bar{\BrownianVect}(t) = \sum_{\env = 1}^{\numenv-1} \Bar{\brownBridge}_\env^2(t).
$$
Therefore, we've just showed that, for any strictly increasing $\phi$ there exists $\psi$ (also strictly increasing) such that:
$$
\nObs{}(\psi \circ \phi(k)) \cdot \wassVar{\empProbBaseVect_{\psi \circ \phi(k)}(\fct)}{\weights(\psi \circ \phi(k))} \xrightarrow[k \rightarrow \infty]{d} \sum_{\env = 1}^{\numenv - 1} \int_{0}^{1} \brownBridge^2_{\env}(t) \quantDensity_{\fct}^2(t) dt.
$$
As this limit is exactly the same in distribution regardless of the selected sequence $\left(\nObsVect(k)\right)_{k \geq 0}$ (and further sub-sequence $\left(\nObsVect(\phi(k))\right)_{k \geq 0}$), this proves proposition \ref{prop:WVAsymptLimit}.

\subsection{Regularity Conditions for Theorem \ref{thm:asymptGuar} and Theorem \ref{thm:generalAymptResult}} \label{app:regCondTh2}

We prove Theorem \ref{thm:asymptGuar} for a generic class of functions $\fctClass'$ -- of course, we are mainly interested in the case $\fctClass' \in \{ \fctClass_{-\predIdx}, \predIdx \in [\numpred] \}$.
The hypotheses of interest here are therefore the following: 
$$
\Tilde{\hyp}_{0}(\envSet): \minWV_{\weights}(\fctClass') = 0, \quad \text{against} \quad \Tilde{\hyp}_{1}(\envSet): \minWV_{\weights}(\fctClass') > 0.
$$
In fact, in Theorem \ref{thm:generalAymptResult} below we prove a slightly more general result than Theorem \ref{thm:asymptGuar}. 
Indeed, we consider the case where the function class used to compute the test statistic depends on $\nObs{}$:
$$
\text{Test statistic:} \quad \hat{\minWV}_{\weights }(\fctClass_{\nObs{}}') \doteq \inf_{\fct \in \fctClass_{\nObs{}}'}  \wassVar{\empProbBaseVect(\fct)}{\weights} \quad \text{where} \quad \forall \nObs{}, \,\, \fctClass_{\nObs{}}' \subseteq \fctClass_{\nObs{}+1}' \text{ and } \fctClass' = \overline{\bigcup_{\nObs{}} \fctClass_{\nObs{}}'},
$$
where the closure in the RHS is w.r.t.~the $\| \cdot \|_{\infty}$ norm.
We also call $\hat{\fct}_{\nObs{}}$ the obtained minimizer for $\hat{\minWV}_{\weights}(\fctClass_{\nObs{}}')$. 
For simplicity, we assume that the class of functions $\fctClass_{\nObs{}}'$ depends only on the total number of observations $\nObs{}$, but one can easily extend our analysis to the case where this class depends on the full array of numbers of observations per environment $\nObsVect \doteq (\nObs{\env})_{\env = 1}^{\numenv}$.

A possible choice for $\fctClass_{\nObs{}}'$ is the class of the regressors that are linear combinations of some basis functions, such as regression splines, Fourier features or wavelet bases for instance, where the number of bases increases with $\nObs{}$; another option is to directly restrict the complexity of $\fctClass_{\nObs{}}'$ by adding a regularization term to the initial Wasserstein variance minimization program $\hat{\minWV}_{\weights }(\fctClass')$, with its hyperparameter implicitly depending on $\nObs{}$.
In general, using a restricted class of function $\fctClass_{\nObs{}}'$ instead of $\fctClass'$ in finite samples can potentially decrease the number of false negatives by reducing overfitting -- at the risk of increasing the number of false positives, though. 
At least under some conditions, we prove that asymptotically such an approach constitutes a consistent test for $\Tilde{\hyp}_{0}(\envSet)$.

In terms of notations, for a compact set $\compactset \subset \mathbb{R}^\numpred$ we will call $\contSpace(\compactset, || \cdot ||_\infty )$ the set of real-valued continuous functions defined on $\compactset$, and $\sobolev^{\sobolevDegree,2}(\compactset)$ the Sobolev space on $\compactset$ with a degree $\sobolevDegree \in \mathbb{N}$ of differentiability (see Definition 6.31 from \cite{debnath2005introduction}, see also \cite{adams2003sobolev}).
For any subsets $\subFctClass_1, \subFctClass_2$ of $\contSpace(\compactset, || \cdot ||_\infty )$ we denote their Hausdorff distance by $\hausDist(\subFctClass_1, \subFctClass_2) \doteq \max(\sup_{\subFct_1 \in \subFctClass_1} d(\subFct_1, \subFctClass_2), \sup_{\subFct_2 \in \subFctClass_2} d(\subFct_2, \subFctClass_1))$ where $d(\subFct_i, \subFctClass_j) \doteq \inf_{\subFct_j \in \subFctClass_j} || \subFct_i - \subFct_j ||_\infty $ for $i,j \in \{ 1,2 \}$. 
Finally, we set $\weight_\env = \nObs{\env} / \nObs{}$.
We are going to use several useful properties of $\sobolev^{\sobolevDegree, 2}(\compactset)$, for $\sobolevDegree > \numpred / 2$, that are summarized in the remark below.

\begin{remark}[Useful properties of the Sobolev space] \label{rem:usefulPropSobolev}
When $\sobolevDegree > \numpred / 2$ and $\compactset$ has a smooth boundary (e.g.~$\compactset$ is a ball in $\mathbb{R}^\numpred$), the Sobolev embedding theorem states that $\sobolev^{\sobolevDegree, 2}(\compactset) \subset \contSpace(\compactset, ||\cdot||_{\infty})$, see Remark 3 of \cite{cucker2002mathematical}; if we call $\sobBall_\sobRad$ the ball centered at the origin and of radius $\sobRad > 0$ in the Sobolev Space $\sobolev^{\sobolevDegree, 2}(\compactset)$, we also have that $\Bar{\sobBall}_{\sobRad}$ is a compact subset of $\contSpace(\compactset, ||\cdot||_{\infty})$, where the closure is w.r.t.~the infinite norm's topology -- in the rest of our paper the closure will always be meant in that sense.
Furthermore, we have $\log(\mathcal{N}(\Bar{\sobBall}_{\sobRad}, \epsilon)) = O((\sobRad / \epsilon)^{\numpred / \sobolevDegree})$, where $\mathcal{N}(\Bar{\sobBall}_{\sobRad}, \epsilon)$ is the covering number of $\Bar{\sobBall}_{\sobRad}$ by balls of radius $\epsilon > 0$. 
Finally, the space $\sobolev^{\sobolevDegree, 2}(\compactset)$ is norm equivalent to the RKHS generated by the Matérn kernel $k_{\sobolevDegree - \numpred / 2, h}$ of degree $\sobolevDegree - \numpred / 2$ for any scale $h>0$, see Example 2.6 from \cite{kanagawa2018gaussian} and references therein.

\end{remark}

We can now fully detail our list of regularity conditions for Theorem \ref{thm:generalAymptResult} below (for Theorem \ref{thm:asymptGuar} these conditions were first summarized in Assumption \ref{regularityAssumptions} where $\fctClass' \in \{ \fctClass_{-\predIdx}, \predIdx \in [\numpred] \}$ and $\fctClass'_{\nObs{}} = \fctClass', \forall \nObs{}$):

\begin{assumption}[Full detail of the regularity conditions] \label{fullDetailRegCond}
The following properties are true:
\begin{enumerate}[label=(\arabic*)]
    \item The $\pred{\env}$'s are bounded, that is there exists a compact set $\compactset \subseteq \mathbb{R}^\numpred$ with smooth boundary (e.g.~a ball) such that $\forall \env \in [\numenv], \, \mathbb{P}(\pred{\env} \in \compactset) = 1$; and the $\target{\env}$'s are sub-Gaussian,
    \item Data from different environments are independent of each other, that is $(\predData^{1}, \targetData^{1}), \ldots, (\predData^{\numenv}, \targetData^{\numenv})$ are mutually independent,
    \item There is a constant $\lambda > 0$ independent of the $\nObs{\env}$'s such that $\nObs{\env} \geq \lambda \nObs{}, \forall \env \in [\numenv]$,
    \item For some integer $\sobolevDegree > \numpred / 2$ and Sobolev space $\sobolev^{\sobolevDegree, 2}(\compactset)$, we have: For any $\delta \in (0,1)$ there exists $\sobRad_\delta > 0$ such that for $\nObs{}$ large enough, with probability at least $1-\delta$, $\hat{\fct}_{\nObs{}} \in \sobBall_{\sobRad_\delta}$, where $\sobBall_{\sobRad_\delta}$ is the ball of radius $\sobRad_\delta > 0$ centered at the origin in $\sobolev^{\sobolevDegree, 2}(\compactset)$.
\end{enumerate}

In what follows, we will call $\subFctClass_{\nObs{}}^\delta \doteq \overline{\fctClass_{\nObs{}}' \bigcap \sobBall_{\sobRad_{\delta}}}$, $\subFctClass^\delta \doteq \overline{\fctClass' \bigcap \sobBall_{\sobRad_{\delta}}}$ and $\subFctClass^\delta_0 \doteq \{ \subFct \in \subFctClass^\delta: \wassVar{\probVectBase(\subFct)}{\weights} = 0 \}$.
\begin{enumerate}[label=(\arabic*)]
\setcounter{enumi}{4}
    \item For any $\delta \in (0,1)$, we have $\hausDist(\subFctClass_{\nObs{}}^\delta, \subFctClass^\delta) = o\left(\nObs{}^{-1/2}\right)$,
    \item $\forall \fct \in \subFctClass^\delta, \forall \env \in [\numenv], \probBase_{\env}(\fct)$ satisfies Assumption \ref{DelBarrioAssumptions} from Section \ref{app:listSuffCondAsymptLimit}. 
    Furthermore, conditions (ii) and (iii) from Assumption \ref{DelBarrioAssumptions} are satisfied uniformly in the following sense:
    \begin{enumerate}[label=(6\alph*)]
        \item $\forall s \in (0,1/2), \exists M_{\delta, s} > 0$ s.t.~$\forall \fct \in \subFctClass^{\delta}, \forall \env \in [\numenv]$, $\sup_{s\leq t \leq 1-s} t(1-t) |{\quantDensity^{\env}_\fct}'(t)| / \quantDensity^\env_{\fct}(t) < M_{\delta, s}$,
        \item $\forall s \in (0,1/2), \exists M'_{\delta,s} > 0$ s.t.~$\forall \fct \in \subFctClass^{\delta}, \forall \env \in [\numenv]$, $\int_s^{1-s} t(1-t) {\quantDensity_\fct^\env}^2(t) dt < M_{\delta,s}' $,
        \item And, $\forall \env \in [\numenv], \sup_{\fct \in \subFctClass^\delta_0} \left( \int_0^s t(1-t) {\quantDensity^\env_\fct}^2(t) dt \right) \vee \left( \int_{1-s}^1 t(1-t) {\quantDensity^\env_\fct}^2(t) dt \right) \xrightarrow[s \rightarrow 0]{} 0$,
    \end{enumerate}
\end{enumerate}
where $\quantDensity^\env_\fct$ refers to the quantile density of $\probBase_\env(\fct)$, for any $\env \in [\numenv]$ and function $\fct \in \subFctClass^{\delta}$.
\end{assumption}

\begin{remark}[On condition (4) of Assumption \ref{fullDetailRegCond}]
Condition (4) typically arises in situations where the class $\fctClass'_{\nObs{}}$ is composed of smooth functions, and is not too rich so that it doesn't tend to overfit the data by returning near-zero residuals in each environment.
For instance, when $\fctClass'_{\nObs{}} = \fctClass', \forall \nObs{}$, where $\fctClass'$ is the class of linear regressors, we often observe in practice that the minimizer $\hat{\fct}_{\nObs{}}$ has coefficients that are not too extreme, which means in particular they are bounded in probability -- condition (4) is valid in that case.
As mentioned earlier, putting restrictions on the class $\fctClass'_{\nObs{}}$ by adding a regularization term to the optimization or by considering a number of basis functions that grows slowly in $\nObs{}$ are other ways to insure that condition (4) is valid.
\end{remark}

\begin{remark}[On conditions (6a)--(6c) of Assumption \ref{fullDetailRegCond}] As long as Assumption \ref{DelBarrioAssumptions} is true for all $\probBase_{\env}(\fct)$ with $\fct \in \subFctClass^\delta$, conditions (6a) and (6b) (respectively (6c)) are automatically verified when $\subFctClass^\delta$ (respectively $\subFctClass^\delta_0$) is finite. 
Furthermore, since the quantile density of the normal distribution depends only the variance parameter, if we focus only conditions (6a) and (6b), notice that these conditions are true whenever the observed variables $\pred{\env}$ and $\target{\env}$ are jointly Gaussian and $\subFctClass^\delta$ is a bounded class of linear regressors -- even though Assumption \ref{DelBarrioAssumptions} is itself not verified for the normal distribution. 
For that reason, if we use a weighted Wasserstein distance as suggested in Remark \ref{rem:onDelBarrioAssumption} with appropriate weight function, the conditions of Assumption \ref{fullDetailRegCond} hold easily when data are generated by a linear Gaussian structural model, with the exception of condition (1) of course -- we believe however that this condition can be weaken to $\pred{\env}$ being sub-Gaussian, for simplicity we keep it as it is.

\end{remark}

\begin{theorem} [More general asymptotic guaranties]\label{thm:generalAymptResult}
Under Assumption \ref{fullDetailRegCond}, for any $\env \in [\numenv]$ set $\weight_{\env} = \nObs{\env} / \nObs{}$ and let $\quantEst_{\nObs{}} \doteq \sum_{\env = 1}^\numenv \weight_\env  \quantEst_{\fct}^{\env}$ for $\fct = \hat{\fct}_{\nObs{}}$, where $\quantEst_{\fct}^{\env}$ is defined as in Definition \ref{def:quantEstDef}.
Call $\hat{t}_{\confLevel}$ the $(1-\confLevel)$-quantile of the following variable:
\begin{equation} 
    \frac{1}{\nObs{}}\sum_{\env = 1}^{\numenv-1} \int_0^1 \brownBridge_{\env}^2(t) \, \quantEst_{\nObs{}}^2(t) dt,
\end{equation}
where $(\brownBridge_{\env}(t))_{\env = 1}^{\numenv-1}$ are $\numenv -1$ independent Brownian bridges between $0$ and $1$. Rejecting $\Tilde{\hyp}_{0}(\envSet)$ whenever we have $\hat{\minWV}_{\weights }(\fctClass_{\nObs{}}') > \hat{t}_{\confLevel}$ forms a consistent test of (asymptotic) level $\confLevel$ when $\nObs{} \rightarrow \infty$. That is:
\[
\text{Under } \Tilde{\hyp}_{0}(\envSet): \, \limsup_{\nObs{} \rightarrow \infty} \, \mathbb{P}(\hat{\minWV}_{\weights }(\fctClass_{\nObs{}}') > \hat{t}_{\confLevel}) \leq \confLevel, \,\, \text{and under } \Tilde{\hyp}_{1}(\envSet): \lim_{\nObs{} \rightarrow \infty}\mathbb{P}(\hat{\minWV}_{\weights }(\fctClass_{\nObs{}}') > \hat{t}_{\confLevel}) = 1.
\]
\end{theorem}
Theorem \ref{thm:asymptGuar} is a direct consequence of Theorem \ref{thm:generalAymptResult} by taking $\fctClass' \in \{ \fctClass_{-\predIdx}, \predIdx \in [\numpred] \}$, $\fctClass'_{\nObs{}} = \fctClass', \forall \nObs{}$ and assuming that Assumption \ref{fullDetailRegCond} is true for any of these $\fctClass'$'s (as summarized in Assumption \ref{regularityAssumptions}). Note that condition (5) from Assumption \ref{fullDetailRegCond} is automatically verified in that case.

\subsection{Proofs of Theorems \ref{thm:asymptGuar} and \ref{thm:generalAymptResult} under $ \Tilde{\hyp}_{0}(\envSet)$} \label{app:th2proofH0}

As we've just mentioned, Theorem \ref{thm:asymptGuar} is a direct consequence of Theorem \ref{thm:generalAymptResult}, therefore we only focus on proving the latter. 
To remain concise, we are going to use directly several technical results that are presented as supporting lemmas and proved in Section \ref{app:supportingLemmasTh2}.

\paragraph{First Steps.}
Let's fix some arbitrary $\delta \in (0,1/2)$ and $\epsilon > 0$. 
Note that by condition (4) of Assumption \ref{fullDetailRegCond}, when $\nObs{} \geq \proofConstA$ (for some constant $\proofConstA$ depending only on $\delta$) with probability at least $1-\delta$ we have $\minimizer \in \Gsetn$. 
In Lemma \ref{lem:th2Lem1} we prove that for any functions $\subFct, \subFct' \in \Czero$ we have:
$$
\left| \sqrt{\wassVar{\probVectBase(\subFct)}{\weights}} - \sqrt{\wassVar{\probVectBase(\subFct')}{\weights}} \right| \leq \| \subFct - \subFct' \|_{\infty}.
$$
As a consequence, with probability at least $1-\delta$:
$$
\left| \sqrt{\hat{\minWV}_{\weights }(\fctClass_{\nObs{}}')} - \sqrt{\hat{\minWV}_{\weights }(\Gset)} \right| = \left| \sqrt{\hat{\minWV}_{\weights }(\Gsetn)} - \sqrt{\hat{\minWV}_{\weights }(\Gset)} \right| \leq \hausDist(\Gsetn, \Gset).
$$
Therefore:
\begin{equation}\label{eq:th2proofEq1}
    \mathbb{P}(\hat{\minWV}_{\weights }(\fctClass_{\nObs{}}') > \hat{t}_{\confLevel}) \leq \mathbb{P}\left(\hat{\minWV}_{\weights }(\Gset) > \hat{t}_{\confLevel} - 2 \sqrt{\hat{\minWV}_{\weights }(\Gset)} \hausDist(\Gsetn, \Gset) - \hausDist^2(\Gsetn, \Gset)\right) + \delta.
\end{equation}
For simplicity we use the notation $q_{1-\confLevel}(T)$ to refer to the $(1-\confLevel)$-quantile of a real variable $T$. 
We are also going to define the following variables for any quantile density function $\quantDensity$ (potentially empirical) that satisfies condition (iii) from Assumption \ref{DelBarrioAssumptions} and $s \in [0,1/2)$:
$$
T_{s}(\quantDensity) \doteq \sum_{\env = 1}^{\numenv-1} \int_{s}^{1-s} \brownBridge_\env^2(t) \quantDensity^2(t) dt,
$$
where $(\brownBridge_{\env}(t))_{\env = 1}^{\numenv-1}$ are $\numenv - 1$ independent Brownian bridges. 
Note that $\nObs{} \hat{t}_{\confLevel} = q_{1-\confLevel}(T_0(\quantEst_{\nObs{}}))$.

First, as a consequence of Theorem \ref{thm:unifBound} and Lemma \ref{lem:th2Lem1} we can show that $\GsetZero$ is an non-empty compact subset of $\Czero$ (see Lemma \ref{lem:th2Lem3} for a proof). 
For any $s \in [0, 1/2)$ and $\confLevel \in (0,1)$, we will define also:
\begin{equation} \label{eq:thm2proofStar1}
    t^\delta_{s,\confLevel} \doteq \inf_{\fct \in \GsetZero } \left( q_{1-\confLevel}(T_{s}(\quantDensity_\fct)) \right),
\end{equation}
where $\quantDensity_\fct$ is the quantile density of $\probBase_{\env}(\fct)$ -- note that since $\fct \in \GsetZero$, $\probBase_{\env}(\fct)$ is identical for all $\env \in [\numenv]$.

Notice that the infimum in \eqref{eq:thm2proofStar1} is always attained by some function in $\GsetZero$, we prove this fact in Lemma \ref{lem:th2Lem8}.
In particular we let $\fct_{0}^\delta \in \GsetZero$ to be a function such that:
\begin{equation} \label{eq:thm2proofBloop}
    t_{0,\confLevel}^\delta = q_{1-\confLevel}(T_0(\quantDensity_{\fct_{0}^\delta})).
\end{equation}
Since $\hat{\minWV}_{\weights }(\Gset) \leq \wassVar{\empProbBaseVect(\fct_{0}^\delta)}{\weights}$, from inequality \eqref{eq:th2proofEq1} we get:
$$
\mathbb{P}(\hat{\minWV}_{\weights }(\fctClass_{\nObs{}}') > \hat{t}_{\confLevel}) \leq \mathbb{P}\left(\nObs{} \wassVar{\empProbBaseVect(\fct_{0}^\delta)}{\weights} > \nObs{} \hat{t}_{\confLevel} - 2 \sqrt{\nObs{}\wassVar{\empProbBaseVect(\fct_{0}^\delta)}{\weights}} \nObs{}^{1/2} \hausDist(\Gsetn, \Gset) - \nObs{}\hausDist^2(\Gsetn, \Gset)  \right) + \delta.
$$
Furthermore, because $\nObs{}^{1/2} \hausDist(\Gsetn, \Gset) = o_{\nObs{}}(1)$ and, from Proposition \ref{prop:WVAsymptLimit}, $\nObs{}\wassVar{\empProbBaseVect(\fct_{0}^\delta)}{\weights}$ converges in distribution, we get that the term in the RHS of $\hat{\minWV}_{\weights}(\fctClass_{\nObs{}}')$ above converges in probability to $0$, that is for $\nObs{} \geq \proofConstB$ we have with probability at least $1 - \delta$:
$$
2 \sqrt{\nObs{} \wassVar{\empProbBaseVect(\fct_{0}^\delta)}{\weights}}  \nObs{}^{1/2} \hausDist(\Gsetn, \Gset) + \nObs{} \hausDist^2(\Gsetn, \Gset) \leq \epsilon,
$$
Therefore: 
$$
\limsup_{\nObs{} \rightarrow \infty} \mathbb{P}(\hat{\minWV}_{\weights }(\fctClass_{\nObs{}}') > \hat{t}_{\confLevel})  \leq \mathbb{P}\left(\nObs{} \wassVar{\empProbBaseVect(\fct_{0}^\delta)}{\weights} > \nObs{} \hat{t}_{\confLevel} - \epsilon \right) + 2 \delta.
$$
Finally, in Lemma \ref{lem:th2Lem8} we show that there exist $(s, \confLevel') \in (0,1/2) \times (0,1)$ such that $\confLevel' > \confLevel$ and $t^\delta_{s,\confLevel'} + \epsilon \geq t_{0,\confLevel}^\delta$. 
From now on we fix $s$ and $\confLevel'$ to be as such. 
Note that $\nObs{} \hat{t}_{\confLevel} = q_{1-\confLevel}(T_0(\quantEst_{\nObs{}})) \geq q_{1-\confLevel}(T_s(\quantEst_{\nObs{}}))$.
We have:
\begin{equation} \label{eq:thm2proofStar2}
    \limsup_{\nObs{} \rightarrow \infty} \mathbb{P}(\hat{\minWV}_{\weights }(\fctClass_{\nObs{}}') > \hat{t}_{\confLevel})  \leq  \limsup_{\nObs{} \rightarrow \infty} \mathbb{P}\left(\nObs{} \wassVar{\empProbBaseVect(\fct_{0}^\delta)}{\weights} > t_{0,\confLevel}^{\delta} + \left(q_{1-\confLevel}(T_s(\quantEst_{\nObs{}})) -t_{0,\confLevel}^\delta \right) - \epsilon \right) + 2 \delta.
\end{equation}
What's left is to study the asymptotic behavior of $q_{1-\confLevel}(T_s(\quantEst_{\nObs{}})) -t_{0,\confLevel}^\delta$.

\paragraph{Asymptotic Behavior of $q_{1-\confLevel}(T_s(\quantEst_{\nObs{}})) -t_{0,\confLevel}^\delta$.}
In the following, for any $\fct \in \Gset$ we will call $\quantDensity_{\fct}^\env$ the quantile density of $\probBase_\env(\fct)$ and $\quantDensity_\fct \doteq \sum_{\env = 1}^{\numenv} \weight_\env \quantDensity_{\fct}^\env$.
Similarly, we call $\quantEst_{\fct} \doteq \sum_{\env = 1}^\numenv \weight_\env \quantEst_{\fct}^{\env}$, where $\quantEst_{\fct}^\env$ is the kernel estimator from Definition \ref{def:quantEstDef} with bandwidth $h_\env = \beta \nObs{\env}^{-1/3}$ for some fixed $\beta>0$.

First, we would like to measure the convergence of the quantile density estimator $\quantEst_{\nObs{}}$ toward $\quantDensity_{\minimizer}$.
However, since $\minimizer$ is random and not fixed, we cannot directly use convergence results like the one we proved in Theorem \ref{thm:quantEstErr} from Section \ref{app:quantEstConv}. 
Instead, we will consider a finite cover of $\Gset$, fine enough such that there exists a function in this cover not too far from $\minimizer$, so that their corresponding quantile density estimators are close to each other; 
and coarse enough so we can make sure that with high probability the quantile density estimators at each of the functions in this cover uniformly converge.
Let $\epsilon_{\nObs{}} \doteq \nObs{}^{-1/3 - r}$ with $r \doteq \frac{1}{3}\frac{2 - \numpred / \sobolevDegree}{2 + \numpred / \sobolevDegree} > 0$ (recall that $\sobolevDegree > \numpred / 2$). 
By Remark \ref{rem:usefulPropSobolev} we know that there is a constant $C_{\delta}$ such that:
$$
\log\left( \mathcal{N}\left( \Gset, \epsilon_{\nObs{}} \right)\right) \leq C_\delta \epsilon_{\nObs{}}^{- \numpred / \sobolevDegree}.
$$
Call $\mathcal{C}_{\nObs{}}^\delta$ the corresponding $\epsilon_{\nObs{}}$-cover of $\Gset$, i.e.~$\log\left| \mathcal{C}_{\nObs{}}^\delta \right| \leq C_{\delta} \epsilon_{\nObs{}}^{-\numpred / \sobolevDegree}$. 
Next, fix $\epsilon' \in (0,1/2)$ and set $\delta_{\nObs{}} \doteq \delta / \left| \mathcal{C}_{\nObs{}}^\delta \right|$.
Recall that the bandwidths for the kernel estimators are all set to be $h_\env = \beta \nObs{\env}^{-1/3}$ for $\env \in [\numenv]$, and that $\nObs{} \geq \nObs{\env} \geq \lambda \nObs{}$ for some constant $\lambda$ under Assumption \ref{fullDetailRegCond}.

In Lemma \ref{lem:th2Lem4} we show that (for any $s \in (0,1/2)$) there exists a constant $C_{\delta, s}$ such that for any $\subFct \in \Gset$ and $\env \in [\numenv]$ we have $\| \quantDensity^\env_{\subFct} \|_{s, \infty} \leq C_{\delta, s}$ and  $\| {\quantDensity^\env_{\subFct}}' \|_{s, \infty} \leq C_{\delta, s}$, where we define $\| \quantDensity \|_{s,\infty} \doteq \sup_{t \in [s,1-s]} | \quantDensity(t) |$.
Because of that, we can apply Theorem \ref{thm:quantEstErr} from Section \ref{app:quantEstConv}; 
we then get that there exist two constants $\proofConstC$ and $\proofConstD$ such that for any $\nObs{} \geq \proofConstC$ and for any $\subFct \in \mathcal{C}_{\nObs{}}^\delta$, with probability at least $1 - \delta_{\nObs{}}$:
$$
\| \quantEst_{\subFct} - \quantDensity_{\subFct} \|_{s, \infty} \leq \proofConstD \left( \nObs{}^{-1/3} + \frac{\log(\nObs{}/\delta_{\nObs{}})}{\nObs{}^{2/3}} + \sqrt{ \frac{\log(\nObs{}/\delta_{\nObs{}})}{\nObs{}^{2/3}}} \right).
$$
Furthermore, one can easily check that $\log(\nObs{}/\delta_{\nObs{}}) \nObs{}^{-2/3} \leq \log(\nObs{} / \delta)\nObs{}^{-2/3} + C_{\delta} \epsilon_{\nObs{}}^{- \numpred / \sobolevDegree} \nObs{}^{-2/3} = \log(\nObs{} / \delta)\nObs{}^{-2/3} + C_{\delta} \nObs{}^{-2 r}$.
By a union bound, we therefore get that there exists a constant $\proofConstE$ such that if $\nObs{} \geq \proofConstE$, with probability at least $1-\delta$ we have:
\begin{equation} \label{eq:thm2proofStar3}
    \forall \subFct \in \mathcal{C}_{\nObs{}}^{\delta}, \quad \| \quantEst_{\subFct} - \quantDensity_{\subFct} \|_{s, \infty} \leq \epsilon'.
\end{equation}

Besides, from Lemma \ref{lem:th2Lem9} we get that when $\nObs{} \geq \proofConstF$ (for some constant $\proofConstF$), with probability at least $1-2\delta$:
\begin{equation} \label{eq:thm2proofStar4}
    \minimizer \in \Gset, \quad \text{and} \quad \exists \bar{\fct}_{\nObs{}} \in \GsetZero \quad \text{s.t.} \quad \| \bar{\fct}_{\nObs{}} - \minimizer \|_{\infty} < \epsilon'.
\end{equation}
Under the above event, there exists a function $\bar{\subFct}_{\nObs{}} \in \mathcal{C}^{\delta}_{\nObs{}}$ such that $\| \minimizer - \bar{\subFct}_{\nObs{}} \|_{\infty} \leq \epsilon_{\nObs{}}$.
By Lemma \ref{lem:th2Lem10} we can choose a constant $\proofConstG$ such that whenever $\nObs{} \geq \proofConstG$, we have both $4 \beta^{-1} L \nObs{}^{-r} \leq \epsilon'$ and:
\begin{equation*}
    \| \quantEst_{\nObs{}} - \quantEst_{\bar{\subFct}_{\nObs{}}} \|_{s, \infty} \leq 4 \beta^{-1} L \nObs{}^{1/3} \epsilon_{\nObs{}} = 4 \beta^{-1} L  \nObs{}^{-r} \leq \epsilon'.
\end{equation*}
For simplicity, let's also assume that $\proofConstG$ was chosen so that $\epsilon_{\nObs{}} \leq \epsilon'$ for $\nObs{} \geq \proofConstG$.
Under the event of equation \eqref{eq:thm2proofStar4}, notice that for $\nObs{} \geq \max(\proofConstF, \proofConstG)$, we have $\| \bar{\subFct}_{\nObs{}} - \bar{\fct}_{\nObs{}}\| \leq 2 \epsilon' < 1$, and Lemma \ref{lem:th2Lem5} implies:
$$
\left(\int_{s}^{1-s} \left( \quantDensity_{\bar{\subFct}_{\nObs{}}}(t) - \quantDensity_{\bar{\fct}_{\nObs{}}}(t)\right)^2 dt\right)^{1/2} \leq A_{\delta, s}^{1/2} (2 \epsilon')^{1/3},
$$
for some constant $A_{\delta, s}$. 

Consider the decomposition $\quantEst_{\nObs{}} - \quantDensity_{\bar{\fct}_{\nObs{}}} = \quantEst_{\nObs{}} - \quantEst_{\bar{\subFct}_{\nObs{}}} + \quantEst_{\bar{\subFct}_{\nObs{}}} - \quantDensity_{\bar{\subFct}_{\nObs{}}} + \quantDensity_{\bar{\subFct}_{\nObs{}}} - \quantDensity_{\bar{\fct}_{\nObs{}}}$.
Combining the events of \eqref{eq:thm2proofStar3} and \eqref{eq:thm2proofStar4}, by a union bound, we get that if $\nObs{} \geq \proofConstH \doteq \max(\proofConstE, \proofConstF, \proofConstG)$, with probability at least $1-3\delta$:
$$
\left(\int_{s}^{1-s} \left( \quantEst_{\nObs{}}(t) - \quantDensity_{\bar{\fct}_{\nObs{}}}(t)\right)^2 dt\right)^{1/2} \leq \epsilon' + \epsilon' + A_{\delta, s}^{1/2} (2 \epsilon')^{1/3}.
$$
Which, by Lemma \ref{lem:th2Lem7} implies that:
$$
\sqrt{q_{1-\confLevel}(T_{s}(\quantEst_{\nObs{}}))} \geq \sqrt{q_{1-\confLevel'}(T_s(\quantDensity_{\bar{\fct}_{\nObs{}}}))} - \frac{(\numenv -1)^{1/2}}{\confLevel' - \confLevel}\left(2 \epsilon' + A^{1/2}_{\delta, s}(2 \epsilon')^{1/3}\right).
$$
Because of our choice of $\confLevel', s$ and that $\bar{\fct} \in \GsetZero$, by definition $q_{1-\confLevel'}(T_s(q_{\bar{\fct}_{\nObs{}}})) \geq t_{s, \confLevel'}^\delta \geq t_{0,\confLevel}^\delta - \epsilon$.
Also, since $\epsilon'$ was arbitrarily chosen, we can take $\epsilon'$ small enough to obtain the following result:
For $\nObs{} \geq \proofConstH$, we have with probability at least $1 - 3\delta$:
\begin{equation} \label{eq:thm2proofStar5}
    q_{1-\confLevel}(T_s(\quantEst_{\nObs{}})) - t^\delta_{0,\confLevel} \geq - 2 \epsilon.
\end{equation}

\paragraph{Conclusion.}
Combining \eqref{eq:thm2proofStar2} with \eqref{eq:thm2proofStar5} yields the following result:
\begin{align*}
    \limsup_{\nObs{} \rightarrow \infty} \mathbb{P}(\hat{\minWV}_{\weights }(\fctClass_{\nObs{}}') > \hat{t}_{\confLevel}) & \leq  \limsup_{\nObs{} \rightarrow \infty} \mathbb{P}\left(\nObs{} \wassVar{\empProbBaseVect(\fct_{0}^\delta)}{\weights} > t_{0,\confLevel}^{\delta} - 3 \epsilon \right) + 5 \delta \\
    & \leq \mathbb{P}\left(T_0(q_{f_{0}^\delta}) > t_{0,\confLevel}^\delta - 3 \epsilon \right) + 5 \delta,
\end{align*}
where the last inequality (which is actually an equality) comes from Proposition \ref{prop:WVAsymptLimit}.
Since $\epsilon$ was arbitrarily chosen, we can send it to $0$ and get that $\mathbb{P}\left(T_0(q_{f_{0}^\delta})> t_{0,\confLevel}^\delta - 3 \epsilon \right)$ goes to $\confLevel$ by equation \eqref{eq:thm2proofBloop}.
Finally, $\delta$ was also arbitrary, sending it to $0$ then yields our result.

\subsection{Proofs of Theorems \ref{thm:asymptGuar} and \ref{thm:generalAymptResult} under $ \Tilde{\hyp}_{1}(\envSet)$}

Again, since Theorem \ref{thm:asymptGuar} is a direct consequence of Theorem \ref{thm:generalAymptResult} we only focus on proving the latter.
The proof of the consistency of the test proposed in Theorem \ref{thm:generalAymptResult} is achieved in two steps: 
We show that, under $\Tilde{\hyp}_{1}(\envSet)$, the statistic $\hat{\minWV}_{\weights }(\fctClass_{\nObs{}}')$ is asymptotically lower bounded by a positive constant independent of $\nObs{}$, while the threshold $\hat{t}_{\confLevel}$ converges in probability toward $0$.
Furthermore, we set $\delta \in (0,1)$.

\paragraph{Lower Bound on $\hat{\minWV}_{\weights }(\fctClass_{\nObs{}}')$.}
In the first steps of the proof of Theorem \ref{thm:generalAymptResult} under $\Tilde{\hyp}_{0}(\envSet)$ (see Section \ref{app:th2proofH0}) we observed that with probability at least $1-\delta$:
$$
\left| \sqrt{\hat{\minWV}_{\weights }(\fctClass_{\nObs{}}')} - \sqrt{\hat{\minWV}_{\weights }(\Gset)} \right| \leq \hausDist(\Gsetn, \Gset).
$$
Therefore, we can derive this first lower bound on $\hat{\minWV}_{\weights }(\fctClass_{\nObs{}}')$:
\begin{equation} \label{eq:th2proofH1eq1}
  \hat{\minWV}_{\weights }(\fctClass_{\nObs{}}') \geq \left(\sqrt{\hat{\minWV}_{\weights }(\Gset)} - \hausDist(\Gsetn, \Gset) \right)^2 \mathbb{1}\left\{\sqrt{\hat{\minWV}_{\weights }(\Gset)} - \hausDist(\Gsetn, \Gset) \geq 0\right\}. 
\end{equation}

Now, we are going to derive a lower bound on $\hat{\minWV}_{\weights }(\Gset)$. 
In the proof of Lemma \ref{lem:th2Lem3} we showed that $\radComp_{\nObs{\env}}(\Gset) = O(\nObs{\env}^{-1/2})$ by using the fact that the Sobolev space $\sobolev^{\sobolevDegree,2}(\compactset)$ is norm-equivalent to the RKHS generated by the Matérn kernel (see Remark \ref{rem:usefulPropSobolev}).
By Theorem \ref{thm:unifBound}, such a property is indeed useful for deriving bounds on $\hat{\minWV}_{\weights }(\Gset)$.
In particular, from equation \eqref{eq:th2lem3eq1} of the proof of Lemma \ref{lem:th2Lem3} we have that with probability at least $1-\delta$:
\begin{equation} \label{eq:th2proofH1eq2}
     \forall \subFct \in \Gset, \quad \left| \wassVar{\empProbBaseVect(\subFct)}{\weights} - \wassVar{\probVectBase(\subFct)}{\weights} \right| \leq c_{\delta} \frac{\log^2(\nObs{})}{\sqrt{\nObs{}}},
\end{equation}
for some constant $c_{\delta}$ that depends only on $\delta$.

Furthermore, by Lemmas \ref{lem:th2Lem1} and \ref{lem:th2Lem2}, since $\Gset$ is a compact subset of $\Czero$ (see Remark \ref{rem:usefulPropSobolev}), under condition (3) of Assumption \ref{fullDetailRegCond} and $\tilde{H}_{1}(\envSet)$, we can find $\gamma_0 > 0$ independent of $\weights$ such that 
\begin{equation} \label{eq:th2proofH1eq3}
    \forall \subFct \in \Gset, \quad  \wassVar{\probVectBase(\subFct)}{\weights} \geq \gamma_0 > 0.
\end{equation}

Combining \eqref{eq:th2proofH1eq2} and \eqref{eq:th2proofH1eq3} it is easy to see that for $\nObs{}$ large enough, we have with probability at least $1-\delta$ that $0< \gamma_0 / 2 \leq \hat{\minWV}_{\weights }(\Gset)$.
Recall also that by condition (5) of Assumption \ref{fullDetailRegCond} we have $\hausDist(\Gsetn, \Gset) = o(n^{-1/2})$.
Therefore, from \eqref{eq:th2proofH1eq1} we get that there exist constants $\gamma > 0$ and $c(\delta, \gamma) > 0$ such that for any $\nObs{} \geq c(\delta, \gamma)$, with probability at least $1-\delta$:
\begin{equation} \label{eq:th2proofH1eq4}
    \forall \weights \in \weightSpace \text{ s.t. } \min_{\env \in [\numenv]} \weight_{\env} \geq \lambda, \quad \hat{\minWV}_{\weights }(\fctClass_{\nObs{}}') \geq \gamma > 0,
\end{equation}
where $\lambda > 0$ is from condition (3) of Assumption \ref{fullDetailRegCond}.

\paragraph{$\hat{t}_{\confLevel}$ Converges to $0$.}
Call $Z_{\env}^\delta \doteq \sup_{\subFct \in \Gset} |\target{\env} - \subFct(\pred{\env})|$, for $\env \in [\numenv]$.
Since $\Gset$ is a bounded subset in $\Czero$ and $\target{\env}$ is sub-Gaussian from condition (1) of Assumption \ref{fullDetailRegCond}, we have that $Z_{\env}^\delta$ is sub-Gaussian $\mathcal{G}(\mu_{\env}^\delta, \sigma_{\env}^\delta)$ for some mean $\mu_{\env}^\delta$ and sub-Gaussian parameter $\sigma_{\env}^\delta$ (see Remark \ref{app:shortDiscussTh1}).
Let $\mu_\delta \doteq \max_{\env \in [\numenv]} \mu_{\env}^\delta$ and $\sigma_\delta \doteq \max_{\env \in [\numenv]} \sigma_{\env}^\delta$.
Call also for any $\env \in [\numenv]$ and $i \in [\nObs{\env}]$, $z_{i,e}^\delta \doteq \sup_{\subFct \in \Gset} | \targetObs_{i}^\env - \subFct(\predObs_{i}^\env)|$ and set $m_{\delta} \doteq \max_{\env \in [\numenv]} \max_{i \in [\nObs{\env}]} z_{i,e}^\delta$.
By a union bound and because the $Z_{\env}^\delta$'s are sub-Gaussian, we have (see also Chapter 2 from \cite{wainwright2019high}):
\begin{equation}  \label{eq:th2proofH1Star1}
    \mathbb{P}(m_\delta > M) \leq \sum_{\env = 1}^{\numenv} \nObs{\env} \mathbb{P}(Z_{\env} > M) \leq \nObs{} \exp\left( - \frac{(M - \mu_\delta)^2}{2 \sigma_\delta^2} \right) \leq \delta,
\end{equation} 
when we set $M = \mu_\delta + \sigma_{\delta} \sqrt{2 \log(\nObs{} / \delta)}$.
Therefore with probability at least $1 - \delta$ we have $m_\delta \leq \mu_\delta + \sigma_{\delta} \sqrt{2 \log(\nObs{} / \delta)}$.

Recall that we set the bandwidth $h_\env$ for the kernel estimator in Definition \ref{def:quantEstDef} at $h_\env = \beta \nObs{\env}^{-1/3}$, for some constant $\beta > 0$.
Notice that from Definition \ref{def:quantEstDef}, for any $\env \in [\numenv]$, $\fct \in \Gset$ and $t \in [0,1]$: 
\begin{equation}\label{eq:th2proofH1Star2}
  | \quantEst_{\fct}^\env(t) | \leq \sum_{i = 2}^{\nObs{\env}} (\residual_{(i)}^\env(\fct) - \residual_{(i-1)}^\env(\fct)) \frac{\| \kernel \|_{\infty}}{h_\env} \leq (\residual_{(\nObs{\env})}^\env(\fct) - \residual_{(1)}^\env(\fct)) \nObs{}^{1/3} \beta^{-1} \| \kernel \|_{\infty} \leq 2 m_\delta \nObs{}^{1/3} \beta^{-1} \| \kernel \|_{\infty}.
\end{equation}

Along with condition (4) of Assumption \ref{fullDetailRegCond}, using \eqref{eq:th2proofH1Star1} and \eqref{eq:th2proofH1Star2} yields that there is a constant $c'(\delta)$ such that for any $\nObs{} \geq c'(\delta)$ we have with probability at least $1-2\delta$:
$$
\minimizer \in \Gset, \quad \text{and} \quad \| \quantEst_{\nObs{}} \|_{\infty} \leq 2 \nObs{}^{1/3} \beta^{-1} \| \kernel \|_{\infty} (\mu_\delta + \sigma_{\delta} \sqrt{2 \log(\nObs{} / \delta)}).
$$
Furthermore, if we use the notations from Section \ref{app:th2proofH0}:
$$
\hat{t}_{\confLevel} = \frac{1}{\nObs{}} q_{1-\confLevel}(T_{0}(\quantEst_{\nObs{}})) \leq \frac{1}{\nObs{}\confLevel} \mathbb{E}\left[ \sum_{\env = 1}^{\numenv-1} \int_{0}^{1} \brownBridge_{\env}^2(t)\quantEst_{\nObs{}}^2(t) dt \right] = \frac{(\numenv - 1)}{\nObs{}\confLevel} \int_0^1 t (1-t) \quantEst_{\nObs{}}^2(t) dt.
$$

Therefore for any $\nObs{} \geq c'(\delta)$, with probability at least $1-2\delta$ we have:
\begin{equation}\label{eq:th2proofH1eq5}
    \hat{t}_{\confLevel} \leq 4 \frac{(E - 1) \| \kernel \|^2_{\infty}}{n^{1/3} \confLevel \beta^2} \left(\mu_\delta + \sigma_{\delta} \sqrt{2 \log(\nObs{} / \delta)} \right)^2 \xrightarrow[\nObs{} \rightarrow \infty]{} 0.
\end{equation}

\paragraph{Conclusion.} As a consequence of both \eqref{eq:th2proofH1eq4} and \eqref{eq:th2proofH1eq5}, we get:
$$
\liminf_{\nObs{} \rightarrow \infty}\mathbb{P}(\hat{\minWV}_{\weights }(\fctClass_{\nObs{}}') > \hat{t}_{\confLevel}) = 1 - 3 \delta.
$$
Since $\delta$ was chosen arbitrarily, we can send it to $0$; and it concludes our proof.

\subsection{Proof of Theorem \ref{thm:multCorrection}}
First, notice that by identifiability it is direct that $\WVMidentifPred = \causalPred$ (we prove this fact in Lemma \ref{lem:th2Lem11}).
For each $k \in [\numpred]$ we will call $\hat{t}_{\confLevel}^k$ the threshold used in Theorem \ref{thm:asymptGuar} for the statistic $\hat{\minWV}_{\weights}(\fctClass_{-k})$. 
By a union bound:
$$
\mathbb{P}(\WVMest = \WVMidentifPred) \geq 1 - \mathbb{P}\left(\exists k \notin \WVMidentifPred \text{ s.t. } \hat{\minWV}(\fctClass_{-k}) > \hat{t}_{\confLevel}^k\right) - \sum_{k \in \WVMidentifPred} \mathbb{P}\left(\hat{\minWV}(\fctClass_{-k}) \leq \hat{t}_{\confLevel}^k\right).
$$

First notice that, by Theorem \ref{thm:asymptGuar}, for $k \in \WVMidentifPred$, $\mathbb{P}(\hat{\minWV}(\fctClass_{-k}) \leq \hat{t}_{\confLevel}^k) \rightarrow_{n \rightarrow \infty} 0$.
Then, because that $\hat{\minWV}(\fctClass_{-k}) \leq \wassVar{\empProbBaseVect(\fct^*)}{\weights}$ for $k \notin \WVMidentifPred = \causalPred$ we have:
$$
\mathbb{P}\left(\exists k \notin \WVMidentifPred \text{ s.t. } \hat{\minWV}(\fctClass_{-k}) > \hat{t}_{\confLevel}^k\right) \leq \mathbb{P}\left(\exists k \notin \WVMidentifPred \text{ s.t. } n \wassVar{\empProbBaseVect(\fct^*)}{\weights} > n \hat{t}_{\confLevel}^k\right).
$$

Set any $\delta \in (0,1/2)$ and $\epsilon > 0$.
Using the notations from Section \ref{app:th2proofH0}, we can see that by identifiability, for any $k \in [\numpred]$, we have the corresponding $t_{0,\confLevel}^\delta$ (when we set $\fctClass' = \fctClass_{-k}$ in the proof of Theorems \ref{thm:asymptGuar} and \ref{thm:generalAymptResult} in Section \ref{app:th2proofH0}) is actually equal to $q_{1-\confLevel}(T_{0}(\quantDensity_{f^*}))$.
Hence the result $\eqref{eq:thm2proofStar5}$ from Section \ref{app:th2proofH0} implies that for any $k\in[\numpred]$, we have a constant $c_k$ such that whenever $n \geq c_k$, with probability at least $1-3\delta$:
$$
n \hat{t}_{\confLevel}^k - q_{1-\confLevel}(T_0(\quantDensity_{\fct^*})) \geq - 2 \epsilon.
$$

Therefore by a union bound and using Proposition \ref{prop:WVAsymptLimit} we get:
\begin{align*}
\limsup_{n \rightarrow \infty}   \mathbb{P}\left(\exists k \notin \WVMidentifPred \text{ s.t. } n \wassVar{\empProbBaseVect(\fct^*)}{\weights} > n \hat{t}_{\confLevel}^k\right) & \leq   \limsup_{n \rightarrow \infty} \mathbb{P}\left(n \wassVar{\empProbBaseVect(\fct^*)}{\weights} >q_{1-\confLevel}(T_0(\quantDensity_{\fct^*})) - 2 \epsilon\right) + 3 \numpred \delta \\
& \leq \mathbb{P}\left(T_0(\quantDensity_{\fct^*})  >q_{1-\confLevel}(T_0(\quantDensity_{\fct^*})) - 2 \epsilon\right) + 3 \numpred \delta.
\end{align*}
Since $\delta$ and $\epsilon$ can be chosen arbitrarily small and $\mathbb{P}\left(T_0(\quantDensity_{\fct^*})  >q_{1-\confLevel}(T_0(\quantDensity_{\fct^*}))\right) = \confLevel$, the above steps implies that:
$$
\liminf_{n \rightarrow \infty} \mathbb{P}(\WVMest = \WVMidentifPred) \geq 1 - \confLevel.
$$

\subsection{High Probability Error Bounds for the Quantile Density Estimator} \label{app:quantEstConv}

We derive here high probability bounds on the supremum of the absolute difference between a quantile density $\quantDensity$ and its estimator $\quantEst$, as defined in Definition \ref{def:quantEstDef}, over intervals of the form $[s, 1-s]$, where $s \in (0,1/2)$.
Such high probability bounds are needed for our proof of Theorem \ref{thm:asymptGuar} and, to the best of our knowledge, cannot be found in the literature.
The known theoretical guarantees for kernel quantile or quantiles density estimators \citep{falk1986estimation, sheather1990kernel, jones1992estimating} only focus on the mean square error at a fixed point in $[0,1]$ and require that $\quantDensity$ is twice differentiable, an assumption we do not make.
Therefore the following result can also be of independent interest.

\begin{theorem}[Quantile density error bounds] \label{thm:quantEstErr}
Let $(Z_i)_{i=1}^n \in \mathbb{R}$ be $n > 1$ i.i.d.~samples of a distribution with twice differentiable CDF $\cdf$ such that its quantile function $\cdf^{-1}$ is also twice differentiable on $(0,1)$ with first derivative $\quantDensity$ and second derivative $\quantDensity'$.

For any $s \in (0, 1/2)$, assume there is a constant $C_s > 0$ such that $\quantDensity(t) \vee |\quantDensity'(t)| < C_s $ for any $t \in [s, 1-s]$.
Call $(Z_{(i)})_{i=1}^n$ the order statistics of the sample $(Z_i)_{i=1}^n$ and let:
$$
\quantEst(t) = \sum_{i = 2}^{n} \left(Z_{(i)} - Z_{(i-1)} \right) \kernel_{h}\left(t - \frac{i-1}{n}\right),
$$
where $h > 1/n$ and $\kernel_h(\cdot) \doteq h^{-1} \kernel(\cdot / h)$ with $\kernel$ a $L$-Lipschitz kernel supported on $[-1,1]$ such that $\int \kernel(u) du = 1$. 
Finally, call also $\| \kernel \|_{\infty} = \sup_{u \in [-1,1]} |\kernel(u)|$. For $\delta \in (0,1)$ and $s \in (0, 1/2)$, whenever $n$ and $h$ satisfy the following condition:
\begin{equation} \label{eq:quantErrCond}
    h + \sqrt{\frac{3 \log(6 n / \delta)}{n h - 1}} + \frac{1}{n} \leq \frac{s}{2},
\end{equation}
we have that with probability at least $1 - \delta$:
\begin{equation} \label{eq:quantErrBound}
  \sup_{t \in [s, 1-s]} | \quantEst(t) - \quantDensity(t)| \leq \frac{A}{h}  \left( 2 h + \sqrt{\frac{ \log(6 n / \delta)}{2 n}} \right)^2 + B_{n,h} \sqrt{\frac{3 \log(6 n / \delta)}{n h -1}} + \frac{C}{n h}, 
\end{equation}
where $A \doteq 2 C_{s/2} \| \kernel \|_\infty$, $B_{n,h} \doteq C_s \left( 12 \| \kernel \|_\infty + \frac{L}{n h^{3/2}}\right)$ and $C \doteq 11 C_s L$.
\end{theorem}
\begin{proof}
First note that $(Z_{(i)})_{i=1}^n \stackrel{d}{=} (\cdf^{-1}(U_{(i)}))_{i = 1}^n$ where $U_{(i)}$ is the $i^{\text{th}}$ order statistic of $n$ i.i.d.~uniform variables on $[0,1]$. 
It is also well-known that $(U_{(i)})_{i = 1}^n \stackrel{d}{=} \left(S_i/S_{n+1}\right)_{i = 1}^n$ where $S_{i} = \sum_{j = 1}^{i} \xi_j$ and $\left( \xi_j \right)_{j =1}^{n+1} \stackrel{i.i.d.}{\sim} \text{exp}(1)$ (see Section 2 from \cite{del2005asymptotics} for instance).
As we are only interested in bounds in probability, we can actually replace $Z_{(i)}$ by $F^{-1}(U_{(i)})$ and $U_{(i)}$ by $S_i/S_{n+1}$ in our analysis.
Finally, note that we have $U_{(i)} \sim \text{Beta}(i, n+1 - i)$.

Let $s \in (0,1/2)$ and take $t \in [s, 1-s]$. 
We are first going to rewrite $\quantEst(t)$ in a more useful way. 
Since $\kernel$ is supported on $[-1,1]$, we have:
$$
\forall i \notin \left[\, \lceil n (t - h) \rceil + 1, \lfloor n (t + h) \rfloor + 1 \, \right], \quad  \kernel_h\left( t - \frac{i-1}{n}\right) = 0.
$$
Under condition \eqref{eq:quantErrCond} we have $\left[\, \lceil n (t - h) \rceil + 1, \lfloor n (t + h) \rfloor + 1 \, \right] \subseteq [2,n]$. Therefore,
$$
\quantEst(t) = \sum_{i = \lceil n (t - h) \rceil + 1}^{ \lfloor n (t + h) \rfloor + 1 } \left( \cdf^{-1}(U_{(i)}) - \cdf^{-1}(U_{(i-1)}) \right) \kernel_{h} \left( t - \frac{i-1}{n} \right).
$$
Finally, by the mean value theorem, we get that there exist variables $\kappa_i \in [U_{(i-1)}, U_{(i)}]$ such that $\cdf^{-1}(U_{(i)}) - \cdf^{-1}(U_{(i-1)}) = (U_{(i)} - U_{(i-1)}) \quantDensity(\kappa_i)$. Hence,
$$
\quantEst(t) = \sum_{i = \lceil n (t - h) \rceil + 1}^{ \lfloor n (t + h) \rfloor + 1 } \left(U_{(i)} - U_{(i-1)} \right) \quantDensity(\kappa_i) \kernel_{h} \left( t - \frac{i-1}{n} \right).
$$
As a consequence, if we define the following three functions depending on $t$:
\begin{gather*}
    A_{n,h}(t) \doteq \left| \sum_{i = \lceil n (t - h) \rceil + 1}^{ \lfloor n (t + h) \rfloor + 1 } \left(U_{(i)} - U_{(i-1)} \right) (\quantDensity(\kappa_i) - \quantDensity(t)) \kernel_{h} \left( t - \frac{i-1}{n} \right) \right|, \\
    B_{n,h}(t) \doteq \left| \sum_{i = \lceil n (t - h) \rceil + 1}^{ \lfloor n (t + h) \rfloor + 1 } \quantDensity(t)\left( U_{(i)} - U_{(i-1)} - \frac{1}{n+1} \right) \kernel_{h} \left( t - \frac{i-1}{n} \right) \right|, \\
    C_{n,h}(t) \doteq \left| \quantDensity(t) -  \sum_{i = \lceil n (t - h) \rceil + 1}^{ \lfloor n (t + h) \rfloor + 1 } \frac{\quantDensity(t)}{n+1} \kernel_{h} \left( t - \frac{i-1}{n} \right)\right|,
\end{gather*}
we can then obviously bound the absolute error as follows:
$$
| \quantEst(t) - \quantDensity(t)| \leq A_{n,h}(t) + B_{n,h}(t) + C_{n,h}(t).
$$

\paragraph{Bounding $A_{n,h}(t)$.}
We first provide a high probability bound for $A_{n,h}(t)$, uniformly for $t \in [s , 1-s]$.
From \cite{marchal2017sub}, Theorem 1, a $\text{Beta}(\alpha, \beta)$ distribution is sub-Gaussian with parameter $1/4(\alpha + \beta + 1)$. 
Therefore (see \cite{wainwright2019high}, Chapter 2):
$$
\forall i \in [n], \forall \epsilon > 0, \quad \mathbb{P}\left(\left| U_{(i)} - \frac{i}{n+1}\right| \geq \epsilon \right) \leq 2 \exp\left( - 2 \epsilon^2 (n+2) \right).
$$
Using a union bound we get that for any $\delta \in (0,1)$, we have with probability at least $1 - \delta / 3$:
\begin{equation}\label{eq:quantBoundBeta}
  \forall i \in [n], \quad \left| U_{(i)} - \frac{i}{n+1} \right| < \sqrt{\frac{\log(6 n / \delta)}{2 (n+2)}}.
\end{equation}
When inequality $\eqref{eq:quantBoundBeta}$ and condition \eqref{eq:quantErrCond} are true, for $i \in \left[\, \lceil n (t - h) \rceil + 1, \lfloor n (t + h) \rfloor + 1 \, \right]$, we have:
\begin{align*}
    \kappa_i \in \left[ U_{(\lceil n (t - h) \rceil)}, \, U_{(\lfloor n (t + h) \rfloor + 1)} \right] & \subseteq \left[ \frac{\lceil n (t - h) \rceil}{n+1} - \sqrt{\frac{\log(6 n / \delta)}{2 (n+2)}}, \, \frac{\lfloor n (t + h) \rfloor + 1}{n+1} + \sqrt{\frac{\log(6 n / \delta)}{2 (n+2)}} \, \right] \\
    & \subseteq \left[ t - h - \sqrt{\frac{\log(6 n / \delta)}{2 (n+2)}} - \frac{1}{n+1}, \, t+h + \sqrt{\frac{\log(6 n / \delta)}{2 (n+2)}} + \frac{1}{n+1} \, \right] \\
    & \subseteq \left[\frac{s}{2}, 1 - \frac{s}{2} \right].
\end{align*}
Therefore, with probability at least $1 - \delta / 3$:
\begin{align*}
    \forall t \in [s, 1 - s], \quad A_{n,h}(t) & \leq \frac{\| \kernel \|_{\infty}}{h} \sum_{i = \lceil n (t - h) \rceil + 1}^{\lfloor n (t + h) \rfloor + 1} \left( U_{(i)} - U_{(i-1)} \right) \sup_{u \in [s/2, 1 - s/2]} \left| \quantDensity'(u) \right| \left| \kappa_i - t \right| \\
    & \leq \frac{C_{s/2} \| \kernel \|_{\infty}}{h} \left( U_{(\lceil n (t - h) \rceil)} - U_{(\lfloor n (t + h) \rfloor + 1)} \right) \left( h + \sqrt{\frac{\log(6 n / \delta)}{2 (n+2)}} + \frac{1}{n+1} \right) \\
    & \leq \frac{2 C_{s/2} \| \kernel \|_{\infty}}{h} \left( h + \sqrt{\frac{\log(6 n / \delta)}{2 (n+2)}} + \frac{1}{n+1} \right)^2 \leq \frac{A}{h} \left( 2 h + \sqrt{\frac{\log(6 n / \delta)}{2 n}} \right)^2.
\end{align*}
\paragraph{Bounding $B_{n,h}(t)$.}
Notice that $2 n h - 2 \leq \lfloor n (t + h) \rfloor - \lceil n (t - h) \rceil \leq 2 n h$. Thus,
\begin{align} \label{eq:quantErrBoundForB}
    B_{n,h}(t) & = |\quantDensity(t)| \left| \sum_{i = \lceil n (t - h) \rceil + 1}^{\lfloor n (t + h) \rfloor + 1} \left( U_{(i)} - U_{(i-1)} - \frac{1}{n+1} \right) \left( \kernel_{h} \left( \frac{\lfloor n t \rfloor}{n} - \frac{i-1}{n} \right) \right. \right. \nonumber \\
    & \quad\quad\quad\quad\quad\quad\quad\quad\quad\quad\quad\quad\quad\left. \left. + \kernel_{h} \left( t - \frac{i-1}{n} \right) - \kernel_{h} \left( \frac{\lfloor n t \rfloor}{n} - \frac{i-1}{n} \right) \right) \right| \nonumber\\
    & \leq C_s \underbrace{  \left| \sum_{i = \lceil n (t - h) \rceil + 1}^{\lfloor n (t + h) \rfloor + 1} \left( U_{(i)} - U_{(i-1)} - \frac{1}{n+1} \right) \kernel_{h} \left( \frac{\lfloor n t \rfloor}{n} - \frac{i-1}{n} \right) \right|}_{\doteq \circled{1}}  \\
    & \quad\quad\quad\quad\quad\quad\quad\quad\quad\quad\quad\quad\quad + \frac{C_s L}{n h^2} \underbrace{\left( U_{(\lfloor n (t + h) \rfloor + 1)} - U_{(\lceil n (t - h) \rceil)} + \frac{2 n h + 1}{n + 1} \right)}_{\doteq \circled{2}}.\nonumber
\end{align}

Starting first with the second term in the last inequality above, from \eqref{eq:quantBoundBeta} we know that we can bound it with probability at least $1-\delta / 3$ as follows:
\begin{equation} \label{eq:quantErrBoundForBTerm2}
    \forall t \in [s, 1-s], \quad \circled{2} \leq  2 h + \sqrt{\frac{2 \log(6 n / \delta)}{n + 2}} + \frac{2 n h + 3}{n+1} \leq 7 h + \sqrt{\frac{2 \log(6 n / \delta)}{n + 2}}.
\end{equation}
Now turning to the first term, we also have:
\begin{align} \label{eq:quantErrBoundForBTerm1}
\circled{1} & = \left| \sum_{i = \lceil n (t - h) \rceil + 1}^{\lfloor n (t + h) \rfloor + 1} \left( \frac{\xi_i}{S_{n+1}} - \frac{1}{n+1} \right) \kernel_{h} \left( \frac{\lfloor n t \rfloor}{n} - \frac{i-1}{n} \right) \right| \nonumber \\
& \leq \frac{n + 1}{S_{n+1}} \cdot \left[ \,\frac{1}{(n+1)h}  \left| \sum_{i = \lceil n (t - h) \rceil + 1}^{\lfloor n (t + h) \rfloor + 1} \left( \xi_i- 1 \right) \kernel \left( \frac{\lfloor n t \rfloor}{n h} - \frac{i-1}{n h} \right) \right| \right. \nonumber \\
& \quad\quad\quad\quad\quad\quad\quad\quad\quad\quad\quad \left. + \frac{1}{(n+1)h} \left| \sum_{i = \lceil n (t - h) \rceil + 1}^{\lfloor n (t + h) \rfloor + 1} \left( \frac{S_{n+1}}{n+1} - 1\right) \kernel \left( \frac{\lfloor n t \rfloor}{n h} - \frac{i-1}{n h} \right) \right| \,  \right] \nonumber \\
& \leq \frac{n + 1}{S_{n+1}} \cdot \left[ \,\frac{1}{(n+1)h}  \left| \sum_{i = \lceil n (t - h) \rceil + 1}^{\lfloor n (t + h) \rfloor + 1} \left( \xi_i- 1 \right) \kernel \left( \frac{\lfloor n t \rfloor}{n h} - \frac{i-1}{n h} \right) \right| + \frac{2 n h + 1}{(n+1)h} \left| \frac{S_{n+1}}{n+1} - 1 \right| \| K \|_{\infty} \,  \right].
\end{align}
And we can in turn bound the last two terms using concentration bounds for sub-exponential variables \citep[eq.~(2.20)]{wainwright2019high}. In particular, for any $\alpha \in \mathbb{R}$, $\alpha(\xi_i - 1)$ is sub-exponential with parameters $(2 \alpha, 2 \alpha)$, which means that for any $i, j, k \in [n]$ s.t. $i + 2 n h - 2 \leq j$:
$$
\mathbb{P}\left( \left|\frac{1}{j - i + 1} \sum_{l = i}^{j} (\xi_l - 1) \kernel \left( \frac{k}{n h} - \frac{l-1}{n h} \right) \right| \geq \epsilon \right) \leq 2 \exp(- (j - i + 1) \epsilon^2 / 8 \| K \|_{\infty}),
$$
as long as $0 \leq \epsilon \leq 2 \| K \|_{\infty}$. Therefore by a union bound we get that with probability at least $1 - \delta / 3$:
\begin{equation}\label{eq:quantErrBoundExpConcentration1}
    \forall i,j,k \in [n] \text{ s.t. } i + 2 n h - 2 \leq j, \left|\frac{1}{j - i + 1} \sum_{l = i}^{j} (\xi_l - 1) \kernel \left( \frac{k}{n h} - \frac{l-1}{n h} \right) \right| < 2 \sqrt{\frac{3 \log(6 n / \delta)}{n h -1}} \| K \|_{\infty},
\end{equation}
whenever $2 \| K \|_{\infty} \sqrt{\frac{3 \log( 6 n / \delta)}{n h - 1}} \leq 2 \| K \|_{\infty}$, which is true under condition \eqref{eq:quantErrCond}. 
Recall again that $2 n h - 2 \leq \lfloor n (t + h) \rfloor - \lceil n (t - h) \rceil \leq 2 n h$ and that $n h > 1$.
Therefore when \eqref{eq:quantErrBoundExpConcentration1} is true, we have:
$$
\forall t \in [s, 1 - s], \quad  \frac{1}{(n+1)h} \left| \sum_{i = \lceil n (t - h) \rceil + 1}^{\lfloor n (t + h) \rfloor + 1} \left( \xi_i - 1 \right) \kernel \left( \frac{\lfloor n t \rfloor}{n h} - \frac{i-1}{n h} \right) \right| \leq 6 \sqrt{\frac{3 \log(6 n / \delta)}{n h -1}} \| K \|_{\infty}.
$$
Similarly to \eqref{eq:quantErrBoundExpConcentration1}, we have with probability at least $1 - \delta / 3$:
\begin{equation} \label{eq:quantErrBoundExpConcentration2}
    \left| \frac{S_{n+1}}{n+1} - 1 \right| \leq 2 \sqrt{\frac{2 \log(6 / \delta)}{n + 1}} \leq \sqrt{\frac{3 \log(6 n / \delta)}{n h -1}} \leq 1/4,
\end{equation}
where the last inequality is implied by condition \eqref{eq:quantErrCond} since $h \leq s/2 \leq 1/4$.

By a union bound over equations  \eqref{eq:quantErrBoundExpConcentration1} and \eqref{eq:quantErrBoundExpConcentration2}, from inequality \eqref{eq:quantErrBoundForBTerm1} we therefore have with probability at least $1 - 2 \delta / 3$:
$$
\forall t \in [s, 1-s], \quad \circled{1} \leq 12 \| K \|_{\infty} \sqrt{\frac{3 \log(6 n / \delta)}{n h - 1}}.
$$
And finally with a union bound over the three events \eqref{eq:quantBoundBeta}, \eqref{eq:quantErrBoundExpConcentration1} and \eqref{eq:quantErrBoundExpConcentration2} we get that, with probability at least $1 - \delta$, the uniform bound we derived for $A_{n,h}(t)$ holds as well as the following, by \eqref{eq:quantErrBoundForB}:
\begin{align*}
 \forall t \in [s, 1-s], \quad B_{n,h}(t) & \leq C_s  \cdot \circled{1} + \frac{C_s L}{n h^2} \cdot \circled{2} \\
 & \leq 12 C_s \| K \|_{\infty} \sqrt{\frac{3 \log(6 n / \delta)}{n h - 1}} + \frac{7 C_s L}{n h} + \frac{C_s L}{n h^{3/2}} \sqrt{\frac{3 \log(6 n / \delta)}{n h - 1}} \\
 & \leq B_{n,h} \sqrt{\frac{3 \log(6 n / \delta)}{n h -1}}  + \frac{7 C_s L}{n h}.
\end{align*}

\paragraph{Bounding $C_{n,h}(t)$.} Recall that $K_h(t - (i-1)/n) = 0$ whenever $i \notin \left[\, \lceil n (t - h) \rceil + 1, \lfloor n (t + h) \rfloor + 1 \, \right]$. Therefore we have:
$$
C_{n,h}(t) =|\quantDensity(t)| \left| 1 -  \frac{1}{n+1}\sum_{i = -\infty}^{+\infty}  \kernel_{h} \left( t - \frac{i}{n} \right)\right| \leq C_s \frac{n}{n+1}\left(\frac{1}{n} + \left| 1 -  \frac{1}{n}\sum_{i = -\infty}^{ +\infty }  \kernel_{h} \left( t - \frac{i}{n} \right)\right|\right).
$$
Recall that by assumption $\int K_h(u) du = 1$, hence:
\begin{align*}
    \left| 1 -  \frac{1}{n}\sum_{i = -\infty}^{ +\infty }  \kernel_{h} \left( t - \frac{i}{n} \right)\right| & \leq \sum_{i = -\infty}^{+\infty} \int_{t - (i+1)/n}^{t - i/n} \left| K_h(u) - K_h(t - i/n) \right| du \\
    & \leq \sum_{i = \lceil n (t - h) \rceil - 1}^{\lfloor n (t + h) \rfloor} \frac{L}{n^2 h^2} \leq \frac{(2 n h + 2) L}{n^2 h^2} \leq \frac{4 L}{n h}.
\end{align*}
Therefore,
$$
\forall t \in [s, 1-s], \quad C_{n,h}(t) \leq \frac{4 C_s L}{n h}.
$$
\end{proof}

\subsection{Supporting Lemmas} \label{app:supportingLemmasTh2}
Note that for the following lemmas, we are going to use the notations from section \ref{app:regCondTh2} without necessarily re-introducing them.

\begin{lemma} \label{lem:th2Lem1} 
Consider any $\subFct, \subFct' \in \Czero$. Under condition (1) of Assumption \ref{fullDetailRegCond}, we have:
$$
\left| \sqrt{\wassVar{\probVectBase(\subFct)}{\weights}} - \sqrt{\wassVar{\probVectBase(\subFct')}{\weights}} \right| \leq \| \subFct - \subFct' \|_{\infty}.
$$
The result above holds also if we replace $\probVectBase(\subFct), \probVectBase(\subFct')$ by their empirical counterparts $\empProbBaseVect(\subFct)$ and $\empProbBaseVect(\subFct')$.
\end{lemma} 
\begin{proof}
Using the notation of Section \ref{sec:implementation}, we call $\cdf^{-1}_{\subFct, \env}$ the quantile function of $\probBase_{\env}(\subFct)$ for $\env \in [\numenv]$, and $\cdf^{-1}_{\subFct', \env}$ is defined similarly for $\subFct'$. 
Then, by equation \ref{eq:explicitWV}:
$$
\wassVar{\probVectBase(\subFct)}{\weights} = \int_{0}^1 \sum_{\env = 1}^\numenv \weight_\env \left( \cdf^{-1}_{\subFct, \env}(t) - \sum_{\env' = 1}^\numenv \weight_{\env'} \cdf^{-1}_{\subFct, \env'} \right)^2 dt.
$$
For short, call $G_{\subFct}(t, \env) = \cdf^{-1}_{\subFct, \env}(t) - \sum_{\env' = 1}^\numenv \weight_{\env'} \cdf^{-1}_{\subFct, \env'}(t)$.
We have that:
$$
\wassVar{\probVectBase(\subFct)}{\weights} = \int G^2_{\subFct}(t, \env) d \lambda \otimes p_{\weights}(t, \env) = \left\|  G_{\subFct} \right\|^2_{\text{L}_2([0,1]\times[\numenv], \lambda \otimes p_{\weights})},
$$
where $\lambda$ is the Lebesgue measure on $[0,1]$ and $p_{\weights}$ is the probability measure on $[\numenv]$ with probabilities $\weights$. Therefore, by the triangular inequality:
$$
\left| \sqrt{\wassVar{\probVectBase(\subFct)}{\weights}} - \sqrt{\wassVar{\probVectBase(\subFct')}{\weights}} \right| \leq \left\|  G_{\subFct} -  G_{\subFct'} \right\|_{\text{L}_2([0,1]\times[\numenv], \lambda \otimes p_{\weights})}.
$$
Note that, because $var(X) \leq \mathbb{E}[X^2]$ for any real variable $X$, we have:
\begin{align*}
  \sum_{\env=1}^{\numenv} \weight_\env \left(G_{\subFct}(t, e) - G_{\subFct'}(t, e)\right)^2 & = \sum_{\env=1}^{\numenv} \weight_\env \left( \cdf^{-1}_{\subFct, \env}(t) - \cdf^{-1}_{\subFct', \env}(t) - \sum_{\env' = 1}^\numenv \weight_{\env'} \left(\cdf^{-1}_{\subFct, \env'}(t) - \cdf^{-1}_{\subFct', \env'}(t) \right) \right)^2 \\
  & \leq  \sum_{\env=1}^{\numenv} \weight_\env \left( \cdf^{-1}_{\subFct, \env}(t) - \cdf^{-1}_{\subFct', \env}(t) \right)^2.
\end{align*}
Therefore we get that:
\begin{align*}
    \left| \sqrt{\wassVar{\probVectBase(\subFct)}{\weights}} - \sqrt{\wassVar{\probVectBase(\subFct')}{\weights}} \right|^2 & \leq \left\|  G_{\subFct} -  G_{\subFct'} \right\|^2_{\text{L}_2([0,1]\times[\numenv], \lambda \otimes p_{\weights})} \leq \int_0^1 \sum_{\env=1}^{\numenv} \weight_\env \left( \cdf^{-1}_{\subFct, \env}(t) - \cdf^{-1}_{\subFct', \env}(t) \right)^2 dt \\
     \leq \sum_{\env = 1}^\numenv  & \weight_\env \wass^2(\probBase_\env(\subFct), \probBase_\env(\subFct')) \leq \sum_{\env = 1}^{\numenv} \weight_\env \mathbb{E}\left[ \left( \subFct(\pred{\env}) - \subFct'(\pred{\env}) \right)^2 \right] \leq \| \subFct - \subFct' \|^2_{\infty},
\end{align*}
where in the third inequality (which is an equality) we used the explicit form of the Wasserstein distance for measures defined on $\mathbb{R}$ (see remarks 2.30 in \cite{peyre2019computational} and Theorem 2.18 in \cite{villani2003topics}); and the fourth inequality is a direct consequence of the definition of the Wasserstein distance.
The last inequality concludes our proof.
\end{proof}

\begin{lemma} \label{lem:th2Lem2} 
Fix $\lambda \in (0,\numenv^{-1}]$ and let $\weightSpace_\lambda = \left\{ \weights = (\weight_\env)_{\env = 1}^\numenv : \forall \env \in [\numenv], \weight_\env \in [\lambda, 1]\text{ and }\sum_{\env = 1}^\numenv = 1 \right\}$.

For any class of functions $\subFctClass$, if for some $\weights^0 \in \weightSpace$ we have $\minWV_{\weights^0}(\subFctClass) > 0$ then there exists $\gamma > 0$ such that for any $\weights \in \weightSpace_\lambda$, we have $\minWV_{\weights}(\subFctClass) \geq \gamma$.

Conversely, if for some $\epsilon > 0$ there exists a $\weights \in \weightSpace_\lambda$ such that $\minWV_{\weights}(\subFctClass) \leq \epsilon$, then for any $\weights^0 \in \weightSpace$, we have $\minWV_{\weights^0}(\subFctClass) \leq \epsilon / \lambda$.
\end{lemma} 
\begin{proof}
For any $\weights^0 \in \weightSpace$ and $\weights \in \weightSpace_\lambda$ it is easy to see that:
$$
\forall \subFct \in \subFctClass, \quad \wassVar{\probVectBase(\subFct)}{\weights} \geq \lambda  \wassVar{\probVectBase(\subFct)}{\weights^0}.
$$
Taking the infimum over $\subFctClass$ on both sides directly yields the result of this lemma.
\end{proof}

\begin{lemma} \label{lem:th2Lem3} 
Under $\Tilde{\hyp}_{0}(\envSet)$ and Assumption \ref{fullDetailRegCond}, for any $\delta \in (0,1/2)$ the set $\GsetZero$ is a non-empty compact subset of $\Czero$.
\end{lemma} 
\begin{proof}
The fact that $\GsetZero$ is compact in $\Czero$ is direct from the continuity of $\subFct \mapsto \wassVar{\probVectBase(\subFct)}{\weights}$ proved in Lemma \ref{lem:th2Lem1} and that $\Gset$ is itself compact (see Remark \ref{rem:usefulPropSobolev}). 
Therefore, we just have to  show that it is non-empty. 

First, recall from Remark \ref{rem:usefulPropSobolev} that $\sobolev^{2,\sobolevDegree}(\compactset)$ is norm equivalent to the RKHS generated by the Matérn kernel $k_{\sobolevDegree - \numpred / 2, h}$ of degree $\sobolevDegree - \numpred / 2$ for any scale $h>0$.
Hence, there exists a ball $B'_{R'_\delta}$ of radius $R'_\delta$ centered at the origin in this RKHS such that: $\Gset \subseteq \overline{B'_{R'_\delta}}$. Therefore, by Lemma 26.10 from \cite{shalev2014understanding}:
$$
\radComp_{\nObs{\env}}(\Gset) \leq \radComp_{\nObs{\env}}(\overline{B'_{R'_\delta}}) = \radComp_{\nObs{\env}}(B'_{R'_\delta}) \leq \frac{R'_\delta}{\sqrt{\nObs{\env}}} \sup_{x \in \compactset} \sqrt{k_{\sobolevDegree - \numpred / 2, h}(x,x)}.
$$
Note that for $\subFctClass = \Gset$ the variables $(\maxVarEnv{\env})_{\env = 1}^\numenv$ from Theorem \ref{thm:unifBound} are sub-Gaussian since $\Gset$ is a bounded subset of $\Czero$ (see Section \ref{app:shortDiscussTh1}). 
Using Theorem \ref{thm:unifBound} and the fact that $\weight_\env = \nObs{\env} / \nObs{}$, we have that with probability at least $1 - \delta / 2$:
\begin{equation} \label{eq:th2lem3eq1}
    \forall \subFct \in \Gset, \quad \left| \wassVar{\empProbBaseVect(\subFct)}{\weights} - \wassVar{\probVectBase(\subFct)}{\weights} \right| \leq c_{\delta / 2} \frac{\log^2(\nObs{})}{\sqrt{\nObs{}}},
\end{equation}
for some constant $c_{\delta / 2}$ that depends only on $\delta$.

We are going to prove that $\GsetZero$ is non-empty by contradiction.
Assume $\GsetZero = \emptyset$.
Since $\Gset$ is compact and $\subFct \mapsto \wassVar{\probVectBase(\subFct)}{\weights}$ is continuous, it means there exists $\gamma > 0$ such that $\forall \subFct \in \Gset$ we have $\wassVar{\probVectBase(\subFct)}{\weights} \geq \gamma$ (by Lemma \ref{lem:th2Lem2} and condition (3) of Assumption \ref{fullDetailRegCond}, this $\gamma$ can be chosen independently of $\weights$). 
Also, since we are under $\Tilde{\hyp}_{0}(\envSet)$ and because the sets $\fctClass'_{\nObs{}}$ are non-decreasing and $\fctClass' = \overline{\bigcup_{\nObs{}} \fctClass'_{\nObs{}}}$, there exist a function $\fct \in \fctClass'$ and a constant $c_{\gamma} > 0$ such that:
$$
\forall \nObs{} \geq c_{\gamma}, \quad \fct \in \fctClass'_{\nObs{}} \,\, \text{ and } \,\,  \wassVar{\probVectBase(\fct)}{\weights} \leq \frac{\gamma}{2}, 
$$
(again by Lemma \ref{lem:th2Lem2}, $\fct$ can be chosen such that this inequality holds for all $\weights \in \weightSpace_\lambda$).

Using Theorem \ref{thm:unifBound} again, one can prove that with probability at least $1-\delta/2$:
\begin{equation} \label{eq:th2lem3eq2}
    \left| \wassVar{\empProbBaseVect(\fct)}{\weights} - \wassVar{\probVectBase(\fct)}{\weights}  \right| \leq c'_{\delta/2} \frac{\log^2(\nObs{})}{\sqrt{\nObs{}}},
\end{equation}
for some constant $c'_{\delta/2}$ that depends on $\delta$.

Combining \eqref{eq:th2lem3eq1} and \eqref{eq:th2lem3eq2} we get that there exists a constant $c = c(c_{\delta/2}, c_\gamma, c'_{\delta/2})$ such that for any $\nObs{} \geq c$, we have with probability at least $1-\delta$:
$$
\hat{\minWV}_{\weights}(\Gset) > \frac{3}{4} \gamma \quad \text{and} \quad \wassVar{\empProbBaseVect(\fct)}{\weights} < \frac{3}{4} \gamma,
$$
and therefore with probability at least $1 - \delta$ we must have $\minimizer \notin \Gset$, a contradiction with Assumption \ref{fullDetailRegCond} condition (4) since $\delta < 1/2$.
\end{proof}

\begin{lemma} \label{lem:th2Lem4} 
Under conditions (6a) and (6b) from Assumption \ref{fullDetailRegCond} we have that, for all $s \in (0,1/2)$ and $\delta \in (0,1)$, there exists a constant $C_{\delta,s} > 0$ such that:
$$
\forall \fct \in \Gset, \forall \env \in [\numenv], \forall t \in [s,1-s], \, \, \quantDensity_{\fct}^\env(t)\vee  |{\quantDensity_\fct^\env}'(t)| < C_{\delta,s}.
$$
\end{lemma} 
\begin{proof}
Take $\fct \in \Gset$ and $\env \in [\numenv]$, from (6a) we have:
\begin{equation} \label{eq:th2lem4eq1}
    \forall t \in [s,1-s], \quad \frac{s}{2}  |{\quantDensity_\fct^\env}'(t)| / \quantDensity_{\fct}^\env(t) < M_{\delta,s}  \quad \Leftrightarrow \quad |{\quantDensity_\fct^\env}'(t)| < \frac{2 M_{\delta, s}}{s}  \quantDensity_{\fct}^\env(t).
\end{equation}
Call $m_{s,\fct}^\env \doteq \max_{t \in [s,1-s]} \quantDensity_{\fct}^\env(t)$ achieved at some $t_{s,\fct}^\env \in [s,1-s]$ by continuity. 
This implies that for any $t \in [s,1-s]$ we have $|{\quantDensity_\fct^\env}'(t)| < 2 M_{\delta, s} m_{s,\fct}^\env$. 

Let $\epsilon > 0$ such that $2 \epsilon (M_{\delta,s}\vee 1) / (s \wedge (1/2 - s)) = 1/2 \,\, \Leftrightarrow \,\, \epsilon = (s \wedge (1/2 - s)) / 4 (M_{\delta, s}\vee 1)$. Then for $t \in [t_{s,\fct}^\env - \epsilon, t_{s,\fct}^\env + \epsilon] \cap [s,1-s]$ we have:
$$
| \quantDensity_{\fct}^\env(t) - \quantDensity_{\fct}^\env(t_{s,\fct}^\env) | = \left| \int_{t_{s,\fct}^\env}^t {\quantDensity_\fct^\env}'(u) du \right| \leq \epsilon  \frac{2 M_{\delta, s}}{s} m_{s, \fct}^\env  \leq \frac{m_{s, \fct}^\env}{2}.
$$
Thus this means that $\quantDensity_{\fct}^\env(t) \geq  \frac{m_{s, \fct}^\env}{2}$ for $t \in [t_{s,\fct}^\env - \epsilon, t_{s,\fct}^\env + \epsilon] \cap [s,1-s]$ and that we have from condition (6b) of Assumption \ref{fullDetailRegCond}:
$$
M_{\delta, s}' > \int_{s}^{1-s} t (1-t) {\quantDensity_{\fct}^\env(t)}^2 dt \geq \frac{s}{2} \epsilon \left(\frac{m_{s, \fct}^\env}{2}\right)^2 = \frac{s (s \wedge (1/2 - s))}{32 (M_{\delta, s} \vee 1)}{m_{s, \fct}^\env}^2.
$$
Therefore $m_{s, \fct}^\env \leq \sqrt{\frac{32 (M_{\delta, s} \vee 1) M_{\delta, s}'}{s (s \wedge (1/2 - s))}}$, a bound which does not depend on either $\fct$ or $\env$.
Combining with inequality \eqref{eq:th2lem4eq1}, we get our result.
\end{proof}

\begin{lemma} \label{lem:th2Lem5} Let $\delta \in (0,1)$.
For any $\env \in [\numenv]$ and functions $\subFct_1$ and $\subFct_2$ in $\Gset$, define for short $\quantDensity_{1}^\env$ and $\quantDensity_{2}^\env$ the respective quantile densities of $\probBase_\env(\subFct_1)$ and $\probBase_\env(\subFct_2)$.
Under Assumption \ref{fullDetailRegCond}, we have for $s \in (0,1/2)$:
$$
\int_{s}^{1-s} (\quantDensity_{1}^\env(t) - \quantDensity_{2}^\env(t))^2 dt \leq A_{\delta, s} \left( \| \subFct_1 - \subFct_2 \|_{\infty}^{2/3} \vee \| \subFct_1 - \subFct_2 \|_{\infty}^2 \right),
$$
where $A_{\delta, s} \doteq \frac{10 C_{\delta, s}^2}{\pi^2}(1 - 3s + 2 s^2) + \frac{108(1 + \pi^2)}{1 - 2s}$ and $C_{\delta, s}$ is the constant derived in Lemma \ref{lem:th2Lem4}.
\end{lemma} 
\begin{proof}
Set $s \in (0,1/2)$. Under Assumption \ref{fullDetailRegCond}, by Lemma \ref{lem:th2Lem4} we have that there exists a constant $C_{\delta, s}$ such that:
\begin{equation} \label{eq:th2lem5eqStar1}
    \forall \subFct \in \Gset, \forall \env \in [\numenv], \forall t \in [s,1-s], \, \, \quantDensity_{\subFct}^\env(t)\vee  |{\quantDensity_\subFct^\env}'(t)| < C_{\delta,s}.
\end{equation}
Let $\subFct_1, \subFct_2 \in \Gset$ and $\env \in [\numenv]$. Call $\cdf_{1,\env}^{-1}$ and $\cdf_{2,\env}^{-1}$ the quantile functions of $\probBase_{\env}(\subFct_1)$ and $\probBase_{\env}(\subFct_2)$ respectively.
We have by the definition of the infinite Wasserstein distance $W_{\infty}$ (see section 5.5.1 from \cite{santambrogio2015optimal} for its definition):
$$
\| \cdf_{1,\env}^{-1} - \cdf_{2,\env}^{-1} \|_\infty = W_\infty(\probBase_{\env}(\subFct_1),\probBase_{\env}(\subFct_2)) \leq \text{ess sup} | \target{\env} - \subFct_1(\pred{\env}) - \target{\env} + \subFct_2(\pred{\env})| \leq \epsilon \doteq \| \subFct_1 - \subFct_2 \|_{\infty},
$$
where "ess sup" refers to the essential supremum. Therefore, in particular for any $ t \in [s, 1-s]$, $| \cdf_{1,\env}^{-1}(t) - \cdf_{2,\env}^{-1}(t) | \leq \epsilon$. Now consider the sequence of Fourier bases in $\text{L}_2([s,1-s])$:
$$
\forall k \in \mathbb{Z}, \quad \phi_{k}(t) \doteq \frac{e^{i 2 \pi k (t - 1/2) / (1-2s)}}{\sqrt{1 - 2s}},
$$
and the Fourier coefficients of $\cdf_{j,\env}^{-1}$ for $j \in \{1, 2\}$ are:
$$
\forall k \in \mathbb{Z}, \quad  \alpha_{j,\env}^k \doteq \int_s^{1-s} \cdf_{j,\env}^{-1}(t) \overline{\phi_{k}(t)} dt.
$$
Next, notice that because of \eqref{eq:th2lem5eqStar1} we have that $\quantDensity_{1}^\env$ and $\quantDensity_{2}^\env$ are in $\text{L}_2([s,1-s])$; by integration by parts we can express the Fourier coefficients of $\quantDensity_{j}^\env$ for $j \in \{1, 2\}$ as follows:
\begin{align*}
    \forall k \in \mathbb{Z}, \quad \beta_{j,\env}^{k} & \doteq \int_{s}^{1-s} \quantDensity_{j}^\env(t) \overline{\phi_{k}(t)} dt \\
    & = \left[\cdf_{j,\env}^{-1}(t) \overline{\phi_{k}(t)} \right]_s^{1-s} + \frac{i 2 \pi k}{1-2s} \cdot \int_{s}^{1-s} \cdf_{j,\env}^{-1}(t) \overline{\phi_{k}(t)} dt \\
    & = \frac{(-1)^k}{\sqrt{1-2s}} \cdot \left(\cdf_{j,\env}^{-1}(1-s) - \cdf_{j,\env}^{-1}(s)\right) + \frac{i 2 \pi k}{1 - 2s} \cdot \alpha_{j,\env}^k.
\end{align*}
Therefore, by Parseval's formula we have for any $K \in \mathbb{N}$:
\begin{align} \label{eq:th2lem5eqStar2}
  \int_{s}^{1-s} & (\quantDensity_{1}^\env(t) - \quantDensity_{2}^\env(t))^2 dt  = \sum_{k \in \mathbb{Z}} | \beta_{1,\env}^k - \beta_{2,\env}^k |^2  = \sum_{|k| \geq K} | \beta_{1,\env}^k - \beta_{2,\env}^k |^2  \nonumber \\
      & +  3 \sum_{|k| < K} \left[ \frac{\left(\cdf_{1,\env}^{-1}(s) - \cdf_{2,\env}^{-1}(s)\right)^2}{1 - 2s} + \frac{\left(\cdf_{1,\env}^{-1}(1-s) - \cdf_{2,\env}^{-1}(1-s)\right)^2}{1-2s} + \left( \frac{2 \pi k}{1 - 2s} \right)^2 \left|\alpha_{1,\env}^k - \alpha_{2,\env}^k\right|^2 \right] \nonumber \\
    & \leq \sum_{|k| \geq K} | \beta_{1,\env}^k - \beta_{2,\env}^k |^2 + \frac{12 K \epsilon^2}{1-2s} + 3 \left(\frac{2 \pi K}{1-2s}\right)^2 \cdot \int_{s}^{1-s} \left(\cdf_{1,\env}^{-1}(t) - \cdf_{2,\env}^{-1}(t)\right)^2 dt \nonumber \\
    & \leq \sum_{|k| \geq K} | \beta_{1,\env}^k - \beta_{2,\env}^k |^2 + \frac{12(1 + \pi^2) K^2 \epsilon^2}{1-2s}.
\end{align}
What's left is to bound $\sum_{|k| \geq K} | \beta_{1,\env}^k - \beta_{2,\env}^k |^2$.
Now consider the derivative ${\quantDensity_{j}^\env}'$ of $\quantDensity_{j}^\env$ for $j \in \{1,2\}$, which is also in $\text{L}_2([s,1-s])$ because of \eqref{eq:th2lem5eqStar1}. 
It has the following Fourier coefficients:
\begin{align*}
    \forall k \in \mathbb{Z}, \quad \gamma_{j,\env}^{k} & \doteq \int_{s}^{1-s} {\quantDensity_{j}^\env}'(t) \overline{\phi_{k}(t)} dt \\
    & = \left[\quantDensity_{j}^\env(t) \overline{\phi_{k}(t)} \right]_s^{1-s} + \frac{i 2 \pi k}{1-2s} \cdot \int_{s}^{1-s} \quantDensity_{j}^\env(t) \overline{\phi_{k}(t)} dt \\
    & = \frac{(-1)^k}{\sqrt{1-2s}} \cdot \left(\quantDensity_{j}^\env(1-s) - \quantDensity_{j}^\env(s)\right) + \frac{i 2 \pi k}{1 - 2s} \cdot \beta_{j,\env}^k.
\end{align*}
Hence, using again bound \eqref{eq:th2lem5eqStar1} we obtain for any $k \in \mathbb{Z}$:
\begin{align*}
    | \beta_{1,\env}^k - \beta_{2,\env}^k |^2 & = \left( \frac{1 - 2s}{2 \pi k} \right)^2 \left| \gamma_{1,\env}^k - \gamma_{2,\env}^k +\frac{(-1)^{k+1}}{\sqrt{1-2s}} \left( \quantDensity_{1}^\env(1-s)- \quantDensity_{2}^\env(1-s) + \quantDensity_{1}^\env(s) - \quantDensity_{2}^\env(s) \right) \right|^2 \\
    & \leq 5 \left( \frac{1-2s}{2 \pi k} \right)^2 \left(|\gamma_{1,\env}^k - \gamma_{2,\env}^k|^2 + \frac{4 C_{\delta, s}^2}{1 - 2s} \right)\\
    & \leq 5 \left( \frac{1-2s}{2 \pi k} \right)^2 \left(2 \int_s^{1-s} |{\quantDensity_{1}^\env}'(t)|^2 dt + 2 \int_s^{1-s} |{\quantDensity_{2}^\env}'(t)|^2 dt + \frac{4 C_{\delta, s}^2}{1 - 2s} \right)\\
    & \leq 5 \left( \frac{1-2s}{2 \pi k} \right)^2 \left(4 C_{\delta, s}^2 + \frac{4 C_{\delta, s}^2}{1 - 2s} \right) \\
    & \leq 20 \left( \frac{1-2s}{2 \pi k} \right)^2 C_{\delta, s}^2 \frac{1-s}{1-2s} = \frac{5}{\pi^2 k^2}(1 - 3 s + 2 s^2) C_{\delta, s}^2.
\end{align*}
From inequality \eqref{eq:th2lem5eqStar2} we thus get:
$$
 \int_{s}^{1-s} (\quantDensity_{1}^\env(t) - \quantDensity_{2}^\env(t))^2 dt \leq \frac{10 C_{\delta, s}^2}{\pi^2}(1 - 3 s + 2 s^2) \cdot \sum_{k \geq K} \frac{1}{k^2} + \frac{12(1 + \pi^2) K^2 \epsilon^2}{1-2s}.
$$
This means that for any $x \geq 1$, we have:
\begin{align*}
    \int_{s}^{1-s} (\quantDensity_{1}^\env(t) - \quantDensity_{2}^\env(t))^2 dt & \leq \frac{10 C_{\delta, s}^2}{\pi^2}(1 - 3 s + 2 s^2) \int_{x}^\infty \frac{1}{u^2} du + \frac{12(1 + \pi^2)\epsilon^2}{1-2s} (x + 2)^2 \\
    & \leq \frac{10 C_{\delta, s}^2}{\pi^2}(1 - 3 s + 2 s^2) \frac{1}{x} + \frac{108(1 + \pi^2)\epsilon^2}{1-2s} x^2.
\end{align*}
If $\epsilon \leq 1$, we set $x = \epsilon^{-2/3}$ and we get:
$$
\int_{s}^{1-s} (\quantDensity_{1}^\env(t) - \quantDensity_{2}^\env(t))^2 dt \leq \left(\frac{10 C_{\delta, s}^2}{\pi^2}(1 - 3 s + 2 s^2)+ \frac{108(1 + \pi^2)}{1-2s} \right) \epsilon^{2/3}.
$$
Otherwise, if $\epsilon > 1$ we set $x = 1$ and we finally obtain:
$$
\int_{s}^{1-s} (\quantDensity_{1}^\env(t) - \quantDensity_{2}^\env(t))^2 dt \leq \left(\frac{10 C_{\delta, s}^2}{\pi^2}(1 - 3 s + 2 s^2)+ \frac{108(1 + \pi^2)}{1-2s} \right) \epsilon^2.
$$

\end{proof}

\begin{lemma} \label{lem:th2Lem6} 
Consider $X, Y $ two real-valued variables. Let $\alpha, \alpha_1, \alpha_2 \in (0,1)$ such that $\alpha = \alpha_1 + \alpha_2$. 
If we define $q_{1-\alpha}(Z)$ as the $(1-\alpha)$-quantile of any real-valued variable $Z$, we have:
$$
q_{1-\alpha}(X+Y) \leq q_{1-\alpha_1}(X) + q_{1-\alpha_2}(Y).
$$
\end{lemma} 
\begin{proof}
For any $a,b \in \mathbb{R}$, we have $\mathbb{P}(X+Y > a + b) \leq \mathbb{P}(X > a) + \mathbb{P}(Y > b)$. If we set $a = q_{1-\alpha_1}(X)$ and $b = q_{1-\alpha_2}(Y)$, we thus get that $\mathbb{P}(X+Y > a + b) \leq \alpha_1 + \alpha_2 = \alpha$. Therefore, $q_{1-\alpha}(X+Y) \leq a + b$. 
\end{proof}

\begin{lemma} \label{lem:th2Lem7} 
Let $s \in [0,1/2)$ and for any quantile density function $\quantDensity$ (potentially empirical) that satisfies condition (iii) from Assumption \ref{DelBarrioAssumptions} we define the following variable:
$$
T_{s}(\quantDensity) \doteq \sum_{\env = 1}^{\numenv-1} \int_{s}^{1-s} \brownBridge_\env^2(t) \quantDensity^2(t) dt,
$$
where $(\brownBridge_{\env}(t))_{\env = 1}^{\numenv-1}$ are $\numenv - 1$ independent Brownian bridges. 
Call also $q_{1-\alpha}(T_{s}(\quantDensity))$ the $(1-\alpha)$-quantile of $T_{s}(\quantDensity)$.
For any quantile densities $\quantDensity_1, \quantDensity_2$ that satisfy (iii) of Assumption \ref{DelBarrioAssumptions}, we have:
$$
\left( \mathbb{E} \left[ \, \left| \sqrt{T_s(\quantDensity_1)} - \sqrt{T_s(\quantDensity_2)} \, \right| \,\right] \right)^2 \leq (\numenv - 1) \cdot \int_{s}^{1-s} t (1-t) (\quantDensity_1(t) - \quantDensity_2(t))^2 dt.
$$
And for any $0 < \confLevel < \confLevel' < 1$, we have:
$$
\sqrt{q_{1-\confLevel}(T_s(\quantDensity_1))} \geq \sqrt{q_{1-\confLevel'}(T_s(\quantDensity_2))} - \frac{(\numenv - 1)^{1/2}}{\confLevel' - \confLevel} \left( \int_{s}^{1-s} t (1-t) (\quantDensity_1(t) - \quantDensity_2(t))^2 dt \right)^{1/2}.
$$
\end{lemma} 
\begin{proof}
Take $s \in [0,1/2)$ and $\quantDensity_1, \quantDensity_2$ two quantile densities that satisfy condition (iii) from Assumption \ref{DelBarrioAssumptions}. Notice that, for $i \in \{1,2\}$, we can rewrite $T_s(\quantDensity_i)$ as follows:
$$
T_s(\quantDensity_i) \doteq \sum_{\env = 1}^{\numenv-1} \int_{s}^{1-s} \brownBridge_\env^2(t) \quantDensity_i^2(t) dt = \| \brownBridge_\env(t) \quantDensity_i(t) \|^2_{\text{L}_2([E-1]\times[s, 1-s], \, \mu \otimes \lambda)},
$$
where $\brownBridge_\env(t) \quantDensity_i(t)$ is treated as a function of both $\env$ and $t$, and we denote the counting by $\mu$ and the Lebesgue measure by $\lambda$.
Therefore, the triangular inequality applies:
$$
\left| \sqrt{T_s(\quantDensity_1)} - \sqrt{T_s(\quantDensity_2)} \, \right| \leq \sqrt{T_s(\quantDensity_1 - \quantDensity_2)}.
$$
Using this fact and Jensen's inequality we get:
\begin{align*}
    \left(\mathbb{E}\left[\, \left| \sqrt{T_s(\quantDensity_1)} - \sqrt{T_s(\quantDensity_2)} \, \right| \, \right] \right)^2 & \leq \left(\mathbb{E}\left[ \sqrt{T_s(\quantDensity_1 - \quantDensity_2)} \, \right] \right)^2 \\
    & \leq \mathbb{E}\left[ \sum_{\env=1}^{\numenv-1} \int_{s}^{1-s} \brownBridge_\env^2(t) (\quantDensity_1(t) - \quantDensity_2(t))^2 dt \right] \\
    & \leq (E-1) \int_{s}^{1-s} t (1-t) (\quantDensity_1(t) - \quantDensity_2(t))^2 dt.
\end{align*}
Now let's turn to the lower bound for $q_{1-\confLevel}(T_s(\quantDensity_1))$.
Let $0 < \confLevel < \confLevel' < 1$, by Lemma \ref{lem:th2Lem6} we have that:
\begin{align*}
    \sqrt{q_{1-\confLevel'}(T_s(\quantDensity_2))} = q_{1-\confLevel'}\left(\sqrt{T_s(\quantDensity_2)}\right) & = q_{1-\confLevel'}\left( \sqrt{T_s(\quantDensity_2)} - \sqrt{T_s(\quantDensity_1)} + \sqrt{T_s(\quantDensity_1)} \, \right) \\
    & \leq q_{1-\confLevel}\left(\sqrt{T_s(\quantDensity_1)}\right) + q_{1-(\confLevel' - \confLevel)}\left( \sqrt{T_s(\quantDensity_2)} - \sqrt{T_s(\quantDensity_1)} \, \right) \\
    & \leq \sqrt{q_{1-\confLevel}\left(T_s(\quantDensity_1)\right)} + q_{1-(\confLevel' - \confLevel)}\left( \left| \sqrt{T_s(\quantDensity_2)} - \sqrt{T_s(\quantDensity_1)} \, \right| \right).
\end{align*}
Furthermore, it is easy to see that:
\begin{align*}
    q_{1-(\confLevel' - \confLevel)}\left( \left| \sqrt{T_s(\quantDensity_2)} - \sqrt{T_s(\quantDensity_1)} \, \right| \right) & \leq \frac{\mathbb{E}\left[\, \left| \sqrt{T_s(\quantDensity_1)} - \sqrt{T_s(\quantDensity_2)} \, \right| \, \right]}{\confLevel' - \confLevel}\\
    & \leq \frac{(\numenv - 1)^{1/2}}{\confLevel' - \confLevel}\left( \int_{s}^{1-s} t (1-t) (\quantDensity_1(t) - \quantDensity_2(t))^2 dt \right)^{1/2}.
\end{align*}
\end{proof}

\begin{lemma} \label{lem:th2Lem8} 
Assume Assumption \ref{fullDetailRegCond} and fix $\delta \in (0,1/2)$.
For any $s \in [0,1/2)$ and quantile density $\quantDensity$ that satisfies condition (iii) from Assumption \ref{DelBarrioAssumptions}, we let $T_s(\quantDensity)$ be the variable introduced in Lemma \ref{lem:th2Lem7}.
For any $\confLevel \in (0,1)$, define:
$$
t_{s,\alpha}^{\delta} \doteq \inf_{\fct \in \GsetZero} \left( q_{1-\confLevel}(T_s(\quantDensity_\fct)) \right),
$$
where $\quantDensity_\fct$ is the quantile density of $\probBase_\env(\fct)$ -- note that since $\fct \in \GsetZero$, the $\probBase_\env(\fct)$'s are actually identical. 
Then, for any $s \in [0,1/2)$ and $\confLevel \in (0,1)$, the infimum for $t_{s,\alpha}^{\delta}$ is attained by a function in $\GsetZero$, that is:
$$
\forall s \in [0,1/2), \forall \confLevel \in (0,1), \quad \exists \fct \in \GsetZero \quad \text{ s.t. } \quad q_{1-\confLevel}(T_s(\quantDensity_f)) = t_{s,\confLevel}^\delta.
$$
Finally, for any $\confLevel \in (0,1)$ we also have:
$$
\forall \epsilon > 0, \quad \exists (\confLevel', s) \in (\confLevel ,1)\times(0,1/2) \quad \text{ s.t. } \quad t_{s,\confLevel'}^\delta + \epsilon \geq t_{0,\alpha}^\delta.
$$
\end{lemma} 
\begin{proof}
Recall that by Lemma \ref{lem:th2Lem3}, $\GsetZero$ is a non-empty compact subset of $\Czero$.
Let $\confLevel \in (0,1)$ and first consider the case where $s \in (0,1/2)$ -- that is $s >0$.
Notice that the function:
$$
Q_{s,\confLevel}: \quad \fct \in \left( \GsetZero, \| \cdot \|_{\infty} \right) \mapsto q_{1-\confLevel}(T_s(\quantDensity_\fct))
$$
is in fact continuous when $s > 0$. 
Indeed, let $\fct \in \GsetZero$ and $(\fct_k)_{k \geq 1} \in \GsetZero$ such that $\| \fct - \fct_k \|_{\infty} \rightarrow_{k \rightarrow \infty} 0$, then by Lemma \ref{lem:th2Lem5} and Lemma \ref{lem:th2Lem7} we have:
\begin{align*}
    \left( \mathbb{E}\left[ \left| \sqrt{T_{s}(\quantDensity_\fct)} - \sqrt{T_{s}(\quantDensity_{\fct_k})} \right| \right] \right)^2 & \leq (\numenv - 1) \int_{s}^{1-s} \left( \quantDensity_{\fct}(t) - \quantDensity_{\fct_k}(t) \right)^2 dt \\
    & \leq (\numenv - 1) \cdot A_{\delta, s} \left(\| \fct - \fct_k \|_{\infty}^{2/3} \vee \| \fct - \fct_k \|_{\infty}^{2} \right) \xrightarrow[k \rightarrow \infty]{} 0.
\end{align*}
This proves at least that $T_{s}(\quantDensity_{\fct_k})$ converges in probability (thus in distribution too) toward $T_{s}(\quantDensity_{\fct})$ as $k$ goes to infinity.
Therefore, since $T_{s}(\quantDensity_{\fct})$ is a continuous variable:
$$
q_{1-\confLevel}( T_{s}(\quantDensity_{\fct_{k}})) \xrightarrow[k \rightarrow \infty]{} q_{1-\confLevel}(T_{s}(\quantDensity_{\fct})).
$$ 
Hence, the above map $Q_{s,\confLevel}$ is continuous when $s>0$. 
Since $\GsetZero$ is compact, it implies that $t_{s,\confLevel}^\delta$ is attained by some function $f_{s,\confLevel}^\delta$ in $\GsetZero$ for $s > 0$.

Now, let's turn to the case where $s = 0$.
Let $(s_k)_{k \geq 1} \in (0,1/2)$ be a decreasing sequence that converges to $0$, and for short let $\fct_k \doteq \fct_{s_k, \confLevel}^\delta$ the minimizer for $t_{s_k, \confLevel}^\delta$.
Since $\GsetZero$ is compact, up to extraction we can consider that $(\fct_k)_{k \geq 1}$ converges w.r.t.~$\| \cdot \|_{\infty}$ to a function $f_{0,\confLevel}^\delta$ in $\GsetZero$. 
We are going to show that $\fct_{0, \confLevel}^\delta$ is indeed a minimizer for $t_{0,\confLevel}^\delta$. 
For simplicity, call $\quantDensity_k \doteq \quantDensity_{\fct_k}$ and $\quantDensity_0 \doteq \quantDensity_{\fct_{0,\confLevel}^\delta}$ and note that $T_{s_k}(\quantDensity_{k}) = T_0(\quantDensity_k \cdot \mathbb{1}_{\cdot \in [s_k, 1-s_k]})$.

Fix $K \geq 1$, and by Lemma \ref{lem:th2Lem7} for any $k \geq K$ we have:
\begin{align*}
     & \left(\mathbb{E}\left[ \left| \sqrt{T_{s_k}(\quantDensity_{k})} - \sqrt{T_{0}(\quantDensity_{0})} \right|\right]\right)^2 \leq (E - 1) \int_{0}^{1} (1-t) t \left(\quantDensity_{k}(t) \mathbb{1}_{t \in [s_k, 1-s_k]} - \quantDensity_{0}(t)\right)^2 dt\\
     \leq &\, 2 (E - 1) \left[ \int_{0}^{s_K} t(1-t) \left(\quantDensity_{k}^2(t) + \quantDensity_{0}^2(t)\right) dt + \int_{1-s_K}^{1} t(1-t) \left(\quantDensity_{k}^2(t) + \quantDensity_{0}^2(t)\right) dt \right] \\
     & \quad\quad\quad\quad\quad + (E - 1) \int_{s_K}^{1-s_K}(\quantDensity_k(t) - \quantDensity_0(t))^2 dt.
\end{align*}
By condition (6c) of Assumption \ref{fullDetailRegCond} we have that for any $\epsilon > 0$ we can choose $K$ large enough such that, for any $k \geq K$:
$$
2 (E - 1) \left[ \int_{0}^{s_K} t(1-t) \left(\quantDensity_{k}^2(t) + \quantDensity_{0}^2(t)\right) dt + \int_{1-s_K}^{1} t(1-t) \left(\quantDensity_{k}^2(t) + \quantDensity_{0}^2(t)\right) dt \right] \leq \epsilon.
$$
Also notice that for a fixed $K$, by Lemma \ref{lem:th2Lem5} the term $\int_{s_K}^{1-s_K}(\quantDensity_k(t) - \quantDensity_0(t))^2 dt$ goes to $0$ as $k \rightarrow \infty$. 
It means that:
$$
\forall \epsilon > 0, \quad \limsup_{k \rightarrow \infty}\left(\mathbb{E}\left[ \left| \sqrt{T_{s_k}(\quantDensity_{k})} - \sqrt{T_{0}(\quantDensity_{0})} \right|\right]\right)^2 \leq \epsilon.
$$
Because $\epsilon$ is chosen arbitrarily, this implies that $T_{s_k}(\quantDensity_{k})$ converges in distribution toward $T_{0}(\quantDensity_{0})$. 
Since $\forall k, t_{s_k, \confLevel}^\delta \leq t_{0, \confLevel}^\delta$, we then have:
$$
t_{0,\confLevel}^\delta \leq q_{1-\confLevel}(T_{0}(\quantDensity_{0})) = \lim_{k \rightarrow \infty} q_{1-\confLevel}(T_{s_k}(\quantDensity_k)) = \lim_{k \rightarrow \infty} t_{s_k,\confLevel}^\delta \leq t_{0,\confLevel}^\delta.
$$
Hence, $f_{0,\confLevel}^\delta$ is a minimizer for $t_{0,\confLevel}^\delta$. Note we also proved that, for a fix $\confLevel \in (0,1)$:
$$
\forall \epsilon > 0, \quad \exists s \in (0,1/2) \quad \text{s.t.} \quad t_{0,\confLevel}^\delta \leq t_{s,\confLevel}^\delta + \frac{\epsilon}{2}.
$$
To prove the last point we just need to show that we can also find $\confLevel' > \confLevel$ such that $t_{s,\confLevel}^\delta \leq t_{s,\confLevel'}^\delta + \frac{\epsilon}{2}$.
We are going to use the same approach as above:
Let $(\confLevel_k)_{k\geq1}$ be a decreasing sequence in $(0,1)$ that converges to $\confLevel$ (meaning in particular that $\confLevel_k \geq \confLevel$ for all $k$); 
call $\fct_k$ the minimizer for $t_{s, \confLevel_k}^\delta$.
By compactness of $\GsetZero$, up to extraction we can consider that $\fct_k$ converges to some $\fct_0 \in \GsetZero$.
Let $K \geq 1$ and take any $k \geq K$, by continuity of $Q_{s,\confLevel_K}$ we have:
$$
t_{s,\confLevel}^\delta \geq t_{s, \confLevel_k}^\delta = Q_{s,\confLevel_k}(f_k) \geq Q_{s, \confLevel_K}(\fct_k) \xrightarrow[k \rightarrow \infty]{} Q_{s, \confLevel_K}(\fct_0) = q_{1-\confLevel_K}(T_s(\quantDensity_{\fct_0})) \xrightarrow[K\rightarrow\infty]{} q_{1-\confLevel}(T_s(\quantDensity_{\fct_0})) \geq t_{s,\confLevel}^\delta.
$$
This proves that $t_{s,\confLevel_k}^\delta \xrightarrow[k \rightarrow \infty]{} t_{s,\confLevel}^\delta$ and hence that for $k$ large enough $t_{s,\confLevel}^\delta \leq t_{s, \confLevel_k}^\delta + \frac{\epsilon}{2}$.
\end{proof}

\begin{lemma} \label{lem:th2Lem9} 
Assume Assumption \ref{fullDetailRegCond} and let $\delta \in (0,1/2)$.
For any $\epsilon > 0$, there exists a constant $c_{\delta, \epsilon}$ such that for any $\nObs{}\geq c_{\delta, \epsilon}$, we have with probability at least $1-2\delta$ that simultaneously:
$$
\minimizer \in \Gset \quad \text{ and } \quad \exists \fct \in \GsetZero \quad \text{s.t.} \quad \| \minimizer - \fct \|_{\infty} < \epsilon.
$$
\end{lemma} 
\begin{proof}
Set $\epsilon > 0$.
We know already that, under Assumption \ref{fullDetailRegCond}, when $\nObs{} \geq c_{\delta}$ for some constant $c_{\delta}$, with probability at least $1-\delta$ we have $\minimizer \in \Gset$. 
Also, based on the proof of Lemma \ref{lem:th2Lem3}, there is a constant $C_{\delta}$ such that with probability at least $1-\delta$:
\begin{equation} \label{eq:th2lem9Star1}
    \forall \subFct \in \Gset, \quad \left| \wassVar{\empProbBaseVect(\subFct)}{\weights} - \wassVar{\probVectBase(\subFct)}{\weights} \right| \leq C_{\delta} \frac{\log^2(\nObs{})}{\sqrt{\nObs{}}}.
\end{equation}
Recall that by Lemma \ref{lem:th2Lem3}, $\GsetZero$ is a non-empty compact subset of $\Gset$.
Consider $\Gset_{\epsilon} \doteq \{ \subFct \in \Gset: \,\, d(\subFct, \GsetZero) \geq \epsilon \}$.
$\Gset_{\epsilon}$ is a closed subset of $\Gset$, which is compact, hence $\Gset_{\epsilon}$ is also compact.
By the continuity of the Wasserstein variance w.r.t.~$\subFct$ implied by Lemma \ref{lem:th2Lem1} we can find a constant $\gamma > 0$ such that for any $\subFct \in \Gset_\epsilon$, we have $\wassVar{\probVectBase(\subFct)}{\weights} \geq \gamma$.
Note that Lemma \ref{lem:th2Lem2} tells us that $\gamma$ can be chosen independently of $\weights$ under Assumption \ref{fullDetailRegCond}.
Therefore, there is a constant $c_{\delta, \epsilon} > 0$ such that $\forall \nObs{} \geq c_{\delta, \epsilon}$, under the event of equation \eqref{eq:th2lem9Star1} we have:
$$
\forall \subFct \in \GsetZero, \quad \wassVar{\empProbBaseVect(\subFct)}{\weights}  < \gamma / 2 \quad \text{and} \quad \forall \subFct \in \Gset_\epsilon, \quad \wassVar{\empProbBaseVect(\subFct)}{\weights}  > \gamma / 2.
$$
Intersecting with the event that $\minimizer \in \Gset$, and using a union bound, we get that whenever $\nObs{} \geq \max(c_{\delta}, c_{\delta, \epsilon})$ with probability at least $1-2\delta$, $\minimizer$ is in $\Gset$ but cannot be in $\Gset_{\epsilon}$. This means that $d(\minimizer, \GsetZero) < \epsilon$.
\end{proof}

\begin{lemma} \label{lem:th2Lem10} 
Let $\fct, \subFct \in \Czero$ and $\env \in [\numenv]$. 
Consider $\quantEst_{\fct}^\env$ and $\quantEst_{\subFct}^\env$ the kernel quantile density estimators at respectively $\fct$ and $\subFct$ as defined from Definition \ref{def:quantEstDef}, where the kernel $\kernel$ is $L$-Lipschitz.
We also set the bandwidth at $h_\env = \beta \nObs{\env}^{-1/3}$ for some constant $\beta>0$.
For any $s \in (0,1/2)$, as long as $\nObs{\env} > \frac{\beta + 1}{s} \vee \beta^{-1/2}$ we have:
$$
\forall t \in [s,1-s], \quad \left| \quantEst_{\fct}^\env(t)  - \quantEst_{\subFct}^\env(t) \right| \leq \frac{4 L}{h_\env} \| \fct - \subFct \|_{\infty}.
$$
\end{lemma} 
\begin{proof}
First, since $\kernel$ is supported on $[-1,1]$, notice that as long as $\nObs{\env}^{1/3} > (\beta + 1)/s$, we have $\kernel_{h_\env}(t - (\nObs{\env}-1) / \nObs{\env}) = \kernel_{h_\env}(t - 1/\nObs{\env}) = 0$ for $t \in [s,1-s]$.
Therefore, we can rewrite $\quantEst_{\fct}^\env(t)$ (and $\quantEst_{\subFct}^\env(t)$ similarly) as follows:
$$
\forall t \in [s, 1-s], \quad \quantEst_{\fct}^\env(t) = \sum_{i = 2}^{\nObs{\env} - 1} \residual_{(i)}^\env(\fct) \cdot\left( \kernel_{h_\env}\left(t - \frac{i-1}{\nObs{\env}} \right) - \kernel_{h_\env}\left(t - \frac{i}{\nObs{\env}}\right) \right).
$$
Because $\kernel$ is supported on $[-1,1]$ and is $L$-Lipschitz, we get that:
\begin{align*}
    \forall t \in [s, 1-s], \quad \left| \quantEst_{\fct}^\env(t)  - \quantEst_{\subFct}^\env(t) \right| & \leq \sum_{i = 2}^{\nObs{\env}-1} \left| \residual_{(i)}^\env(\fct) - \residual_{(i)}^\env(\subFct) \right| \cdot \left| \kernel_{h_\env}\left( t - \frac{i-1}{\nObs{\env}} \right) - \kernel_{h_\env}\left( t - \frac{i}{\nObs{\env}} \right) \right| \\
    & \leq  \sum_{i = \lceil (t - h_\env) \nObs{\env} \rceil}^{\lfloor (t + h_\env) \nObs{\env} \rfloor + 1} \left| \residual_{(i)}^\env(\fct) - \residual_{(i)}^\env(\subFct) \right| \cdot \left| \kernel_{h_\env}\left( t - \frac{i-1}{\nObs{\env}} \right) - \kernel_{h_\env}\left( t - \frac{i}{\nObs{\env}} \right) \right| \\
    & \leq \frac{2 \nObs{\env} h_\env + 2}{\nObs{\env} h_\env^2} L \max_{i \in [\nObs{\env}]}\left| \residual_{(i)}^\env(\fct) - \residual_{(i)}^\env(\subFct) \right| \\
    & \leq \frac{4 L}{h_\env} \max_{i \in [\nObs{\env}]} \left| \residual_{(i)}^\env(\fct) - \residual_{(i)}^\env(\subFct) \right|,
\end{align*}
where we used in the last inequality the fact that $\nObs{\env}^{1/3} > \beta^{-1/2} \Leftrightarrow \nObs{\env} h_\env > 1$.
Call $\residual_i^\env(\fct) \doteq \targetObs_i^\env - \fct(\predObs_i^\env)$ and $\residual_i^\env(\subFct) \doteq \targetObs_i^\env - \subFct(\predObs_i^\env)$ the unordered residuals.
Finally, we have:
$$
\max_{i \in [\nObs{\env}]} \left| \residual_{(i)}^\env(\fct) - \residual_{(i)}^\env(\subFct) \right| = W_\infty(\empProbBase(\fct), \empProbBase(\subFct)) \leq \max_{i \in [\nObs{\env}]} \left| \residual_{i}^\env(\fct) - \residual_{i}^\env(\subFct) \right| \leq \| \fct - \subFct \|_{\infty},
$$
where the definition of $W_\infty$ can be found in Section 5.5.1 of \cite{santambrogio2015optimal}. This concludes the proof.
\end{proof}

\begin{lemma} \label{lem:th2Lem11} 
Assume that for any $\fctClass' \in \{ \fctClass_{-k}, \,\, k \in [\numpred]\}$, Assumption \ref{fullDetailRegCond} holds with $\fctClass'_{n} = \fctClass', \forall n$ (as it is summarized in Assumption \ref{regularityAssumptions}).
Then under Assumption \ref{invar-assump}, if $\fct^*$ is the unique function in $\overline{\fctClass}$ such that $\wassVar{\empProbBaseVect(\fct^*)}{\weights} = 0$ then $\WVMidentifPred = \causalPred$.
\end{lemma}
\begin{proof}
The fact that $\minWV_{\weights}(\fctClass_{-k}) = 0$ for $k \notin \causalPred$ is obvious from Assumption \ref{invar-assump}, so we just need to prove that $\minWV_{\weights}(\fctClass_{-k}) > 0$ for $k \in \causalPred$.
Let's do it by contradiction, and assume there is a $k \in \causalPred$ such that $\minWV_{\weights}(\fctClass_{-k}) = 0$.
Let $\delta \in (0,1/2)$, and $\Gset_{k}$ the corresponding set from Assumption \ref{fullDetailRegCond} for $\fctClass' = \fctClass_{-k}$ and $\minimizer^k$ the minimizer of $\hat{\minWV}_{\weights}(\fctClass_{-k})$.

Notice that because $\fct^*$ is the unique function in $\overline{\fctClass}$ such that $\wassVar{\probVectBase(\fct^*)}{\weights} = 0$, it implies that $\wassVar{\probVectBase(\subFct)}{\weights} > 0$ for any $\subFct \in \Gset_{k}$. 
Then, based on the proof of Lemma \ref{lem:th2Lem3} we can conclude that there is a constant $\gamma > 0$ (independent of $\weights$) such that for any $g \in \Gset_k$, $\wassVar{\probVectBase(\subFct)}{\weights} \geq \gamma$ and, since $\minWV_{\weights}(\fctClass_{-k}) = 0$, there exits a function $\fct \in \fctClass_{-k}$ such that $\wassVar{\probVectBase(\fct)}{\weights} \leq \gamma / 2$.
Furthermore, using equations \eqref{eq:th2lem3eq1} and \eqref{eq:th2lem3eq2} from the proof of Lemma \ref{lem:th2Lem3}, we get that for $n$ large enough, with probability at least $1-\delta$, $\minimizer^k$ must be outside of $\Gset_k$. A contradiction with condition (4) of Assumption \ref{fullDetailRegCond}.
\end{proof}

\section{ADDITIONAL IMPLEMENTATION DETAILS} \label{app:addImpDetails}

\subsection{Optimization Details}
For any function \(f_{\theta}\) parametrized by \(\theta \in \Theta\) and any \(e \in [E]\),
denote \((\epsilon^e_{(i)}(f_{\theta}))_{i=1}^{n_e}\) as the residuals obtained with \( f_\theta \) in environment \(e\) 
sorted in increasing order as in Definition~\ref{def:quantEstDef}. Moreover, define 
\(\Pi \doteq \{i / n_e : \, e \in [E], \,\, i \in [n_e] \}\).
We can also rewrite this set as \(\Pi = \{\pi_1, \dots, \pi_L\}\) where $L \doteq |\Pi|$ and 
the \(\pi_{\ell}\)'s are the elements of $\Pi$ sorted in increasing order; 
we also set $\pi_0 \doteq 0$. Furthermore note that for \(e \in [E]\), 
the quantile function of $\empProbBase_e(f_\theta)$ can be written as:
\[
    F_e^{-1}(t) = \epsilon_{(1)}^e(f_\theta) \mathbb{1}\left\{t \in \left[0, \, \frac{i}{n_e}\right]\right\} 
              + \sum_{i=2}^{n_e} \epsilon^e_{(i)}(f_\theta) 
              \mathbb{1}\left\{t \in \left( \frac{i-1}{n_e},\, \frac{i}{n_e}\right]\right\}.
\]
Therefore, for the empirical distributions \(\empProbBaseVect(f_\theta)\) the closed form of 
the Wasserstein variance from Equation~\eqref{eq:explicitWV} becomes:
\begin{equation} \label{eq:empWV}
  \wassVar{\empProbBaseVect(f_{\theta})}{\weights}
= \sum_{\ell=1}^L \left(\sum_{e=1}^E w_e \left(\epsilon^e_{(\ceil{\pi_{\ell} n_e})}(f_{\theta}) 
- \sum_{e'=1}^E w_{e'}\epsilon^{e'}_{(\ceil{\pi_{\ell} n_{e'}})}(f_{\theta})\right)^2\right)
\cdot \left(\pi_{\ell} - \pi_{\ell - 1} \right). 
\end{equation}

Let's call $j = j(\theta; e, \ell)$ the index of the observation in environment $e$ that corresponds to the $\ceil{\pi_{\ell} n_e}$th ordered residual $\epsilon^e_{(\ceil{\pi_{\ell} n_e})}(f_{\theta}) $, that is $\epsilon^e_{(\ceil{\pi_{\ell} n_e})}(f_{\theta})  =  y^e_j - f_\theta(x^e_j)$. 
If in a neighborhood of $\theta$ the order of the residuals doesn't change, i.e.~$j$ is locally independent of $\theta$ in that neighborhood, then $\epsilon^e_{(\ceil{\pi_{\ell} n_e})}(f_{\theta})$ is differentiable w.r.t.~$\theta$ at that point, and we have $\nabla_{\theta} \epsilon^e_{(\ceil{\pi_{\ell} n_e})}(f_{\theta}) = - \nabla_{\theta} f_\theta(x^e_j)$.
Hence, whenever the order of the residuals in each environment remain unchanged in a neighborhood of $\theta$, then the gradient of \eqref{eq:empWV} at this $\theta$ is:
\begin{equation*}
  \nabla_{\theta}
  \wassVar{\empProbBaseVect(f_{\theta})}{\weights}
    = 
     \sum_{e = 1}^E \sum_{\ell = 1}^L 
    2 (\pi_{\ell} - \pi_{\ell - 1}) w_e
    \nabla_{\theta} f_\theta\left(x^e_{j(\theta; e, \ell)}\right)
    \left(\sum_{e'=1}^E w_{e'} \epsilon^{e'}_{(\ceil{\pi_{\ell} n_{e'}})}(f_\theta) - 
    \epsilon^e_{(\ceil{\pi_{\ell} n_e})}(f_\theta) \right).
\end{equation*}
At points where $\wassVar{\empProbBaseVect(f_{\theta})}{\weights}$ is not differentiable, 
the above expression is still a supergradient.
Note also that this gradient can be computed efficiently; it requires only sorting and matrix multiplications.
Finally, the optimization we use in our experiments is L-BFGS with fixed memory $m = 50$.

\begin{remark}
The problem of minimizing the Wasserstein variance to compute the statistics $\hat{\minWV}_{\weights }(\fctClass_{-\predIdx})$ is in general non-convex. 
However, we haven't found in our experiments any examples where the optimization reached a ``bad" local minimum.
We believe that finding the minimal Wasserstein variance values like $\hat{\minWV}_{\weights }(\fctClass_{-\predIdx})$ shouldn't be hard in general, at least for classes of linear functions; but of course a deeper study of the landscape of the Wasserstein variance needs to be done to be able to answer this question formally. 
Furthermore, the optimization procedure behind each of the \(p\) tests WVM needs to perform also scales with the number of predictors.
For example, L-BFGS scales 
linearly in the number of predictors yielding an overall complexity for WVM to be \(\mathcal{O}(p^2)\) -- see Figure \ref{fig:wvm-run-time}.
As we saw in Section \ref{sec:experiments}, compared to the exponential scaling of ICP the quadratic scaling of WVM is modest. 
\end{remark}

\begin{figure}[!t]
    \centering
    \includegraphics[scale=0.35]{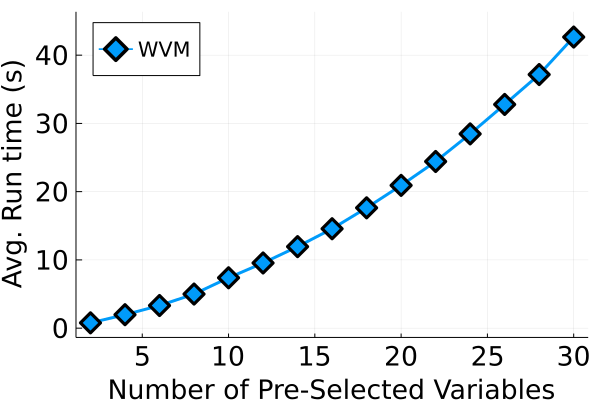}
    \caption{Average run time of WVM in seconds over the 100 simulations from Section \ref{sec:experiments} for different number of pre-selected variables. This is the same data from Figure \ref{fig:icp-wvm-time}, but displayed without the log-scale axis.}
    \vskip 0.1in
    \label{fig:wvm-run-time}
\end{figure}

\subsection{Approximation of the Asymptotic Distribution}
\begin{figure}[!h]
    \centering
    \includegraphics[scale=0.5]{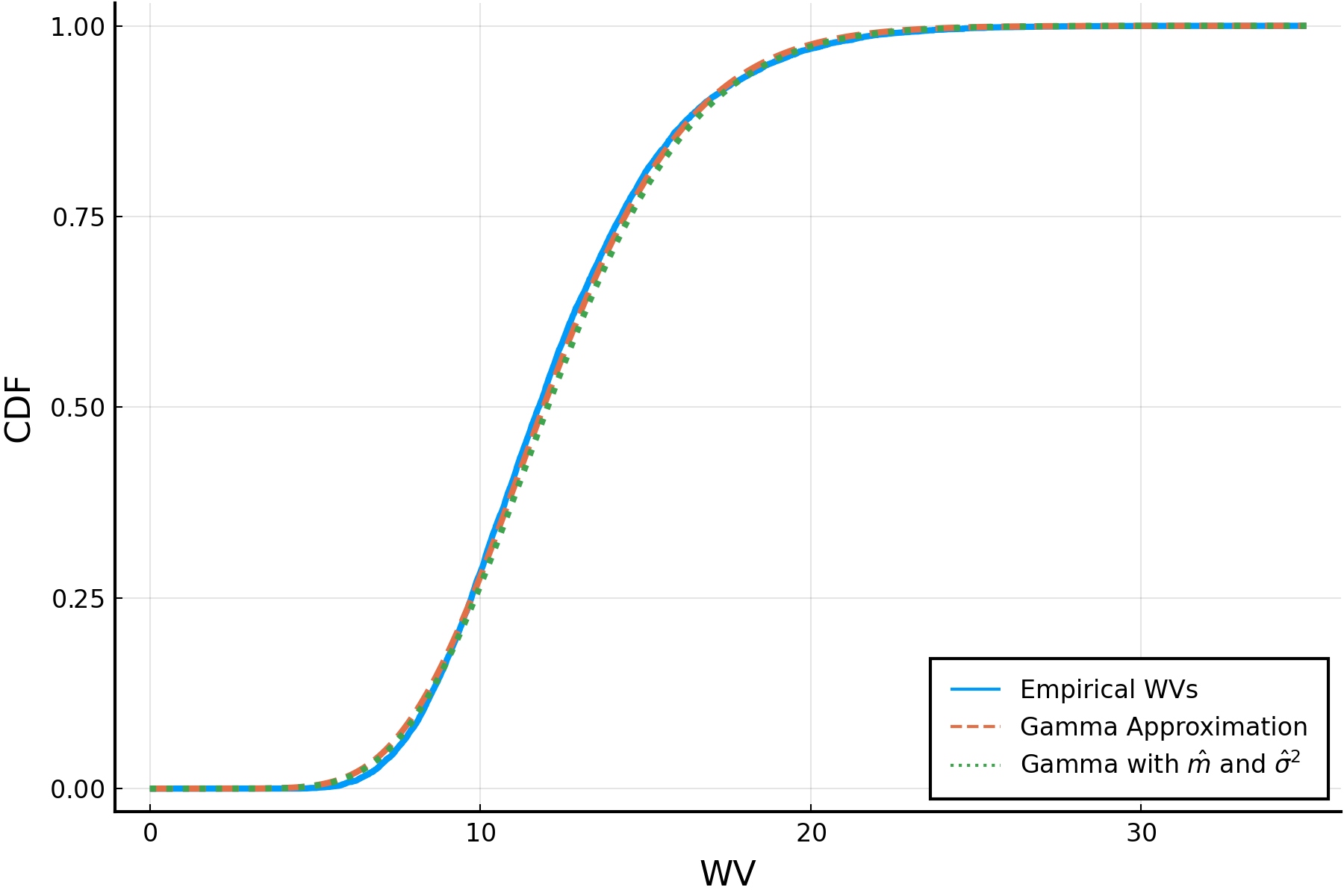}
    \caption{Comparison of CDFs of the Wasserstein Variance and its approximations with a Gamma distribution.}
    \label{fig:approx-dist}
\end{figure}

To show that the Gamma distribution can well approximate the asymptotic distribution of the Wasserstein Variance under $\tilde{H}_{0}(\envSet)$ from \eqref{eq:WVExactAsymptLimit} and \eqref{eq:limitVariable}, we plot in Figure~\ref{fig:approx-dist} 
the empirical CDF for $\wassVar{\empProbBaseVect}{\weights}$ compared to a Gamma distribution with the 
same mean and variance as well as the Gamma approximation introduced in Proposition \ref{prop:gammaApprox}.
More precisely, we sampled in each environment (\(E = 5\)) $500$ i.i.d.~$\mathcal{N}(0,1)$ observations, so that the distributions across environments are identical; the resulting empirical distributions are denoted by $\empProbBaseVect$.
We repeat this process $10,000$ times and compute $\wassVar{\empProbBaseVect}{\weights}$ at each iteration (we set the weights at $1/E$) to generate the empirical CDF of $\wassVar{\empProbBaseVect}{\weights}$ under that setting ("Empirical WVs" in Figure~\ref{fig:approx-dist}).
Moreover, we compute the empirical mean and variance of $\wassVar{\empProbBaseVect}{\weights}$ over these $10,000$ simulations
and consider a Gamma distribution with the same mean and variance ("Gamma with Empirical Params" in 
Figure \ref{fig:approx-dist}). 
Lastly, we use the approximation with the kernel quantile density estimators proposed in Proposition~\ref{prop:gammaApprox}
("Quantile KDE" in Figure~\ref{fig:approx-dist}),
increasing the sample size to $5,000$ samples for each environment to make sure the kernel estimator converged.
In addition to Figure~\ref{fig:approx-dist}, we refer to Figure 1 of~\cite{gretton2007kernel} to show a further comparison between the CDFs of a generalized \(\chi^2\)-distribution and its approximation with a Gamma distribution with the 
same mean and variance.

Finally, we present below a proof of Proposition~\ref{prop:gammaApprox} that characterizes 
the mean and variance of the random variable we use to construct our test in terms
of integrals of the covariance $\covFct{t}{s}$ of the Brownian bridge.
\begin{proof}[Proof of Proposition~\eqref{prop:gammaApprox}]
For any quantile density $\quantDensity$ that satisfies condition (iii) from Assumption \ref{DelBarrioAssumptions}, we can define the variable:
\[
T_0(\quantDensity) \doteq \sum_{\env = 1}^{\numenv - 1} \int_{0}^1 \brownBridge_\env^2(t) \quantDensity^2(t) dt
\]
where $(\brownBridge_\env(t))_{\env = 1}^{\numenv - 1}$ are $(\numenv - 1)$ independent Brownian bridges.
Note that the variable in Equation \eqref{eq:limitVariable} is simply $n^{-1} T_0(\quantEst)$. 
So we just need to compute the expectation and variance of \(T_0(\quantDensity)\) in order to prove 
Proposition~\ref{prop:gammaApprox}. First the expectation; by switching the integral sign with the expectation, we get:
\begin{equation*}
  \mathbb{E}[T_0(\quantDensity)] = \sum_{\env = 1}^{\numenv - 1} \int_{0}^1 \mathbb{E}\left[\brownBridge_\env^2(t) \quantDensity^2(t) \right]dt = (\numenv - 1) \int_{0}^1 t(1-t) \quantDensity^2(t)dt = (\numenv - 1) \int_{0}^1 \covFct{t}{t} \quantDensity^2(t)dt.
\end{equation*}
Then, by independence of the Brownian bridges the variance is:
\begin{equation*}
  \text{Var}[T_0(\quantDensity)] = 
  \sum_{\env = 1}^{\numenv - 1} 
  \text{Var}\left[\int_{0}^1 \brownBridge_\env^2(t) \quantDensity^2(t)dt\right] 
  = (\numenv - 1)\left( \mathbb{E}\left[\left(\int_{0}^1 \brownBridge_1^2(t) \quantDensity^2(t)dt\right)^2\right] 
  - \left(\int_0^1 \eta(t,t) q^2(t)dt\right)^2\right).
\end{equation*}

Furthermore notice that:
\begin{align*}
      \mathbb{E}\left[\left(\int_{0}^1 \brownBridge_1^2(t) \quantDensity^2(t)dt\right)^2\right] &= 
      \mathbb{E}\left[\left(\int_{0}^1 \brownBridge_1^2(t) \quantDensity^2(t)dt\right) \cdot 
                                      \left(\int_{0}^1 \brownBridge_1^2(s) \quantDensity^2(s)ds\right)\right]\\
  &=  \mathbb{E} \left[\int_0^1 \int_0^1 \brownBridge_1^2(t) \brownBridge_1^2(s)  
    \quantDensity^2(t)\quantDensity^2(s)dtds\right] \\
  &= \int_0^1 \int_0^1 \mathbb{E} [\brownBridge_1^2(t) \brownBridge_1^2(s)] 
 \quantDensity^2(t)\quantDensity^2(s)dtds\\
    &= \int_0^1 \int_0^1 \left(\text{Var}[\brownBridge_1(t)] \text{Var}[\brownBridge_1(s)] + 2 \text{cov}(\brownBridge_1(t), \brownBridge_1(s))^2\right)\quantDensity^2(t)\quantDensity^2(s)dtds\\
    &= \int_0^1\int_0^1\left(\eta(t,t)\eta(s,s)+2\eta^2(s,t)\right)\quantDensity^2(t)\quantDensity^2(s)dtds\\
    & = 2 \int_0^1\int_0^1\eta^2(s,t)\quantDensity^2(t)\quantDensity^2(s)dtds
    + \left(\int_0^1\eta(t,t)\quantDensity^2(t)dt\right)^2 
\end{align*}
Therefore:
\[
 \text{Var}[T_0(\quantDensity)] = 2 \int_0^1\int_0^1\eta^2(s,t)\quantDensity^2(t)\quantDensity^2(s)dtds
\]

\end{proof}
\newpage
\subsection{Bootstrap Approximation}
Our bootstrap heuristic proceeds as follows: 
\begin{enumerate}
    \item Generate bootstrap samples in each environment by drawing  with replacement $\nObs{\env}$ samples from the existing observations.
    \item Based on the resulting bootstrap samples, compute $\hat{\minWV}(\fctClass)$.
    \item Repeat the above process $B = 50$ times, and compute the average $\hat{m}$ and variance $\hat{\sigma}^2$ for $\hat{\minWV}(\fctClass)$ over these $B$ simulations.
    \item Use a Gamma distribution with mean $\hat{m}$ and variance $\hat{\sigma}^2$ for setting the thresholds. 
\end{enumerate} 
Note that we saw in practice that increasing the number of bootstrap iterations $B$ over $50$ did not increase performance in our simulations.



\section{DETAILS ON THE SIMULATIONS AND ADDITIONAL EXPERIMENTS}
\label{app:fullDetailSim}
In this section we provide more details on the experiments carried out in 
Section~\ref{sec:experiments}, as well as other additional experiments.

\subsection{Additional Details on the Simulations from Section~\ref{sec:experiments}} \label{app:detailsOnLinSim}
For each of the simulated graphs, we consider $p+1$ variables that we randomly permute to determine their causal order; 
the 21st variable in this permutation is declared as the target variable, so that it has 20 non-descendants and 30 non-ancestors.
Then for each pair of variables with 
probability \(k / p\) where \(k = 12\) is the average degree, we connect them with an arrow (the direction of which is determined by their causal orders). 
For the target, we drop any of its incoming arrows generated by the above procedure, and instead we randomly select a subset $S^*$ (where $|S^*| = 6$) among its 20 non-descendants as its parents and we connect the target to them with incoming arrows.

Once the structure of the graph is drawn, we generate the linear Gaussian structural equations for each of the nodes in the graph as follows:
The Gaussian noises each have mean zero and their variances 
are sampled uniformly and independently for each graph from \([0.3^2 , 1]\).
We uniformly sample the linear coefficients for the parents of the node in absolute value independently for each node in the graph from
[0.2, 1], and we switch their signs with probability $1/2$.
We normalize the coefficients so that the linear function of the parents has variance 1; we do so to avoid having extreme variances for variables that appear at the bottom of the graph.

At each intervened node, with probability 2/3, we scale the noise by a random scaling factor uniformly distributed on \([lb, ub]\), where \(lb\) and \(ub\) are 
chosen uniformly at random from [0.5, 5] with \(lb < ub\). However, with 
probability 1/3, the scaling factor is chosen to be a constant, equal to the mid-point between $lb$ and $ub$. Furthermore, the mechanistic
intervention only changes the coefficients with probability 1/3 by adding a standard normal
noise to them, otherwise the coefficients remain unchanged. 
Lastly, we added two additional constraints: we let
all variables (except the target) be intervened on at least once and each environment 
with its consecutive environment 
share 
40\% of the variables. More precisely, in each environment we intervene on 65\% of the predictors, chosen at random (note however that the scale of each of these interventions based on the choice of $lb$ and $ub$ above can often be negligible or not statistically significant w.r.t.~the sample size).

Finally, we use in our experiments the R package \emph{pcalg}~\citep{pcalg} for GIES and LiNGAM, and 
for ICP we use the R package \emph{InvariantCausalPrediction}~\citep{icp}. 
All experiments were run on a laptop with a quadcore 2.7 GHz Intel Core i7 processor. 

\subsection{Additional Linear Experiments}

\paragraph{Mixed Noise Distributions.}

In addition to the main simulations performed in Section~\ref{sec:experiments}, we investigate 
how using a variety of distributions would affect the error ratio and the false positive rate
for the various methods, in particular how the introduction of heavy-tailed distributions would
affect the results. 
Concretely, we sample uniformly for each variable a noise distribution from either: a standard
normal, a standardized Student's T-distribution (mean zero, unit variance) with degrees of 
freedom equal to 3, 5, 10, 20, or 50, and a mean zero, unit variance uniform distribution.
All other parameters are the same as in the main simulations from Section~\ref{sec:experiments}. 
As is seen in Figures~\ref{fig:error-ratio-mixed} and~\ref{fig:fprs-mixed}, the results are not 
dramatically different from the results using only a standard normal distribution as the noise 
distribution for each variable. 
The only noticeable differences are seen with GIES having a worse FPR (and
thus worse Error Ratio) and WVM having a slightly worsened error ratio now comparable to ICP. 
A possible reason why WVM might behave slightly worse under that setting could be the fact that the asymptotic result from \eqref{eq:WVExactAsymptLimit} doesn't hold for heavy-tailed distributions (see Remark \ref{rem:onDelBarrioAssumption} in Section \ref{app:listSuffCondAsymptLimit} for further details, and a simple solution to this issue).

\begin{figure}[!t]
\centering
  \subfloat[\label{fig:error-ratio-mixed}]{%
      \includegraphics[width=0.45\textwidth]{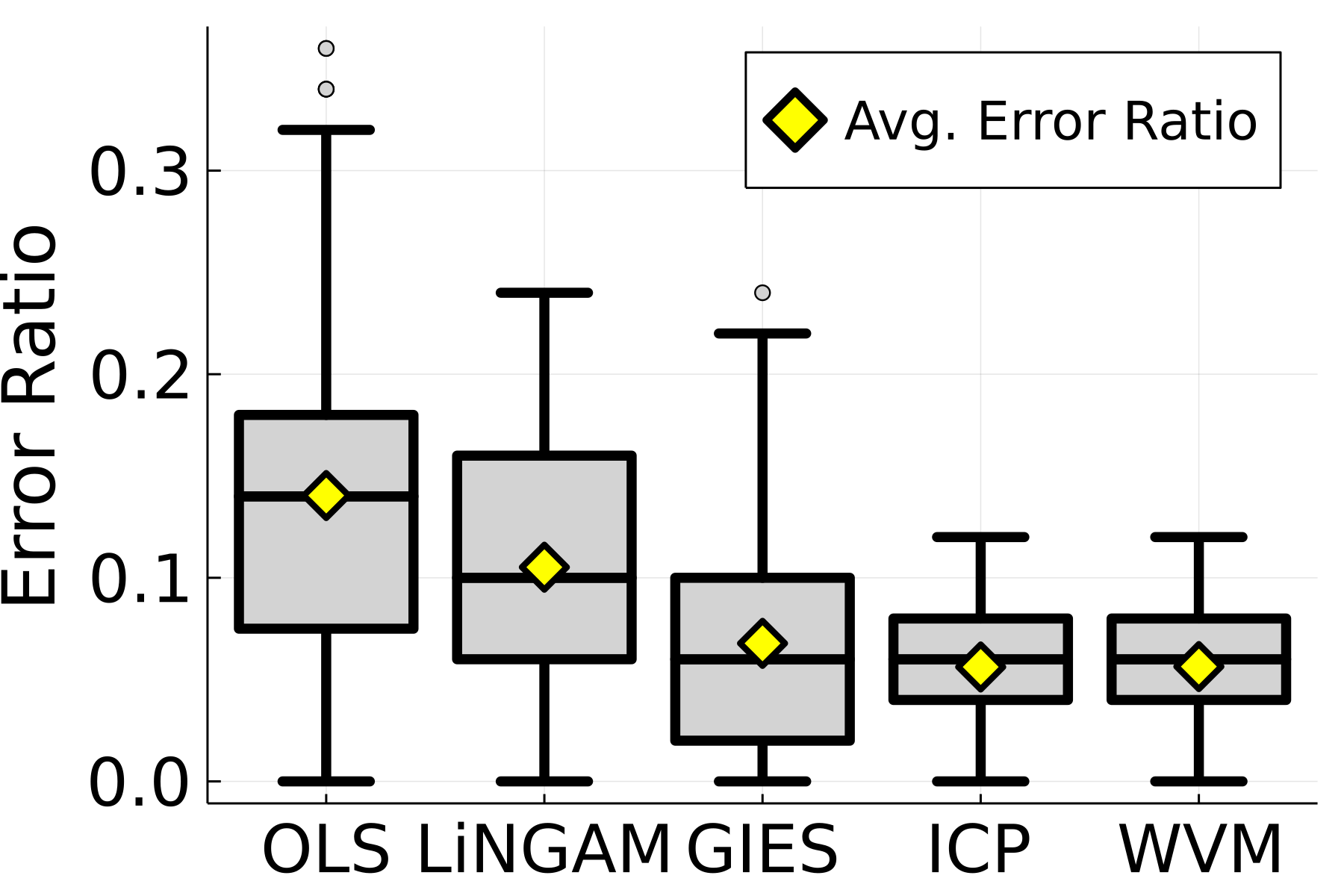}}
  \subfloat[\label{fig:fprs-mixed} ]{%
      \includegraphics[width=0.45\textwidth]{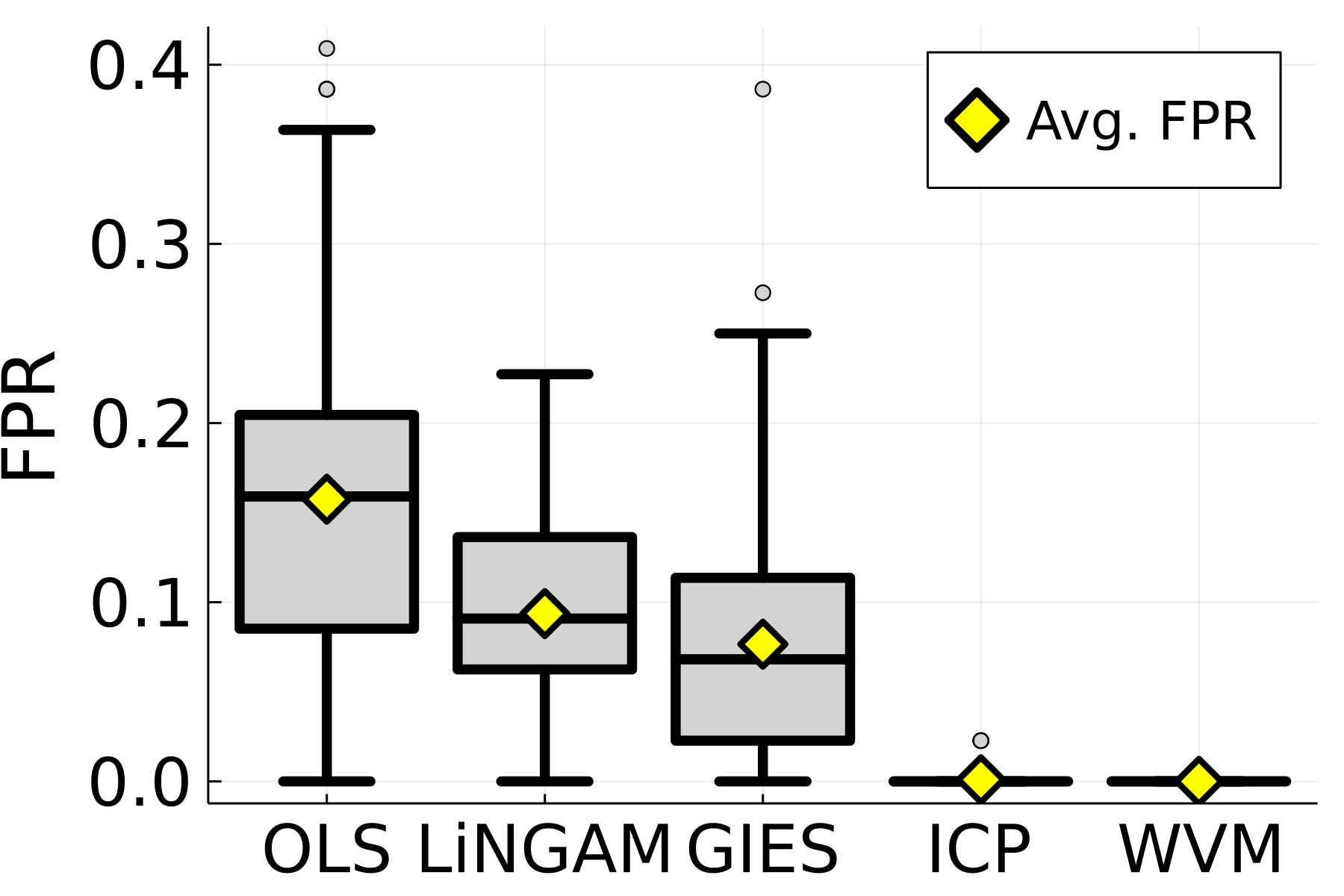}}
\caption{Error ratio (a) and False Positive Rate (b) of various causal discovery methods with noise distributions sampled from either a Normal distribution, some Student's T-Distributions or a uniform distribution.}
\end{figure}

\begin{figure}[!h]
\centering
  \subfloat[\label{fig:fns-ratio-diff-alphas}]{%
      \includegraphics[width=0.45\textwidth]{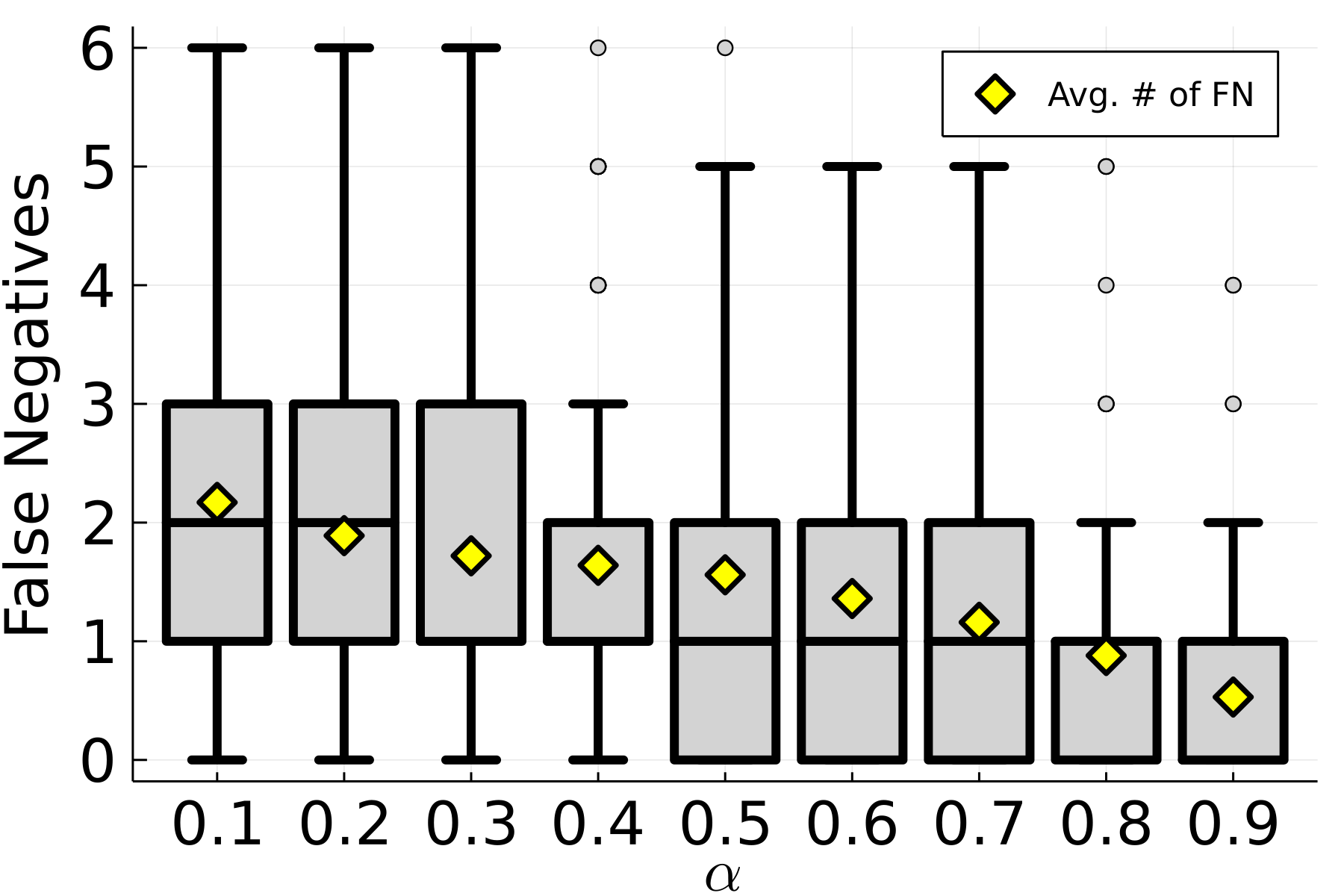}}
  \subfloat[\label{fig:fps-diff-alphas} ]{%
      \includegraphics[width=0.45\textwidth]{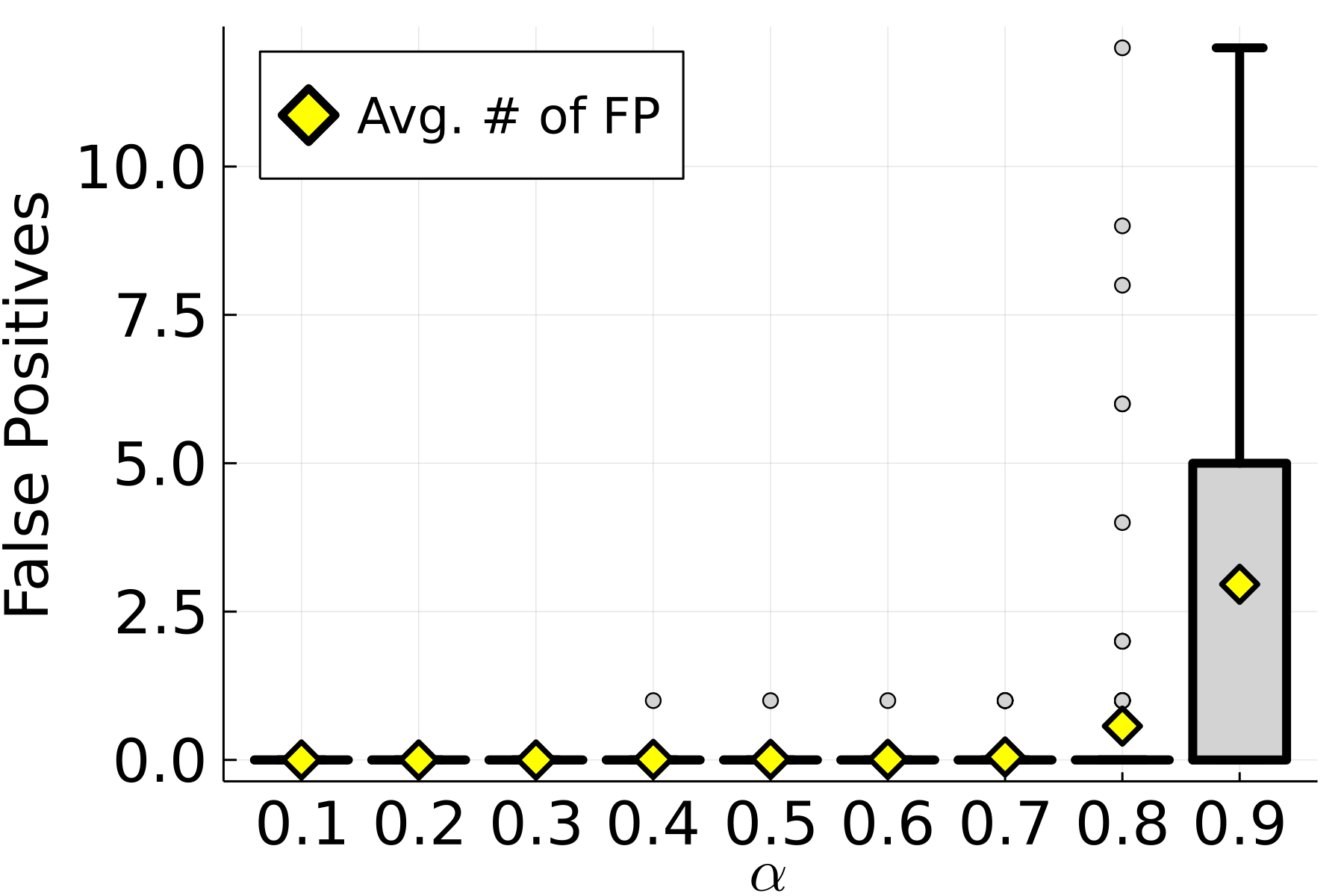}}
\caption{Numbers of false negatives (a) and false positives (b) vs varying levels of significance (\(\alpha\)) for WVM.}
\vskip -0.1in
\end{figure}

\paragraph{Varying \(\alpha\) for WVM.}
Since the asymptotic distribution from Theorem 2 and our bootstrap heuristic can sometimes lead to thresholds that are too conservative in finite samples, we investigate here the use of higher confidence levels $\alpha$ for the WVM algorithm.
More precisely, using the same simulations with linear Gaussian SCMs as in Section \ref{sec:experiments}, which are also described in Section \ref{app:detailsOnLinSim}, we display in Figures \ref{fig:fns-ratio-diff-alphas} and \ref{fig:fps-diff-alphas} the distributions of the numbers of false negatives and false positives across the 100 generated data-sets for different values of $\confLevel$ from $0.1$ to $0.9$.
It turns out that taking a higher confidence level $\confLevel$ can in fact significantly improve the performance of WVM.
In particular in this experiment, using $\confLevel = 0.7$ halves the average number of false negatives compared to $\confLevel = 0.1$, while maintaining the average number of false positives quasi-identical.

A reason why using a higher $\confLevel$ can be beneficial in some cases is that the statistical test developed in Theorem 2 was derived to be consistent and of asymptotic level $\alpha$ for rather general classes of functions $\fctClass$;
it's possible that for specific classes of functions (e.g.~the class of linear functions) the distribution of the variable in \eqref{eq:limitVariable} could be too conservative in the sense that it would only be an upper bound on the true asymptotic distribution of $\hat{\minWV}_{\weights}(\fctClass_{-k})$ under $\Tilde{\hyp}_{0, \predIdx}(\envSet)$.
Therefore, it could be of interest to investigate the asymptotic distribution of $\hat{\minWV}_{\weights}(\fctClass_{-k})$ for more specific classes of functions.
As discussed in Section \ref{sec:conclusion}, we leave such considerations for future research.

\begin{figure}[!h]
\centering
  \subfloat[\label{fig:small-fns}]{%
      \includegraphics[width=0.45\textwidth]{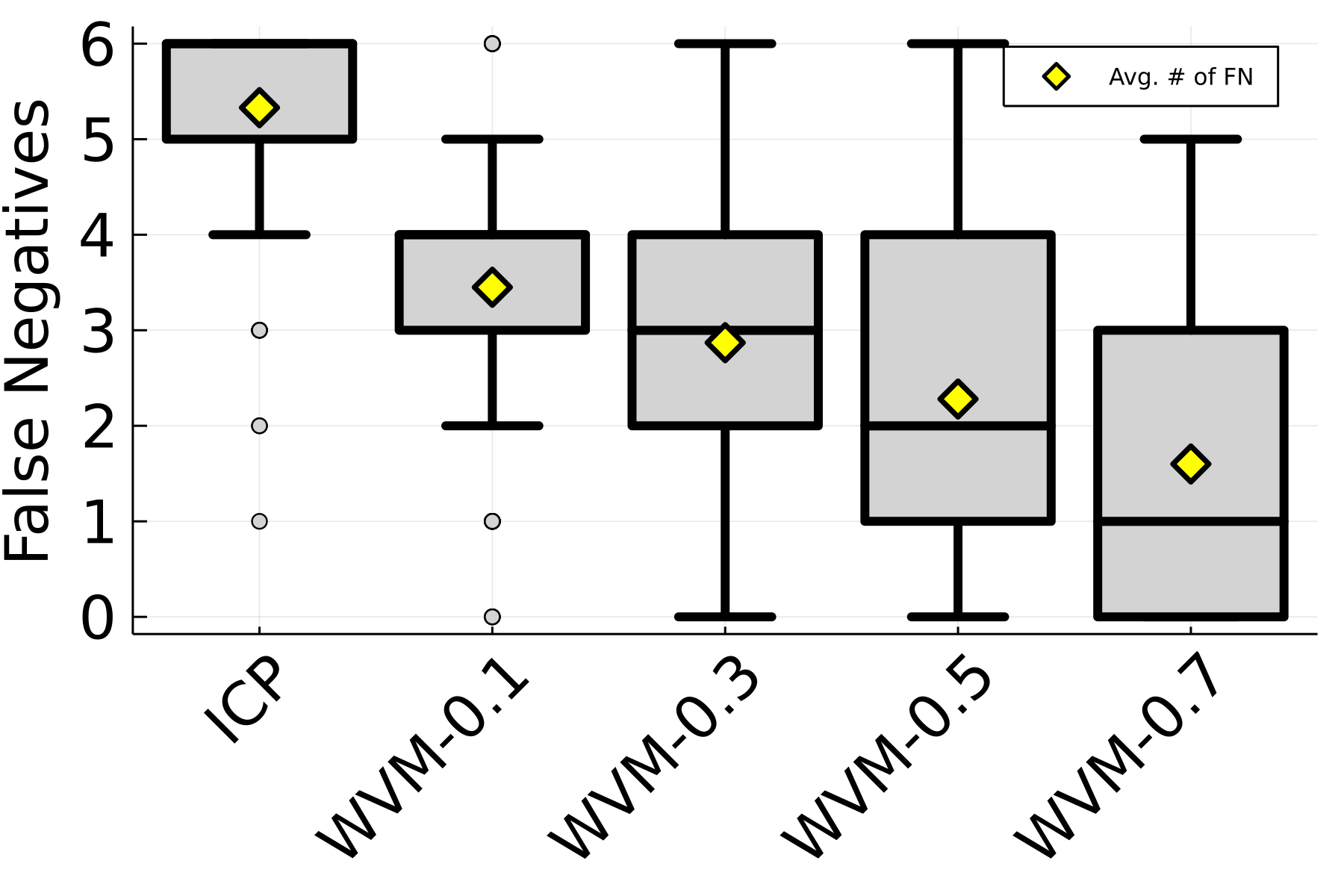}}
  \subfloat[\label{fig:small-fps} ]{%
      \includegraphics[width=0.45\textwidth]{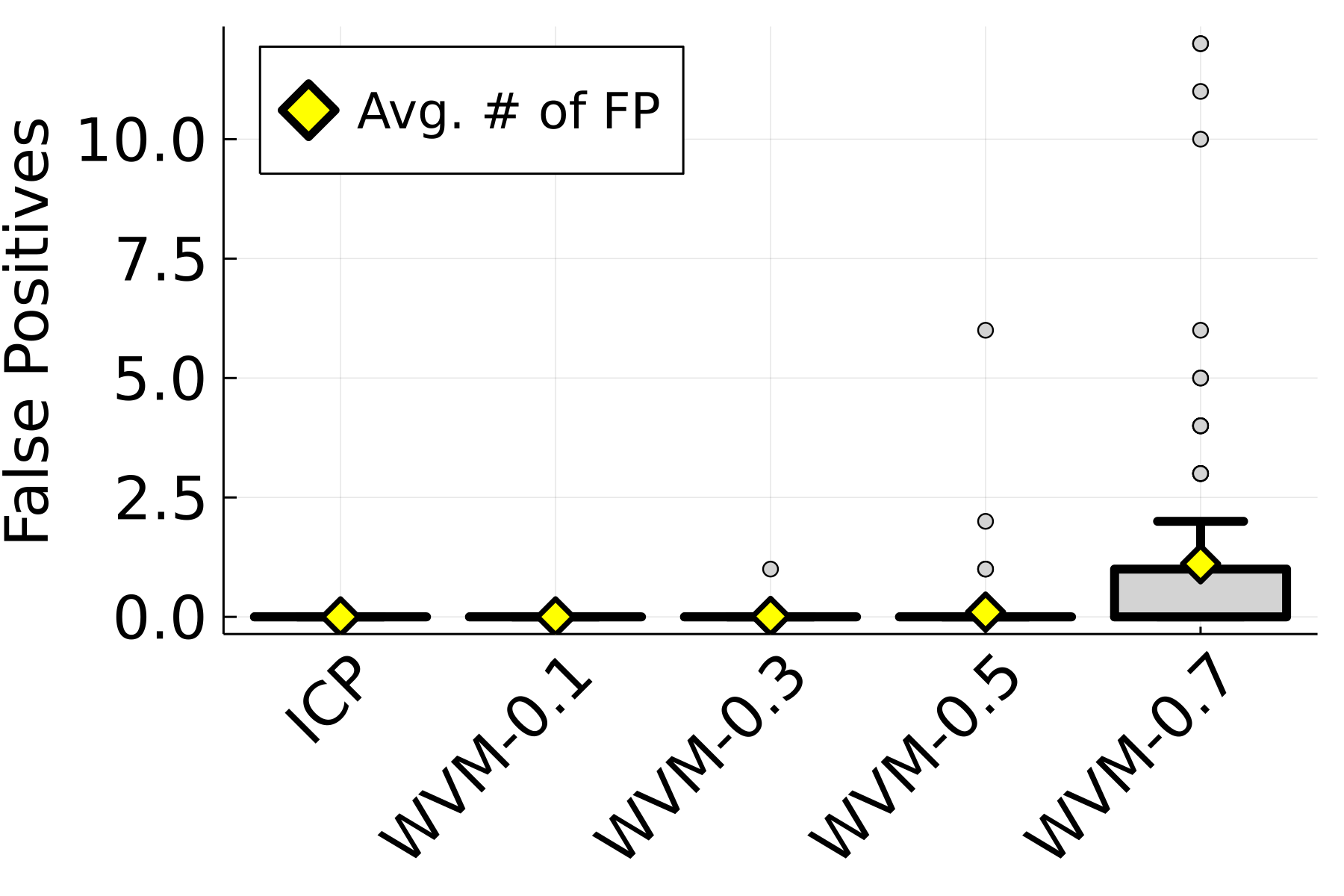}}
\caption{Number of false negatives (a) and number of false positives (b) for ICP and WVM with varying levels of \(\alpha\) for a smaller sample size simulation with \(n_e = 100\).}
\end{figure}

\paragraph{Smaller Sample Size.} We now investigate the performance of WVM vs ICP for small sample sizes.
We use again the simulations from Sections \ref{sec:experiments} and \ref{app:detailsOnLinSim} with the only difference that we now set $n_e = 100$ in each environment -- all other parameters are kept identical.
We display our results in Figures \ref{fig:small-fns} and \ref{fig:small-fps}, where we also choose $\alpha \in [0.1, 0.3, 0.5, 0.7]$ for WVM and $\alpha = 0.1$ for ICP -- we observed in practice that using different confidence levels $\alpha$ for ICP does not improve its performance significantly.
In this setting, WVM significantly outperforms ICP in terms of number of false negatives, while maintaining a number of false positives equal to zero as ICP (or at least close to zero on average for some higher $\alpha$s); the difference between ICP and WVM is even more pronounced when $\confLevel$ is larger than $0.1$ for WVM.

\begin{figure}[!h]
\centering
  \subfloat[\label{fig:fns-numS-15}]{%
      \includegraphics[width=0.45\textwidth]{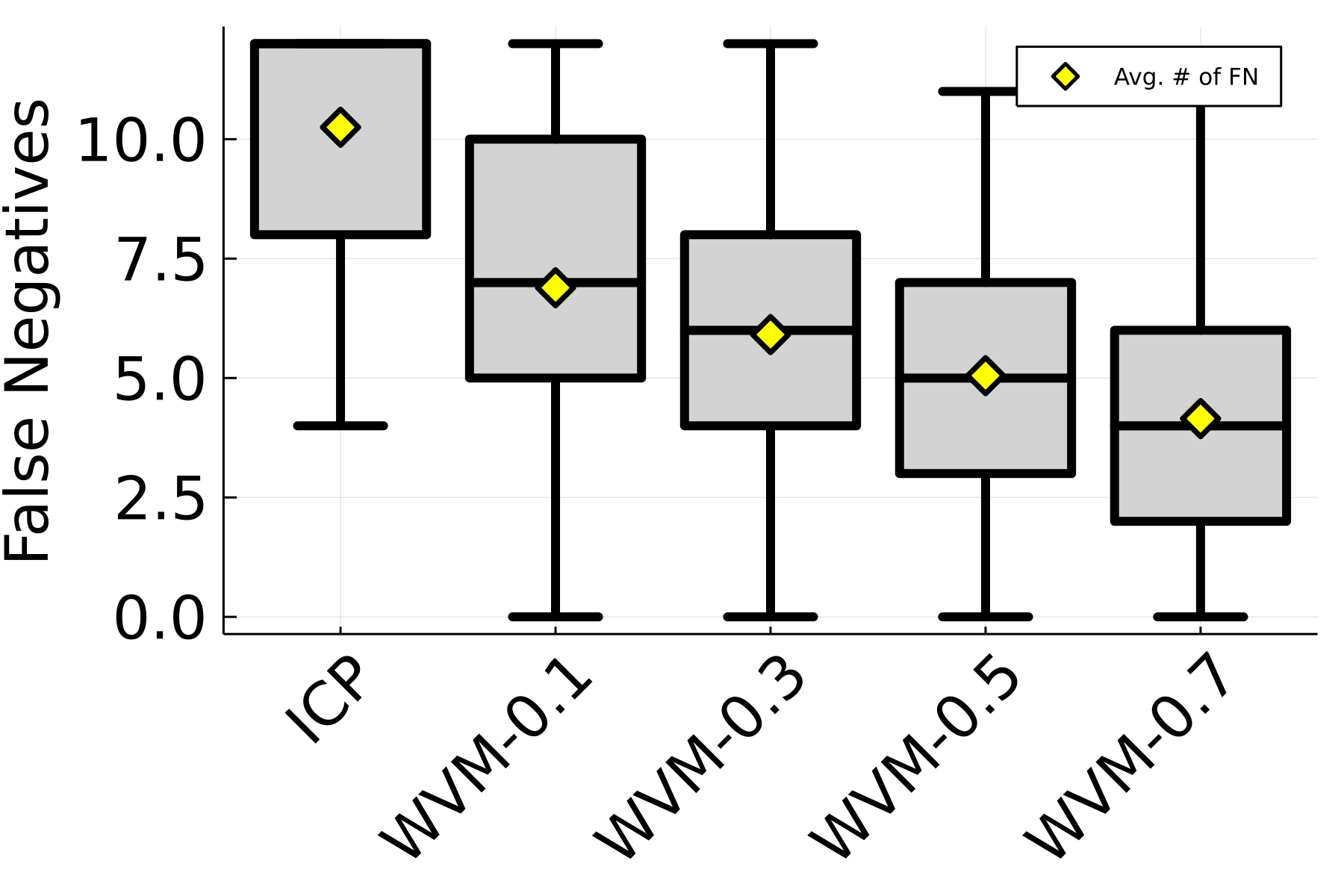}}
  \subfloat[\label{fig:fps-numS-15} ]{%
      \includegraphics[width=0.45\textwidth]{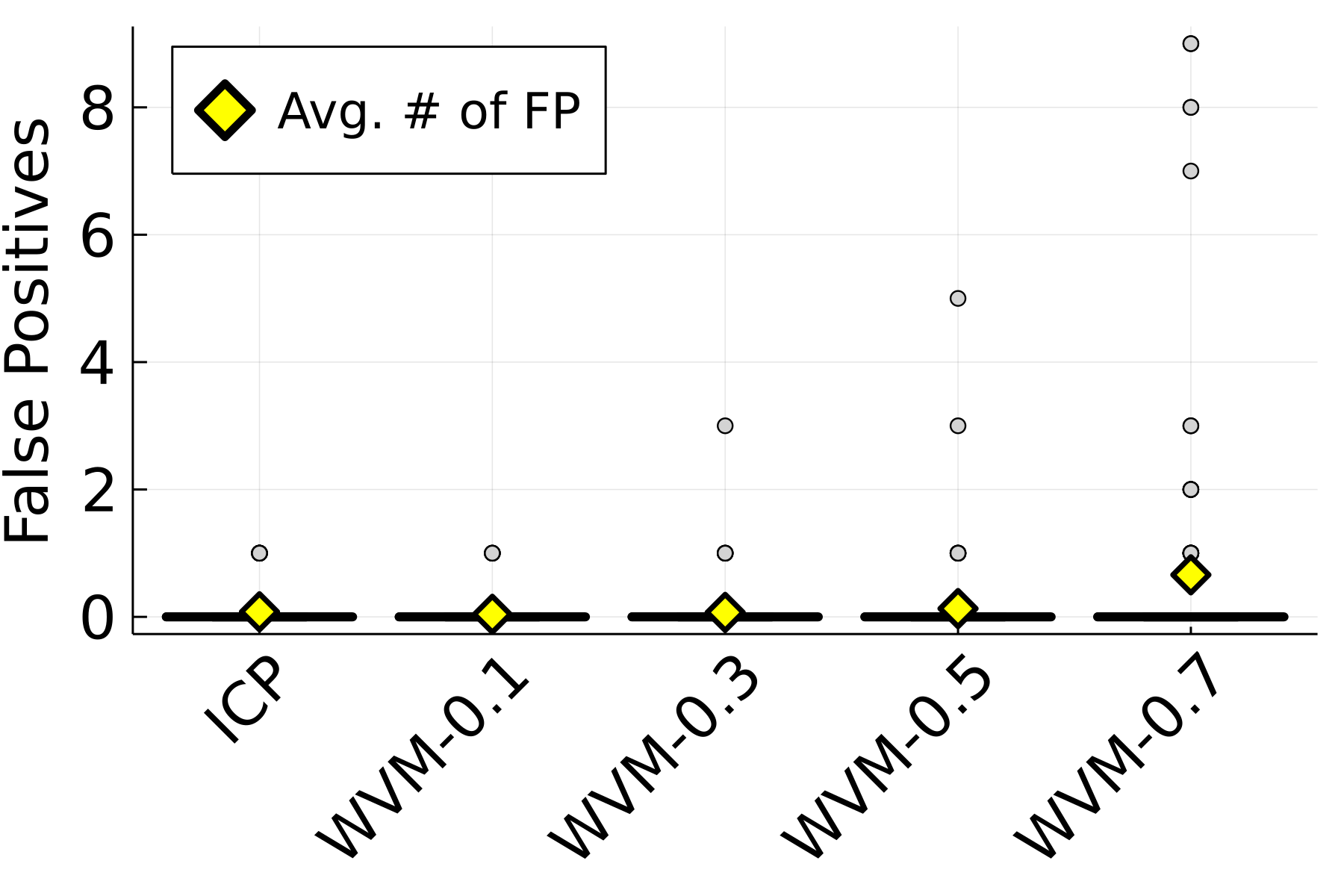}}
\caption{Number of false negatives (a) and number of false positives (b) for ICP and WVM with
varying levels of \(\alpha\) for a simulation with a larger number of direct causes, \(|S^*| = 12\).}
\end{figure}

\paragraph{More Direct Causes.}
We look at the case where the number of direct causes is set to be $|S^*| = 12$, which is half more the number of variables that ICP pre-selects with boosting (or lasso) in its default implementation.
The purpose of this section is therefore to investigate the behavior of ICP when the number of direct causes is higher than the number of pre-selected variable it uses, compared to WVM for different levels $\alpha \in [0.1,0.3,0.5,0.7]$ and with a number of pre-selected variables set at $18$. 
Such a scenario can happen in practice since ICP cannot be used efficiently even on a moderate number of variables, and therefore has to restrict itself to a small subset of pre-selected variables, while WVM can easily handle dozens of variables.

The simulations used for this experiments are the same as in Sections \ref{sec:experiments} and \ref{app:detailsOnLinSim}, with the only difference that  we set $S^*$ to be of size $12$.
As shown in Figures \ref{fig:fns-numS-15} and \ref{fig:fps-numS-15} ICP returns an empty set in many occurrences under this setting; WVM on the other hand significantly outperforms ICP by retrieving at least roughly half of the direct causes on average with almost no false positives; 
again, the difference between ICP and WVM is even more pronounced when $\confLevel$ is higher than $0.1$ for WVM.

\begin{figure}[!h]
\centering
  \subfloat[\label{fig:fns-p-10}]{%
      \includegraphics[width=0.45\textwidth]{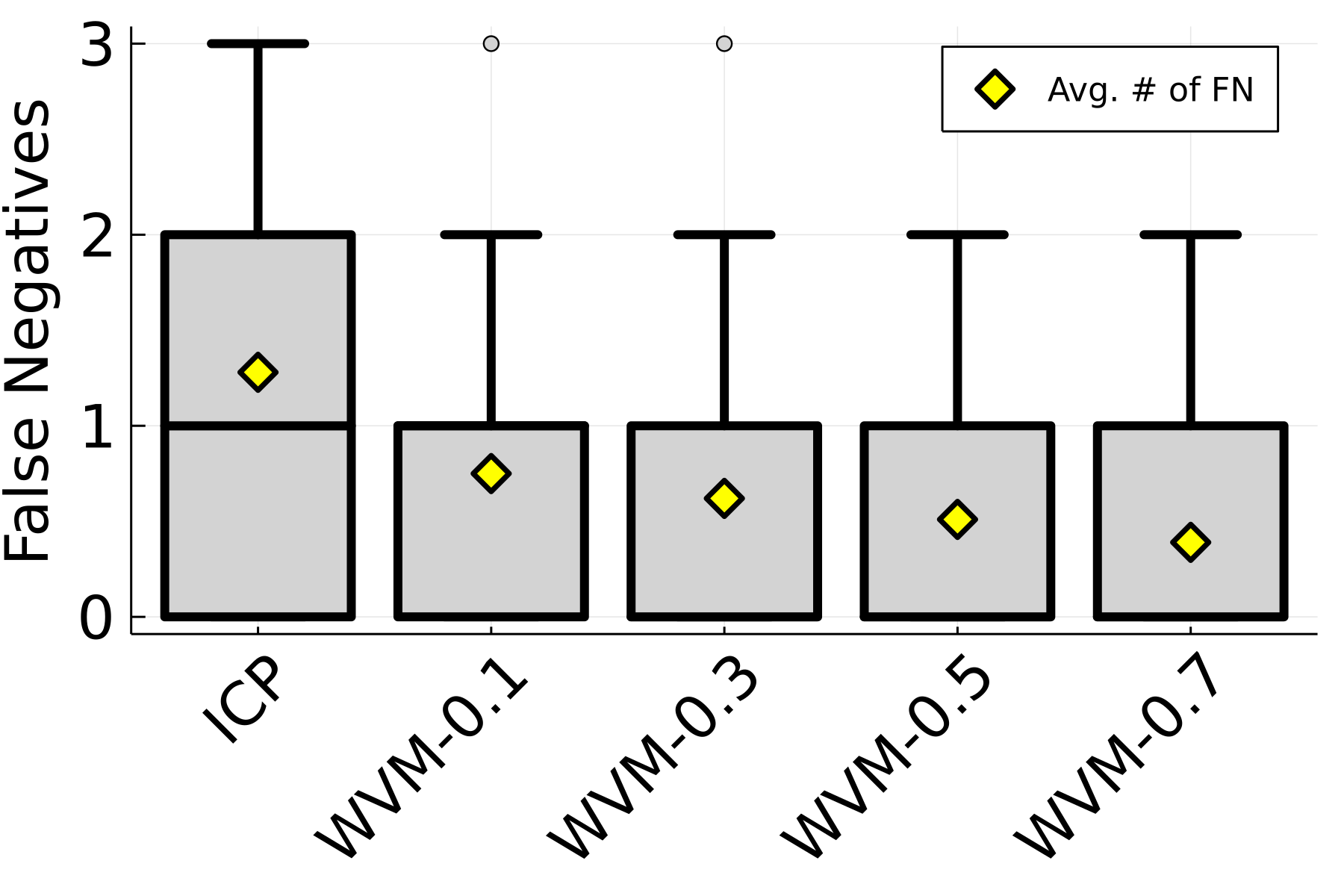}}
  \subfloat[\label{fig:fps-p-10} ]{%
      \includegraphics[width=0.45\textwidth]{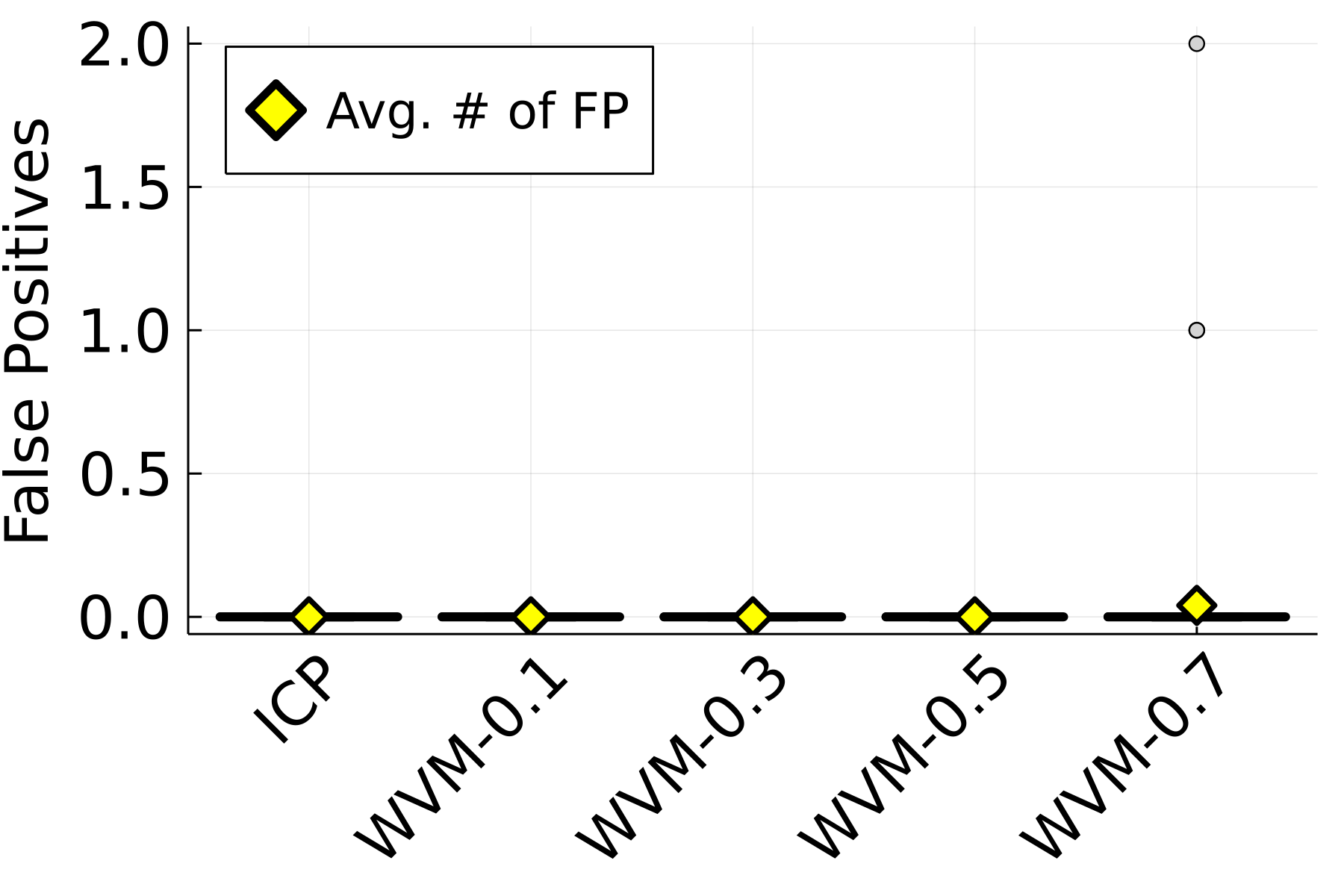}}
\caption{Number of false negatives (a) and number of false positives (b) for ICP and WVM with
varying levels of \(\alpha\) for a simulation with a smaller number of predictors, \(p = 10\).}
\end{figure}

\paragraph{Smaller Number of Predictors.}
Finally, we consider the case where the number of predictors is 10, i.e. \(p = 10\), so that there was no need to pre-select the variables with lasso for both ICP and WVM.
Since \(p = 10\), we also needed to change the number of direct causes to \(|S^*| = 3\), the number of non-descendants to 6, number of non-ascendants to 4 and the average degree was set to \(k = 3\). 
All the other parameters for the data generating process remained the same. 
This simulation shows that the higher power of WVM compared to ICP is not just a result of ICP requiring a more restrictive pre-selection step than WVM; 
with no pre-selection for both methods WVM still outperforms ICP (Figures~\ref{fig:fns-p-10} and \ref{fig:fps-p-10}).

\subsection{Real Data -- Educational Attainment}
We ran WVM on a real-world data set about the educational attainment of US teenagers~\citep{realdata-edu-attainment} -- 
the data was accessed from~\cite{AER-ref}. 
The dataset consists of 4739 students from
1100 US high schools; the purpose here is to study which factors are causal predictors of
whether these students will obtain a Bachelor of Arts degree. 
Concretely, there are 13 features recorded, some of which
include gender, ethnicity, composite score
on an achievement test, whether the father or mother graduated college, etc. 
In this setting, the target variable is binary and indicates whether
the student had greater than 16 years of education or not (the length of time required to obtain a 
Bachelor of Arts degree in the US).
Even though WVM assumes a continuous target variable, we investigate its performance when the target is binary;
we do so by applying WVM as is, but acknowledge that some modifications could improve its results.

To construct different environments from this observational data we consider
an approach taken from the original ICP paper (c.f. section 3.3~\cite{icp})
in which a variable \(U\) is chosen that is not the target and is 
known to be a non-descendent of the target in order to split the data
by conditioning on \(U\). A concrete example is when
\(U\) precedes the target chronologically. We choose (as was done in ICP) the
distance to the nearest 4-year college as the conditioning variable and
split the data into two environments: students who lived within 
the median distance of 10 miles to a 4 year college, and students who lived
farther away. 

Figure~\ref{fig:edu-attainment} shows the output of WVM run on the
educational attainment dataset. At \(\alpha = 0.1\) WVM infers only one
of the variables to be a direct cause of whether a student will obtain a
Bachelor of Arts degree or not -- the composite score on a standardized test (denoted as \textit{score} in Figure~\ref{fig:edu-attainment}). 
This result is similar to that of ICP's; they only infer one more variable as a
cause: Whether the student's father received a college degree, the variable \textit{fcollege}. 
We believe this discrepancy can be explained by model misspecification. Specifically,
ICP can sometimes return ancestors of the target that are not necessarily direct causes 
under model misspecifications such as hidden confounders,
which there are likely to be in this real world dataset (c.f. Section 6.3 and Proposition 5~\cite{icp}).
Since, WVM is agnostic to hidden confounders (assuming that the confounding factor is independent of the environment, see Appendix~\ref{app:generalSetting}) this might explain why WVM does not consider \textit{fcollege} as a direct cause (and why ICP does) as it is likely that 
\textit{score} is a mediator between \textit{fcollege} and the educational attainment of the student.
We note however,
that we are not trying to to make any causal claims here and include this example to showcase
that our method could be applied to real data. 

\begin{figure}[!ht]
\centering
  \subfloat[\label{fig:wv-vals-edu}]{%
      \includegraphics[width=0.5\textwidth]{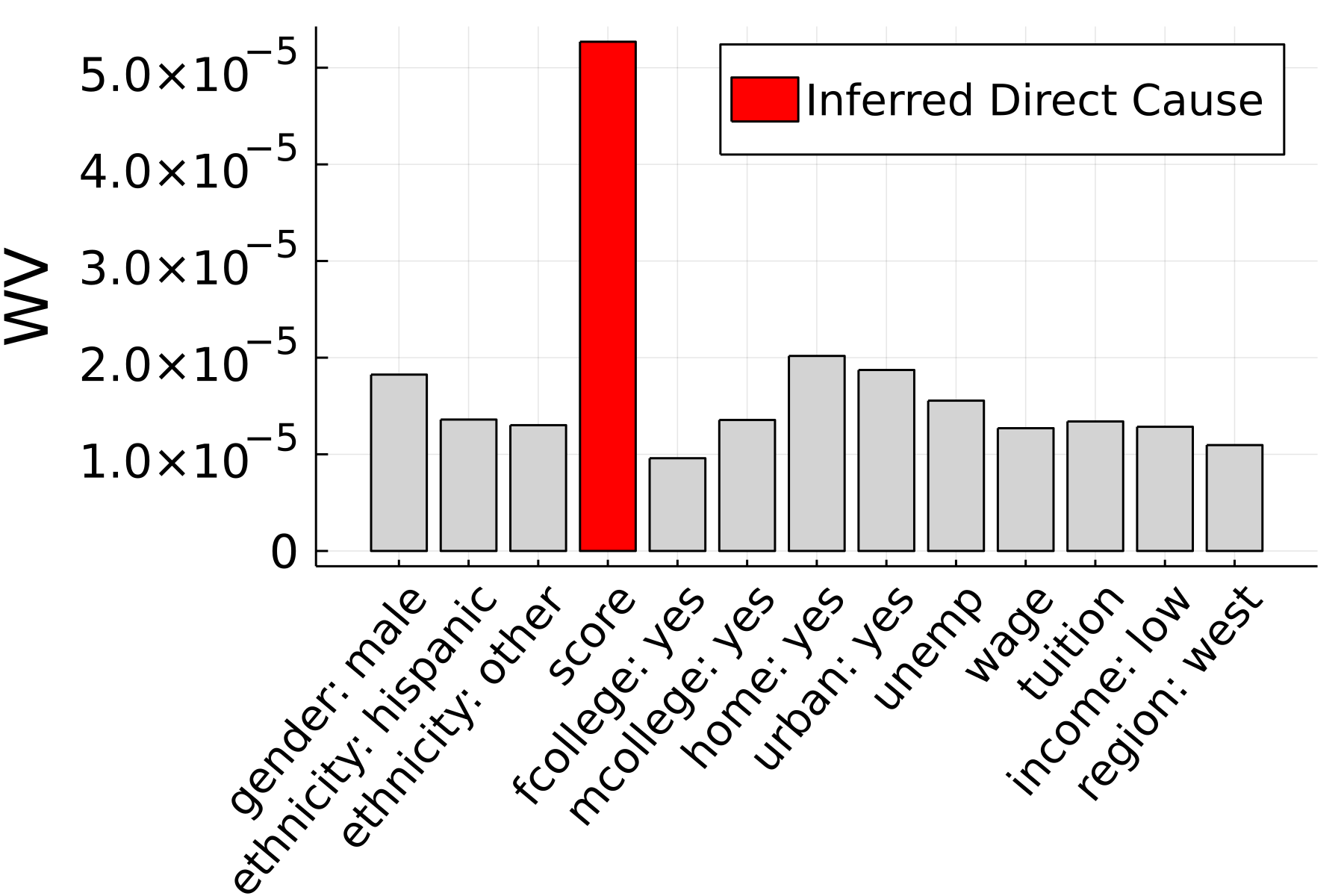}}
  \subfloat[\label{fig:p-vals-edu} ]{%
      \includegraphics[width=0.5\textwidth]{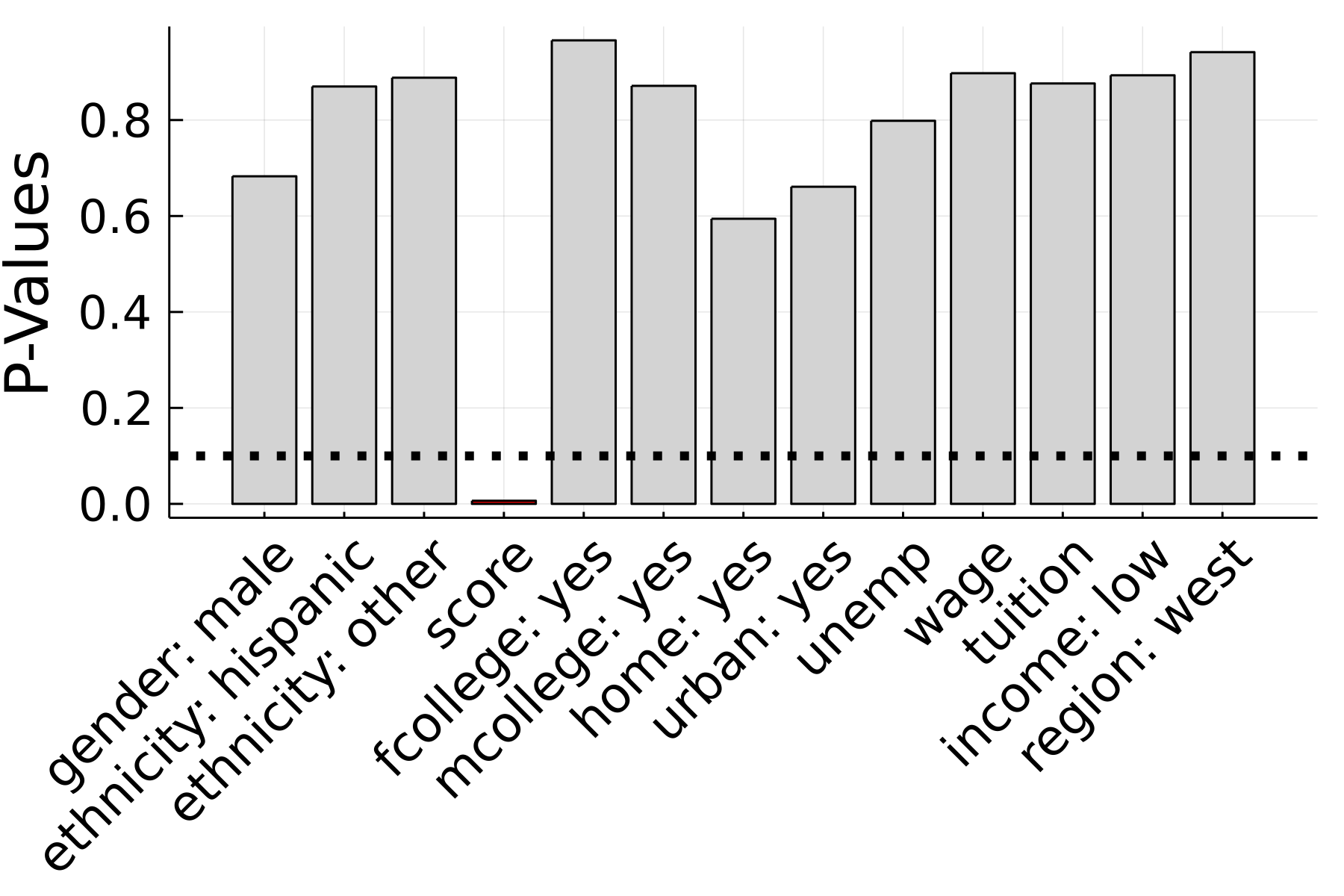}}
\caption{Histograms of WV Values and P-values for the educational attainment dataset with two environments. The inferred direct cause from WVM is given in red.}
\label{fig:edu-attainment}
\end{figure}

\end{document}